%% file: main.tex
\documentclass[letterpaper]{article} 
\usepackage{aaai23}  
\usepackage{times}  
\usepackage{helvet}  
\usepackage{courier}  
\usepackage[hyphens]{url}  
\usepackage{graphicx} 
\usepackage{natbib}  
\usepackage{caption} 
\usepackage{algorithm}
\usepackage{algorithmic}

\usepackage{newfloat}
\usepackage{listings}
\DeclareCaptionStyle{ruled}{labelfont=normalfont,labelsep=colon,strut=off} 
\floatstyle{ruled}
\newfloat{listing}{tb}{lst}{}
\floatname{listing}{Listing}
\usepackage[hidelinks]{hyperref}

\usepackage{lineno}
\usepackage{amssymb}
\usepackage{amsmath}
\usepackage{amsthm}
\usepackage{multirow}
\usepackage{color}
\usepackage{makeidx}
\usepackage{xspace}
\usepackage{nicefrac}
\usepackage{placeins}
\usepackage{xcolor}      
\usepackage{sidecap}
\usepackage{microtype}
\usepackage{thmtools} 
\usepackage{thm-restate}
\usepackage{booktabs}       
\usepackage{amsfonts}       
\usepackage[font=small]{caption}
\usepackage{mathtools} 
\usepackage{siunitx} 
\usepackage{tikz}
\usepackage{subcaption}
\usepackage{makecell}
\usepackage{paralist}
\usepackage{bbm}
\usepackage{adjustbox}
\usepackage{cleveref}
\creflabelformat{equation}{#2\textup{#1}#3} 

\declaretheorem{theorem}

\declaretheorem[sibling=theorem]{proposition}

\declaretheorem{assumption}

\DeclareMathOperator*{\argmin}{argmin}

\newcommand{\ifcp}{ACP\xspace}
\newcommand{\model}{\ensuremath{\theta}\xspace}
\newcommand{\reg}{\ensuremath{\lambda}\xspace}
\newcommand{\modelspace}{\ensuremath{\Theta}\xspace}
\newcommand{\loss}{\ensuremath{\ell}\xspace}
\newcommand{\risk}{\ensuremath{R}\xspace}

\newcommand{\ncm}{\ensuremath{A}\xspace}
\newcommand{\score}{\ensuremath{\alpha}\xspace} 
\newcommand{\pval}{\ensuremath{p}\xspace}
\newcommand{\predset}{\ensuremath{\Gamma}\xspace}
\newcommand{\xset}{\ensuremath{X}\xspace}
\newcommand{\yset}{\ensuremath{Y}\xspace}
\newcommand{\zset}{\ensuremath{Z}\xspace}
\newcommand{\objspace}{\ensuremath{\mathcal{X}}\xspace}
\newcommand{\labelspace}{\ensuremath{\mathcal{Y}}\xspace}

\newcommand{\bigoh}[1]{\ensuremath{\mathcal{O}(#1)\xspace}}

\newcommand{\mlpA}{MLP\textsubscript{A}\xspace}
\newcommand{\mlpB}{MLP\textsubscript{B}\xspace}
\newcommand{\mlpC}{MLP\textsubscript{C}\xspace}

\newcommand{\modelz}{\ensuremath{\model_{Z}\xspace}}
\newcommand{\modelloo}{\ensuremath{\model_{Z \setminus \{z_i\}}\xspace}}
\newcommand{\modelzhat}{\ensuremath{\model_{Z \cup \{\hat{z}\}}\xspace}}
\newcommand{\modelncm}{\ensuremath{\model_{Z \cup \{\hat{z}\} \setminus \{z_i\}}\xspace}}

\newcommand{\influence}{\ensuremath{I}\xspace}
\newcommand{\ifscore}{\ensuremath{\tilde{\score}}\xspace} 
\newcommand{\ifmodel}{\ensuremath{\tilde{\model}}\xspace}
\newcommand{\ifmodelloo}{\ensuremath{\tilde{\model}_{\zset \setminus \{z_i\}}}\xspace}
\newcommand{\ifmodelncm}{\ensuremath{\ifmodel_{\zset \cup \{\hat{z}\} \setminus \{z_i\}}}\xspace}
\newcommand{\ifmodelzhat}{\ensuremath{\ifmodel_{\zset \cup \{\hat{z}\}}\xspace}}
\newcommand{\ifloss}{\ensuremath{\tilde{\loss}}\xspace}

\newcommand\card{\text{\ttfamily\#}} 
\newcommand{\suchthat}{\ensuremath{\,:\,}} 

\newbox{\bigpicturebox}

\definecolor{algochange}{HTML}{005AB5}
\definecolor{oldalgo}{HTML}{AE422B}
\definecolor{star}{HTML}{F08080}

\title{Approximating Full Conformal Prediction at Scale via
Influence Functions}

\author {
    Javier Abad\textsuperscript{\rm 1 2},
    Umang Bhatt \textsuperscript{\rm 1 3},
    Adrian Weller \textsuperscript{\rm 1 3},
    Giovanni Cherubin \textsuperscript{\rm 4}
}
\affiliations {
    \textsuperscript{\rm 1} University of Cambridge, UK\\
    \textsuperscript{\rm 2} ETH Zurich, Switzerland\\
    \textsuperscript{\rm 3} The Alan Turing Institute, London, UK\\
    \textsuperscript{\rm 4} Microsoft Research, Cambridge, UK\\
    javier.abadmartinez@ai.ethz.ch,
    \{usb20, aw665\}@cam.ac.uk,
    giovanni.cherubin@microsoft.com
}

\begin{document}

\maketitle

\begin{abstract}

Conformal prediction (CP) is a wrapper around traditional machine learning models, giving coverage guarantees under the sole assumption of exchangeability; in classification problems, for a chosen significance level $\varepsilon$, CP guarantees that the error rate is at most $\varepsilon$, irrespective of whether the underlying  model is misspecified. However, the prohibitive computational costs of ``full'' CP led researchers to design scalable alternatives, which alas do not attain the same guarantees or statistical power of full CP. In this paper, we use influence functions to efficiently approximate full CP. We prove that our method is a consistent approximation of full CP, and empirically show that the approximation error becomes smaller as the training set increases; e.g., for $10^{3}$ training points the two methods output p-values that are $<10^{-3}$ apart: a negligible error for any practical application. Our methods enable scaling full CP to large real-world datasets. We compare our full CP approximation (\ifcp{}) to mainstream CP alternatives, and observe that our method is computationally competitive whilst enjoying the statistical predictive power of full CP.
\end{abstract}
\section{Introduction}

Conformal prediction (CP) is a post-hoc approach to providing validity guarantees on the outcomes of machine learning (ML) models;
in classification, an ML model wrapped with ``full''
CP outputs prediction sets
that contain the true label with (chosen)
probability $1-\varepsilon$,
under mild distribution assumptions.
Unfortunately, full CP is notoriously computationally expensive.
Many have proposed alternative methods to avoid the full CP objective;
these include:
split (or ``inductive'') CP~\cite{papadopoulos2002inductive},
cross-CP~\cite{vovk2012cross},
jackknife+~\cite{barber2021predictive},
RAPS~\cite{Angelopoulos2021UncertaintySF},
CV+~\citep{Romano2020ClassificationWV}.
While these methods have shown practical promise,
they do not attain the
same validity guarantee as full CP
or its statistical power
(e.g., prediction set size).
Recent work optimized full CP for ML models that support incremental and
decremental learning by speeding
up the leave-one-out (LOO) procedure
required for the prediction set
calculation~\citep{pmlr-v139-cherubin21a};
however, this approach may not scale to complex models such as neural networks.

\begin{figure*}[t]
\vspace{-0.75cm}
\hspace{-0.65cm}
\centering
\begin{minipage}{.28\textwidth}
\begin{center}
\begin{subfigure}{.3\textwidth}
\adjustbox{raise=-5pc}{\scalebox{2}[2]{
\includegraphics[width=0.55\linewidth]{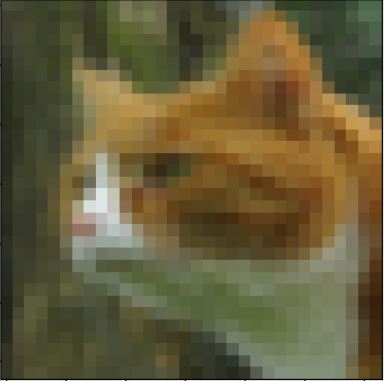}}}
\end{subfigure}
\end{center}
\begin{subfigure}[b]{.3\textwidth}
\begin{center}
\footnotesize
\begin{tabular}{ll}
    \toprule
     Method & Prediction set \\ \midrule
     ACP  & bird, \underline{\textbf{cat}}, deer, frog \\ 
     SCP & bird, deer, frog \\ 
     RAPS & bird, \underline{\textbf{cat}}, deer, dog, frog \\ 
     CV+ & bird, \underline{\textbf{cat}}, deer, dog, frog 
 \end{tabular}
\end{center}
\end{subfigure}
\end{minipage}
\begin{minipage}{.33\textwidth}
\begin{center}
\begin{subfigure}{.33\textwidth}
\adjustbox{raise=-5pc}{\scalebox{1.7}[1.7]{ \includegraphics[width=0.5\linewidth]{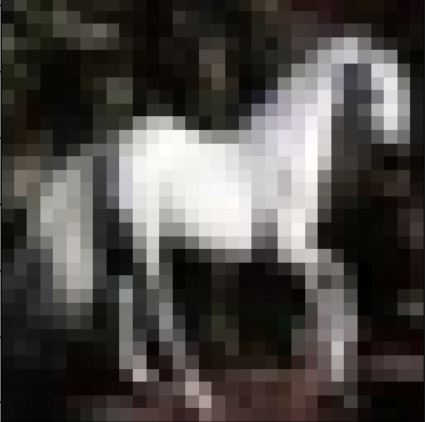}}}
\end{subfigure}
\end{center}
\begin{subfigure}[b]{.33\textwidth}
\begin{center}
\footnotesize		 
\begin{tabular}{ll}
    \toprule
     Method & Prediction set \\ \midrule
     ACP  & auto, cat, frog, \underline{\textbf{horse}}, truck \\ 
     SCP & auto, deer, frog, truck \\ 
     RAPS & plane, auto, bird, deer, frog, ship, truck\\ 
     CV+ & plane, auto, deer, frog, \underline{\textbf{horse}}, truck 
 \end{tabular}
\end{center}
\end{subfigure}
\end{minipage}
\hspace{0.75cm}
\begin{minipage}{.3\textwidth}
\vspace{0.06cm}
\begin{center}
\begin{subfigure}{.3\textwidth}
\adjustbox{raise=-5pc}{\scalebox{2}[2]{
\includegraphics[width=0.5\linewidth]{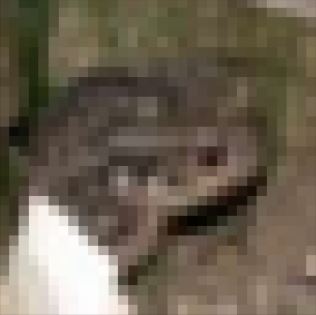}}}
\end{subfigure}
\end{center}
\begin{subfigure}[b]{.26\textwidth}
\begin{center}
\footnotesize
\begin{tabular}{ll}
    \toprule
     Method & Prediction set \\ \midrule
     ACP  & cat, deer, \underline{\textbf{frog}}, horse \\  
     SCP & cat, deer, dog, \underline{\textbf{frog}}, horse, truck \\  
     RAPS & cat, deer, dog, \underline{\textbf{frog}}, horse, truck \\ 
     CV+ & cat, deer, dog, \underline{\textbf{frog}}, horse
 \end{tabular}
\end{center}
\end{subfigure}
\end{minipage}
\caption{Prediction sets generated by CP methods ($\varepsilon=0.2$) for CIFAR-10 examples. Our method (ACP) yields prediction sets that (1) contain the true label and (2) are the smallest.
ACP approximates well full CP in large training sets, such as CIFAR-10, inheriting its statistical power.
}
\label{fig:figure1}
\end{figure*}

Herein, we first discuss how to approximate the full CP objective.
We focus on full CP for classification, and optimize it for
ML models trained via ERM (e.g., logistic regression, neural networks).
The key insight we leverage is that,
for each test point, full CP:
(i) retrains the underlying ML model on the additional test point, 
and (ii) performs a LOO procedure for each training point.
We observe we can approximate both steps, and avoid retraining each time, by using first order influence functions~\citep{hampel1974influence}.
We term our method \textit{Approximate full Conformal Prediction} (\ifcp{}), and we prove finite-sample error guarantees:
as the training set grows,
\ifcp{} approaches full CP.
We then show that a stronger regularization parameter
for training the underlying ML model improves
the approximation quality.

We empirically demonstrate that \ifcp{} is competitive with existing methods on MNIST~\cite{lecun1998mnist}, CIFAR-10~\cite{Cifar}, and US Census~\cite{ding2021retiring}.
Unlike full CP, \ifcp{} scales to large datasets
for real-world ML models (logistic regression, multilayer perceptrons, and convolutional neural networks). 
Performance-wise, \ifcp{} is consistently better than existing
alternatives in terms of statistical power: it attains the
desired error rate $\varepsilon$ with tighter prediction
sets;
\Cref{fig:figure1} shows on CIFAR-10 examples where, unlike other methods, \ifcp{} learns smaller prediction sets that still contain the true label.

\section{Preliminaries}
\label{sec:back}

We describe full CP, and then introduce influence functions, our main
optimization tool.

\subsection{Notation and full CP}\label{subsec:notation}
Consider a training set $\zset = (\xset, \yset) \in (\objspace \times \labelspace)^N$.
For a test object $x \in \objspace$ and a chosen
significance level $\varepsilon \in [0, 1]$, a CP returns
a set $\predset^\varepsilon_x \subseteq \labelspace$
containing $x$'s true label with probability
at least $1-\varepsilon$.
This guarantee (\textit{validity})
holds for any exchangeable distribution on
$\zset \cup \{(x, y)\}$.
Since the error rate of a CP is guaranteed by validity,
a data analyst only needs to control the tightness (\textit{efficiency})
of
its prediction set;
average
$|\predset^\varepsilon_x|$ is a common efficiency criterion~\citep{vovk2016criteria}.
Efficiency is controlled by improving the
underlying model that CP wraps.

\paragraph{Underlying model.}
A CP can be built around virtually any ML model $\model$.
We assume the underlying model is trained via
ERM by minimizing the risk:
$\risk(\zset, \hat{\model}) \equiv \frac{1}{N}\sum_{z_i \in \zset} \loss(z_i, 
\hat{\model})$;
$\loss(z, \model)$ is the loss of the model at a point $z$, which
we assume to be convex and twice differentiable
in $\model$. This assumption holds for many popular loss functions.
Let
$\modelz \equiv \argmin_{\hat{\model} \in \modelspace} \risk(\zset, \hat{\model})$
be the ERM solution;
we assume $\modelz$ to be unique, and discuss
relaxations in \Cref{sec:discussion}.

\paragraph{Nonconformity measure.}
The underlying model is used to construct a nonconformity
measure, which defines a CP.
A nonconformity measure is a function
$\ncm: (\objspace\times\labelspace) \times (\objspace\times\labelspace)^N \rightarrow \mathbf{R}$ which
scores how \textit{conforming} (or \textit{similar})
an example $(x, y)$ is to a bag of examples
$\bar{\zset}$.
We focus on the two most common
ways of defining nonconformity measures (and, hence, CP) on the basis of
a model: the \textit{deleted} and the \textit{ordinary} scheme~\cite{vovk2005algorithmic}.

\paragraph{Full CP (deleted).}
Consider example $\hat{z}$ and a training set $\zset$.
The nonconformity measure can be defined from the
deleted (LOO) prediction: $\ncm(z_i, \zset) = \loss(z_i, \modelncm)$,
for all $z_i \in \zset \cup \{\hat{z}\}$.
Computing this nonconformity measure
requires training the model on
$\zset \cup \{\hat{z}\} \setminus \{z_i\}$. This scheme computes the loss at a point
after removing it from the model's training data.

Algorithm 1 shows how the nonconformity
measure is used in full CP.
For a test point $x$, CP runs a statistical test
for each possible label $\hat{y} \in \labelspace$
to decide if it should be included in the prediction
set $\predset^\varepsilon_x$.
The statistical test requires computing a nonconformity score
$\alpha_i$ by running $\ncm$ for each point
in the \textit{augmented} training set $\zset \cup \{(x, \hat{y})\}$;
then, a p-value is computed, and a decision is
taken based on the threshold $\varepsilon$.
This algorithm is particularly expensive.
Crucially, for each test point, and for every candidate
label, one needs to retrain the underlying
ML model $N+1$ times.

\paragraph{Full CP (ordinary).}
A computationally faster scheme
is achieved by taking the loss at the point: $\ncm(z_i, \zset) = \loss(z_i, \modelzhat)$.
We refer to this as the \textit{ordinary} scheme (\Cref{algo:cp-ordinary}).
This method is inherently faster than
the deleted approach, as it
only requires training one model per test example and candidate label.
However, the ordinary scheme generally leads to less efficient
predictions (\Cref{sec:experiments}).

\paragraph{Optimizing CP.}
The complexity of full CP depends on: (i) 
the number of training points $N$, and (ii) the number of labels $|\labelspace|$.
Optimizing w.r.t.  $|\labelspace|$ 
is necessary for regression settings, where full CP is
not applicable directly; this was done, for specific choices
of nonconformity measures, by \citet{papadopoulos2011regression,nouretdinov2001ridge,lei2019fast,ndiaye2019computing,ndiaye2022stable}.
Our work focuses on optimizing w.r.t. $N$; this enables
applying full CP classification to large datasets.
Future work may combine our optimizations and CP regression
strategies to obtain faster regressors on large training sets (e.g., \citealp{pmlr-v139-cherubin21a}).

\subsection{Influence functions}
Influence functions (IF) are at the
core of our proposal.
For a training example $z_i \in \zset$,
let $\influence_\model(z_i) = -\frac{1}{N}H_\model^{-1} \nabla_\model \loss(z_i, \model)$
be the \textit{influence} of $z_i$ on model $\model$,
where $H_\model = \nabla^2_\model \risk(\zset, \model)$ is the Hessian;
by assumption, $H_\model$ exists and is invertible.
A standard result by \citet{hampel1974influence} shows that:
\begin{equation}
	\modelloo - \model \approx - \influence_\model(z_i).
	\label{eq:if-model}
\end{equation}
$\influence_\model$ says
how much $z_i$ affects
$\model$
during training.
We can apply influence functions for computing the influence of a point
$z_i$ on any functional.
In our work, we are interested in the influence
on the loss function at a point $z$. Let
$\influence_\loss(z, z_i) = \nabla_\model \loss(z, \model)^\top \influence_\model(z_i)$.
Then, similarly to above, we have
\begin{equation}
	\loss(z, \modelloo) - \loss(z, \model)
		\approx -\influence_\loss(z, z_i).
	\label{eq:if-loss}
\end{equation}

\begin{figure*}[tb]
\begin{minipage}{0.5\textwidth}
    \input{figures/algo-full-cp}
\end{minipage}
\hfill
\begin{minipage}{0.5\textwidth}
    \input{figures/algo-our-method}
\end{minipage}
\caption{Full CP (left) and our proposal, \ifcp{}, (right). In \ifcp{}, the
underlying ML model is only trained once,
and nonconformity scores are approximated
via influence functions in Line 6.
}
\label{fig:algorithms}
\end{figure*}

\section{Approximate full Conformal Prediction}
\label{sec:method}

Our proposal (\ifcp{}) hinges on approximating
the nonconformity scores via IF.
We describe our approach, and prove theoretical
results on its consistency and
approximation error.

\subsection{Approach}
The bottleneck of running full CP is the
computation of the nonconformity scores $\score_i = \loss(z_i, \modelncm)$.
Each score is determined by computing the loss of the model
at point $z_i \in \zset \cup \{\hat{z}\}$ after
adding point $\hat{z}$ and removing point $z_i$
from the model's training data $\zset$.
There are two ways to approximate $\score_i$ via IF:
we can approximate the contribution of adding
and removing the points to the learned model
$\modelz$, and then evaluate its loss
at $z_i$ (\textit{indirect} approach),
or we can directly approximate the contribution
of the points
on the loss function
(\textit{direct} approach).
We describe both below.

\paragraph{Indirect approach.}
We can use \Cref{eq:if-model} to
approximate model $\modelncm$ and then
compute its loss.
That is, let
$\ifmodelncm \equiv \model_{\zset}
+\influence_{\model_\zset}(\hat{z})
-\influence_{\model_\zset}(z_i)$.
Then:
\begin{equation}
    \label{eq:score-approx-model}
    \score_i \approx \loss(z_i, \ifmodelncm) \,.
\end{equation}

\paragraph{Direct approach.}
We can directly compute the
influence on the loss.
Let $\model_{\zset}$ be a model trained via ERM on
the entire training set $\zset$.
The direct approximation for the score
is:
\begin{align}
    \label{eq:score-approx-loss}
    \score_i
    &\approx
        \ifloss(z_i, \modelncm)
    \equiv
        \loss(z_i, \model_\zset) +
        \influence_\loss(z_i, \hat{z}) -
        \influence_\loss(z_i, z_i).
\end{align}
$\influence_\loss(z_i, \hat{z})$
and $-\influence_\loss(z_i, z_i)$ are the influence
of \textit{including} point $\hat{z}$ and \textit{excluding}
$z_i$ (\Cref{eq:if-loss}).
Algorithm 2 (\ifcp{}) shows how both approaches enable
approximating full CP.

\ifcp{} gives a substantial speed-up over full CP.
In contrast to full CP, \ifcp{} has
a training phase, in which we:
compute the Hessian, the gradient for each
point $z_i$,
and provisional scores
$\loss(z_i, \modelz)$ for
$z_i \in \zset$.
For predicting a test point $\hat{z}$, it suffices
to compute its influence by using the Hessian
and gradients at $z_i$ and $\hat{z}$,
which is cheap, and update the provisional scores (see time complexities in \Cref{tab:complexities}). This enables \ifcp{} to scale to large real-world datasets such as CIFAR-10 (\Cref{sec:experiments}). As a reference, running full CP for the synthetic dataset (\Cref{sec:synthetic}) with 50 features,  1000 training points, and 100 test points took approximately 5 days, whereas ACP took less than 1 hour (CPU time with an Intel Core i7-8750H).

\subsection{Theoretical analysis}\label{sec:theory}

In this section, we establish the consistency of \ifcp{}: its approximation error gets smaller as the training set grows.
Further, we study its \textit{finite-sample} validity, and how the underlying model's regularization parameter affects its approximation error.
The consistency of \ifcp{} for the indirect approach comes from a result by \citet{giordano2019swiss}.
Proving consistency for the direct approach requires a condition, which we state in the next part as a conjecture.

\subsubsection*{Direct approximation is better than indirect}
We 
conjecture that the direct approach approximates better than the indirect one. Intuitively, it is much easier to approximate the loss at a point (direct) than to estimate the effect of a training point on the model weights, which lay in a high-dimensional space (indirect).
We observed this conjecture to hold consistently across a number of simulations (\Cref{sec:synthetic}).
Formally:

\begin{restatable}[]{condition}{directvsindirect}
\label{thm:approx-loss-vs-model}
Assume that the loss $\loss$ is convex and differentiable.
Then
the direct method (\Cref{eq:score-approx-loss})
is a better approximation
than the indirect one (\Cref{eq:score-approx-model}).
That is, let
$\score_i = \loss(z_i, \modelncm)$; then:
$$|\ifloss(z_i, \modelncm) - \score_i| \leq
    |\loss(z_i, \ifmodelncm) - \score_i| \,.$$
\end{restatable}
Without \Cref{thm:approx-loss-vs-model}, we can prove consistency for the indirect approach, but not for the direct approach.

\subsubsection*{Consistency of \ifcp{}}
We show that \ifcp{} is a consistent estimator of full CP.
We establish this equivalence in the most generic form possible:
we demonstrate that nonconformity scores produced by
Algorithm 2
approximate those produced by full CP.
In turn, the p-values
(and, consequently, error rates) of the two methods
get increasingly closer.

Our result is an extension of the work by \citet{giordano2019swiss},
who showed that IF
consistently estimate a model's parameters in a LOO
setting.
This result holds under a set of assumptions
(\Cref{assumption:giordano} in
\Cref{proofs}), which \citet{giordano2019swiss} showed
to hold for a variety of settings;
e.g., they hold when $\zset$ are well-behaved IID data
and $\loss(\cdot, \model)$ is an appropriately smooth
function.
Note that \Cref{assumption:giordano} limits the set of applicable nonconformity measures, e.g., by assuming them to be continuously differentiable.

\begin{restatable}[Consistency of approximate full CP]{theorem}{consistency}
Under \Cref{assumption:giordano} and \Cref{thm:approx-loss-vs-model},
let $\score_i = \loss(z_i, \modelncm)$,
and suppose $\loss$ is $K$-Lipschitz.
For every $N$ there is a constant $C$ such
that for every $z_i \in \zset \cup \{\hat{z}\}$:
$$|\ifloss(z_i, \modelncm) - \score_i| \leq
KC\frac{\max\{C_g, C_h\}^2}{N} \,,$$
for finite constants $C_g, C_h$ s.t.
	$\sup\limits_{\model \in \modelspace} \frac{1}{\sqrt{N}}
		||\nabla_\model\loss(z, \model)||_2 \leq
			C_g$ and
	$\sup\limits_{\model \in \modelspace} \frac{1}{\sqrt{N}}
	||\nabla^2_\model\loss(z, \model)||_2 \leq
	C_h$. \label{thm:consistency}
\end{restatable}
	
This result gives a finite-sample bound for the error of the direct approach for \ifcp{};
the error of the indirect approach
is also bounded as a byproduct of the same
proof.
We conclude that,
as $N$ grows, \ifcp{}'s
scores get
increasingly closer to those produced
by full CP.
We evaluate this in
\Cref{sec:synthetic}.

\subsubsection*{Validity of ACP}
\citet{Lin2021locally} state that \textit{finite-sample} validity is not guaranteed when the LOO is estimated with IF since they cannot be exactly computed. They exemplify this issue in the Discriminative Jackknife \cite{alaa2020discriminative}, which approximates the IF using Hessian-Vector-Products \cite{pearlmutter1994fast}. 

Although we alleviate part of the issue by computing the exact Hessian, we cannot guarantee that our LOO estimation is exact. \citet{basu2021influence} also summarize several issues with using IF in deep learning. 
Nevertheless, ACP still inherits the high efficiency of full CP, and we observe that validity holds in practice (\Cref{sec:experiments}). Future work can prove if the approximate scores follow the same distribution as the true ones and, consequently, if exchangeability still holds.

\subsubsection*{Relation to regularization parameter}
By extending a result by \citet{koh2019accuracy},
we investigate the effect of the ERM regularization parameter on
ACP's approximation error.
This result makes fairly simplistic assumptions (\Cref{appendix:theory-regularization}).

\begin{restatable}[Approximation goodness w.r.t.\ regularizer]{theorem}{regularizer}
\label{thm:approx-regularizer}
Suppose the model is trained via ERM with
regularization parameter $\reg$.
Under the assumptions of \Cref{thm:koh-self-loss},
Assumption~\ref{assumption:influence-zhat},
and neglecting $\bigoh{\lambda^{-3}}$ terms,
we have the following cone constraint between
the true nonconformity measure
$\score_i = \loss(z_i, \modelncm)$ and its direct
approximation
$\ifloss(z_i, \modelncm) \equiv
    \loss(z_i, \modelz) +
    \influence_\loss(z_i, \hat{z}) -
    \influence_\loss(z_i, z_i)$, where $g(\lambda) = (1 + \nicefrac{3\sigma_{max}}{2\reg} + \nicefrac{\sigma_{max}^2}{2\reg^2})$,
and $\sigma_{max}$ is the maximum eigenvalue
of the Hessian $H$:
\begin{equation*}
    \loss(z_i, \modelz) +
    \influence_\loss(z_i, \hat{z}) -
    g(\lambda)\influence_\loss(z_i, z_i) 
\leq \score_i
    \leq \ifloss(z_i, \modelncm) .
\end{equation*}
\end{restatable}

\section{Experiments on synthetic data}
\label{sec:synthetic}
We study the properties of \ifcp{} outlined in~\Cref{sec:theory} on synthetic
data (\Cref{app:exp});
the underlying model is logistic regression with cross-entropy loss.
Results are averaged across 100 test points.

\paragraph{Direct and indirect approximation.}
We empirically evaluate \Cref{thm:approx-loss-vs-model}, which
claims that the direct method (\Cref{eq:score-approx-loss})
is never worse than indirect (\Cref{eq:score-approx-model}).
\Cref{fig:methods_comparison_norm} shows
the absolute
distance between full CP and
\ifcp{}'s nonconformity scores as a function of the training
set size.
Results confirm that
direct is always better than indirect,
although the two get close for large $N$.
Importantly, 
the nonconformity scores produced by \ifcp{}
get increasingly better at approximating those
of full CP as the training set grows
(cf. \Cref{thm:consistency}). We shall now focus on the direct approach.

\begin{figure}
\centering
    \includegraphics[width=0.23\textwidth]{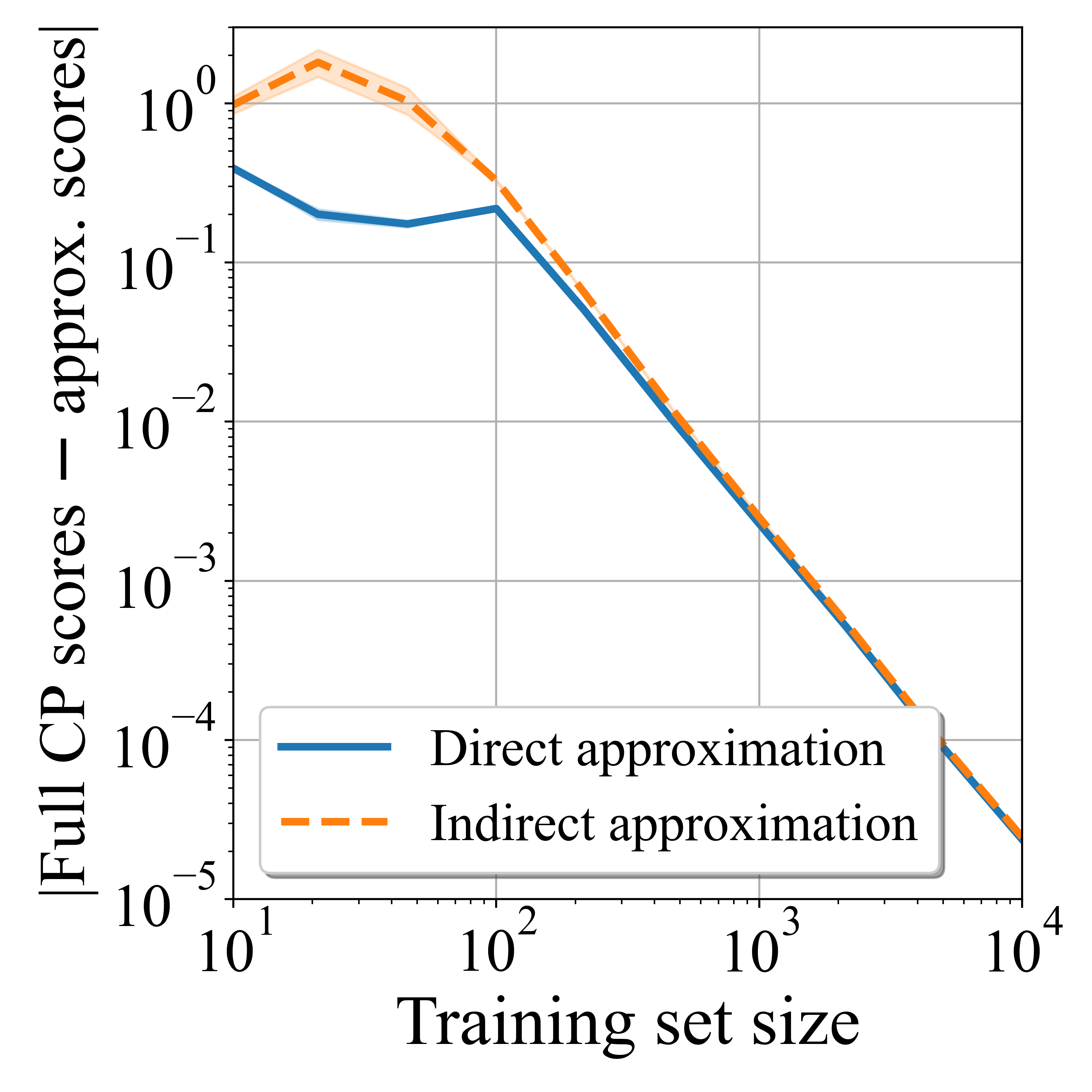}
  \caption{Comparison between direct and indirect approximations. We show the difference between the nonconformity scores of full CP and their approximation as a function of the training set size, averaged across 100 test points. These results support that the direct method is never worse than the indirect (\Cref{thm:approx-loss-vs-model}), and that both
  approximate full CP increasingly better (\Cref{thm:consistency}).
  The standard deviation for the direct approach (blue) is negligible.}
  \label{fig:methods_comparison_norm}
\end{figure}

\paragraph{Approximation goodness.}
\label{goodness}
We evaluate how well \ifcp{} approximates full CP,
under various parameter choices, as the
training set grows.
\Cref{fig:nonconf_features} shows the difference
between the nonconformity scores of full CP and \ifcp{}
as the number of features ranges in 5-100.
The number of features does impact the
IF approximation, although the error becomes
negligible as the training set increases.
\Cref{thm:approx-regularizer} shows that, unsurprisingly,
a larger $\reg$ (i.e.,
stronger regularization) implies better approximation.
We confirm this in \Cref{fig:nonconf_reg}.
Our analysis focuses on the approximation error between nonconformity
scores;
yet, we remark that a small error between scores implies
a more fundamental equivalence between full CP and \ifcp{}:
their p-values should also have a small distance.

\begin{figure*}[t]
\begin{subfigure}{.24\textwidth}
  \centering
  \includegraphics[width=0.99\linewidth]{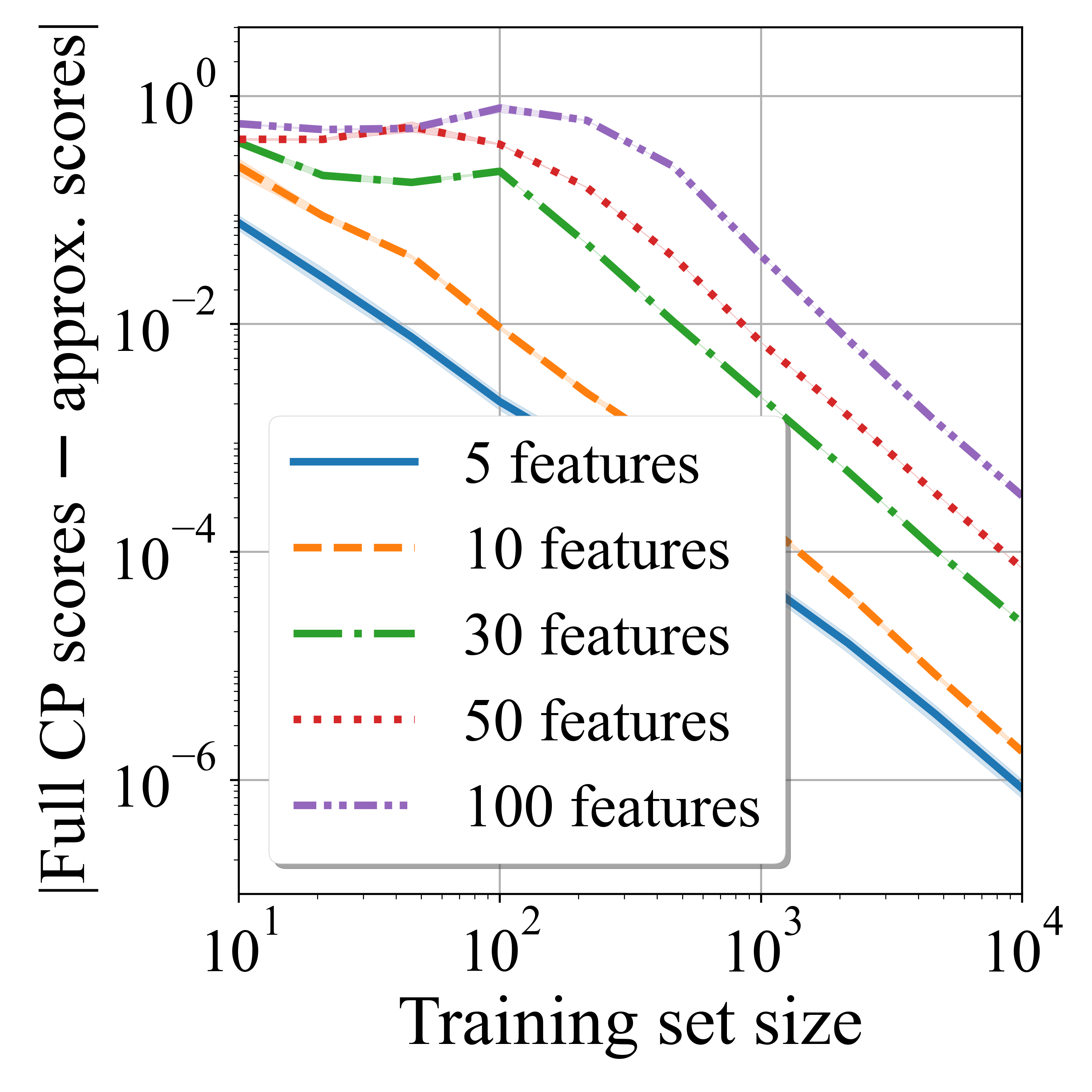}
  \caption{ Vary $\#$ of features}
  \label{fig:nonconf_features}
\end{subfigure}
\begin{subfigure}{.24\textwidth}
  \centering
  \includegraphics[width=0.99\linewidth]{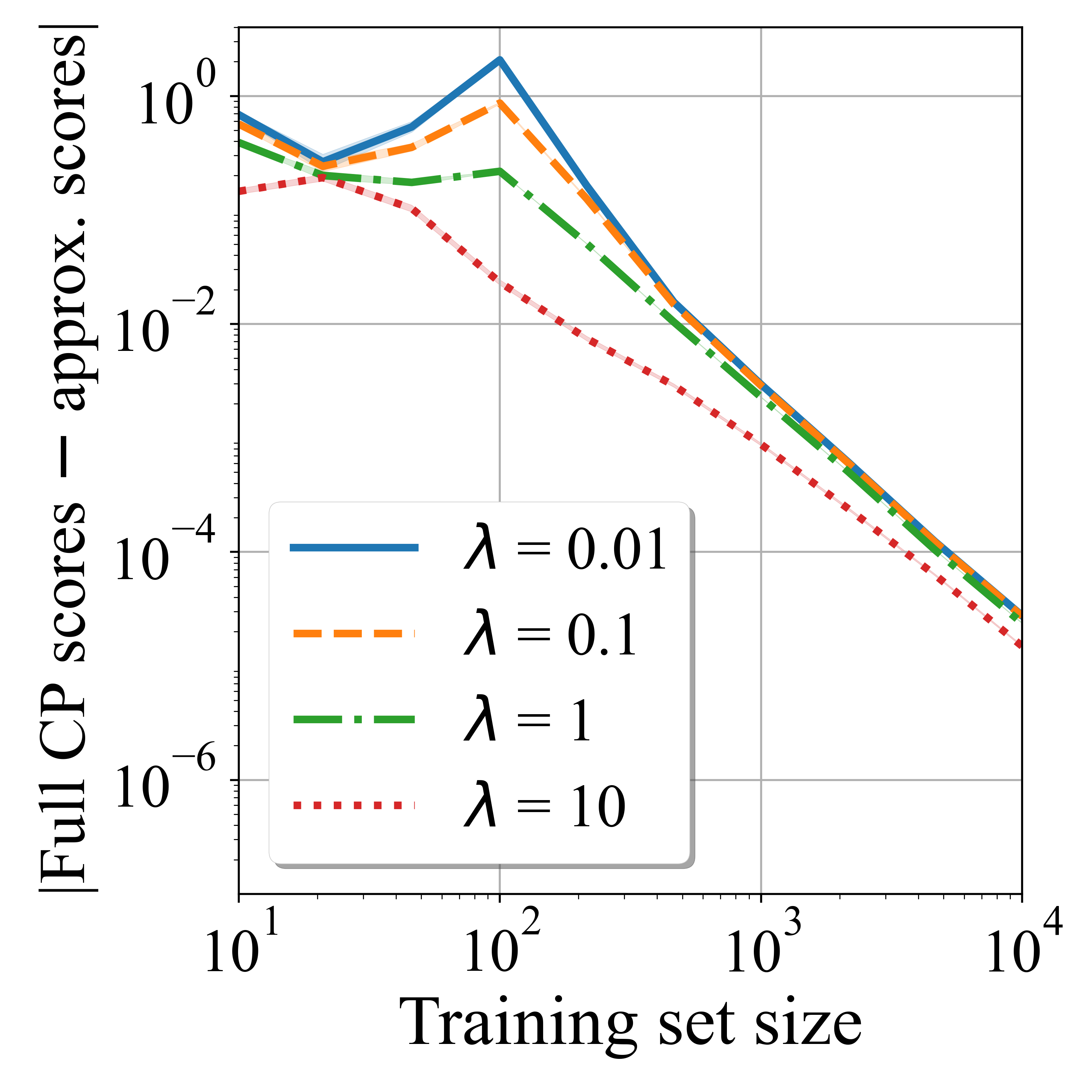}
  \caption{ Vary regularization}
 \label{fig:nonconf_reg}
\end{subfigure}
\begin{subfigure}{.24\textwidth}
  \centering
  \includegraphics[width=0.99\linewidth]{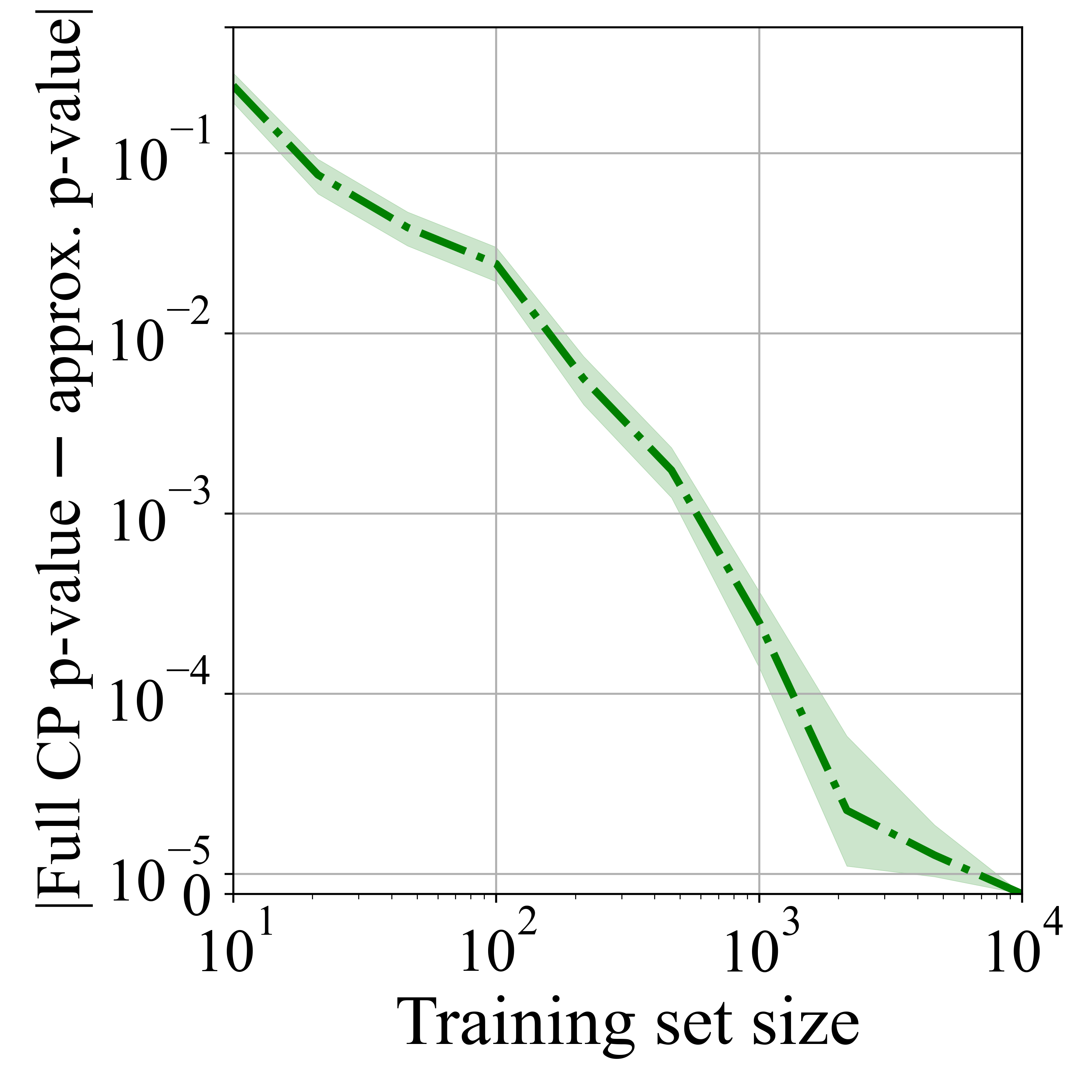}
  \caption{ P-values difference}
 \label{fig:p-values_30_features}
\end{subfigure}
\begin{subfigure}{.24\textwidth}
  \centering
  \includegraphics[width=0.99\linewidth]{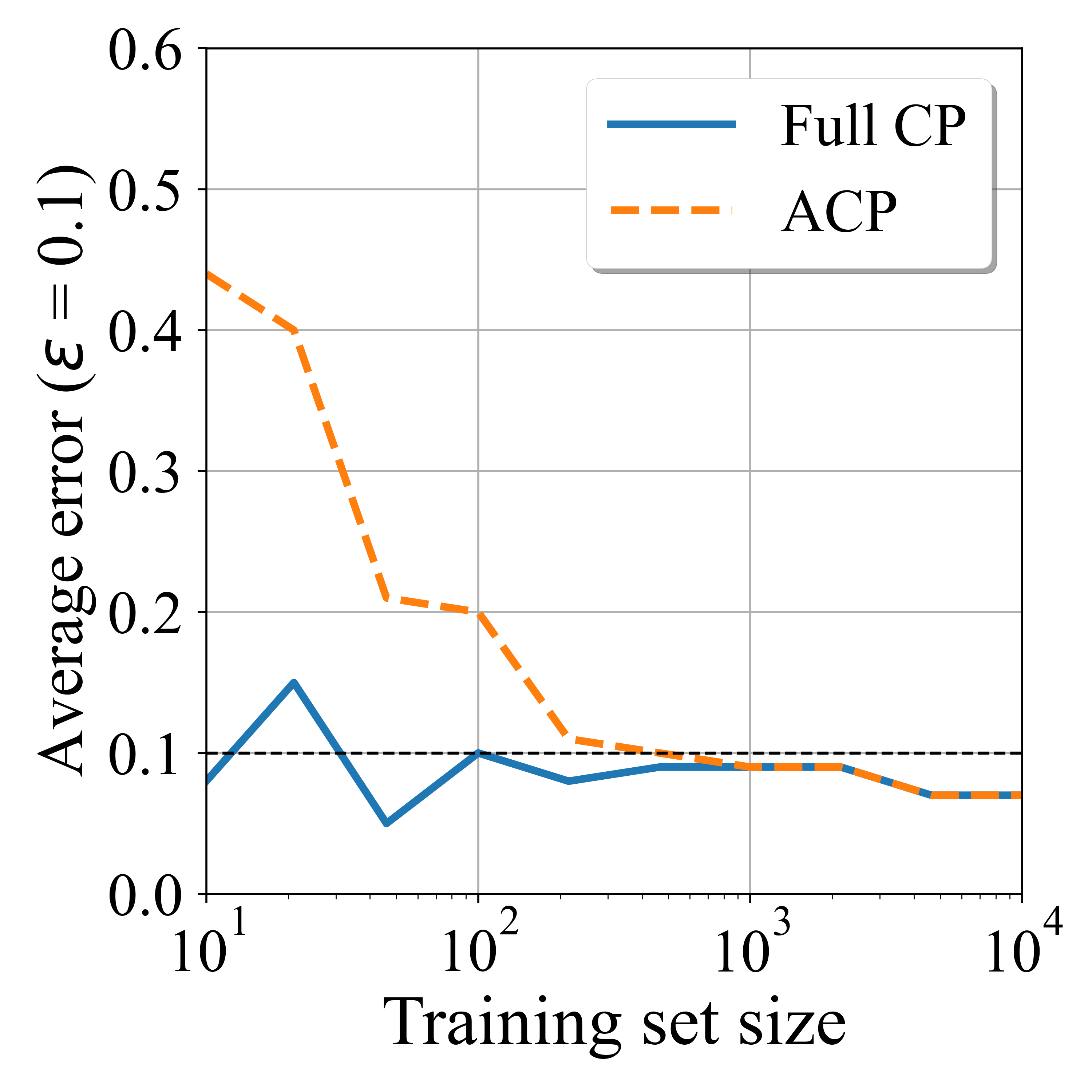}
  \caption{ Validity}
  \label{fig:validity}
\end{subfigure}
\caption{ Difference in nonconformity scores for different numbers of features (\ref{fig:nonconf_features}) and for various regularization strengths (\ref{fig:nonconf_reg}). The differences in scores are averaged across 100 test points.
We also report the difference between p-values (\ref{fig:p-values_30_features}) and error rate
(\ref{fig:validity}).
}
\label{fig:syn}
\end{figure*}

\Cref{fig:p-values_30_features} compares
 full CP and \ifcp{}'s p-values.
The difference is smaller than
$10^{-3}$ with a training set of $600$,
and it becomes negligible with $N = 10k$ training
examples.
Observe that in CP the p-value is thresholded
by the significance value $\varepsilon$
to obtain a prediction 
(Algorithm 1
).
As practitioners are generally interested
in values $\varepsilon$ with no more than 2 decimals of precision (e.g., 
$\varepsilon=0.15$),
we argue that an approximation error smaller than
$10^{-3}$ between p-values is more than sufficient
for any practical application.
\Cref{fig:validity} compares the error rate
(for $\varepsilon=0.1$) between full CP and \ifcp{}.
We observe that, after 500 training points,
the two methods have the same error.

\section{Experiments with real data}
\label{sec:experiments}

We compare mainstream CP alternatives with \ifcp{}
on the basis of their predictive power (efficiency).
Because of its computational complexity, it
is infeasible
to include full CP in these experiments.
Nevertheless, given the size of the training
data, the consistency of \ifcp{}
(\Cref{thm:consistency}), and the results in \Cref{goodness}, we expect \ifcp{} to perform similarly to full CP.

\subsection{Existing alternatives to Full CP}

There are several alternative approaches to CP for classification. In this work, we compare \ifcp{} with:

\begin{itemize}
\item Split (or ``inductive'') Conformal Prediction (SCP) \citep{papadopoulos2002inductive} works by dividing the training set into \textit{proper training set} and \textit{calibration set}. The model is fit on the proper training set, and the calibration set is used to compute the nonconformity scores.
\item Regularized Adaptive Prediction Sets (RAPS) \citep{Angelopoulos2021UncertaintySF}, is a regularized version of Adaptive Prediction Sets (APS) \citep{Romano2020ClassificationWV}.
APS constructs \textit{Generalized inverse quantile conformity scores} from a calibration set adaptively w.r.t. the data distribution. RAPS uses regularization to minimize the prediction set size while satisfying validity.
\item Cross-validation+ (CV+) \citep{Romano2020ClassificationWV} exploits a cross-validation approach while constructing the conformity scores similarly to APS. CV+ does not lose predictive power due to a data-splitting procedure, but it is computationally more expensive.
\end{itemize}

\paragraph{Datasets.}
We select datasets to illustrate the performance of \ifcp{} in various scenarios: a simple classification problem with images (MNIST~\cite{lecun1998mnist}), a more complex setting (CIFAR-10~\citep{Cifar}), and a binary classification with tabular data (US Census~\cite{ding2021retiring});
details in \Cref{appendix:datasets}.

\begin{figure*}
\centering
\begin{minipage}[t!]{.38\textwidth}
    \begin{subfigure}[t]{\textwidth}
    \begin{center}
    \begin{subfigure}{.29\textwidth}
        \includegraphics[width=\linewidth]{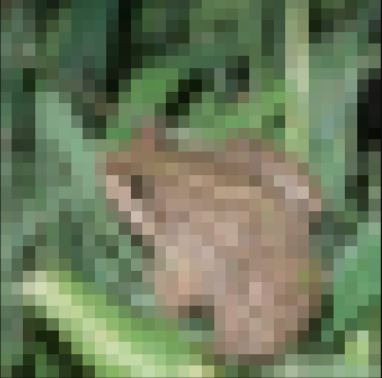}
    \end{subfigure}
    \end{center}
    \begin{subfigure}[!ht]{.26\textwidth}
    \begin{center}
    \scriptsize 
        \begin{tabular}{ll}
            \toprule
             Method & Prediction set \\ \midrule
             ACP (D) & bird, cat, deer, dog, \underline{\textbf{frog}}, horse \\ 
             ACP (O) & plane, bird, cat, deer, dog, \underline{\textbf{frog}}, horse \\
             SCP &  plane, bird, cat, deer, dog, \underline{\textbf{frog}}, horse \\ 
             RAPS &  plane, bird, cat, deer, dog, \underline{\textbf{frog}}, horse, truck \\ 
             CV+ &  plane, auto, bird, cat, deer, dog, \underline{\textbf{frog}}, horse \\ 
         \end{tabular}
    \end{center}
    \end{subfigure}
    \caption{ Test image and prediction set ($\varepsilon=0.05$)}
    \label{fig:motivating-example-table}
    \end{subfigure}
\end{minipage}%
\begin{minipage}[t!]{.31\textwidth}
    \begin{subfigure}[t]{\textwidth}
      \includegraphics[width=\linewidth]{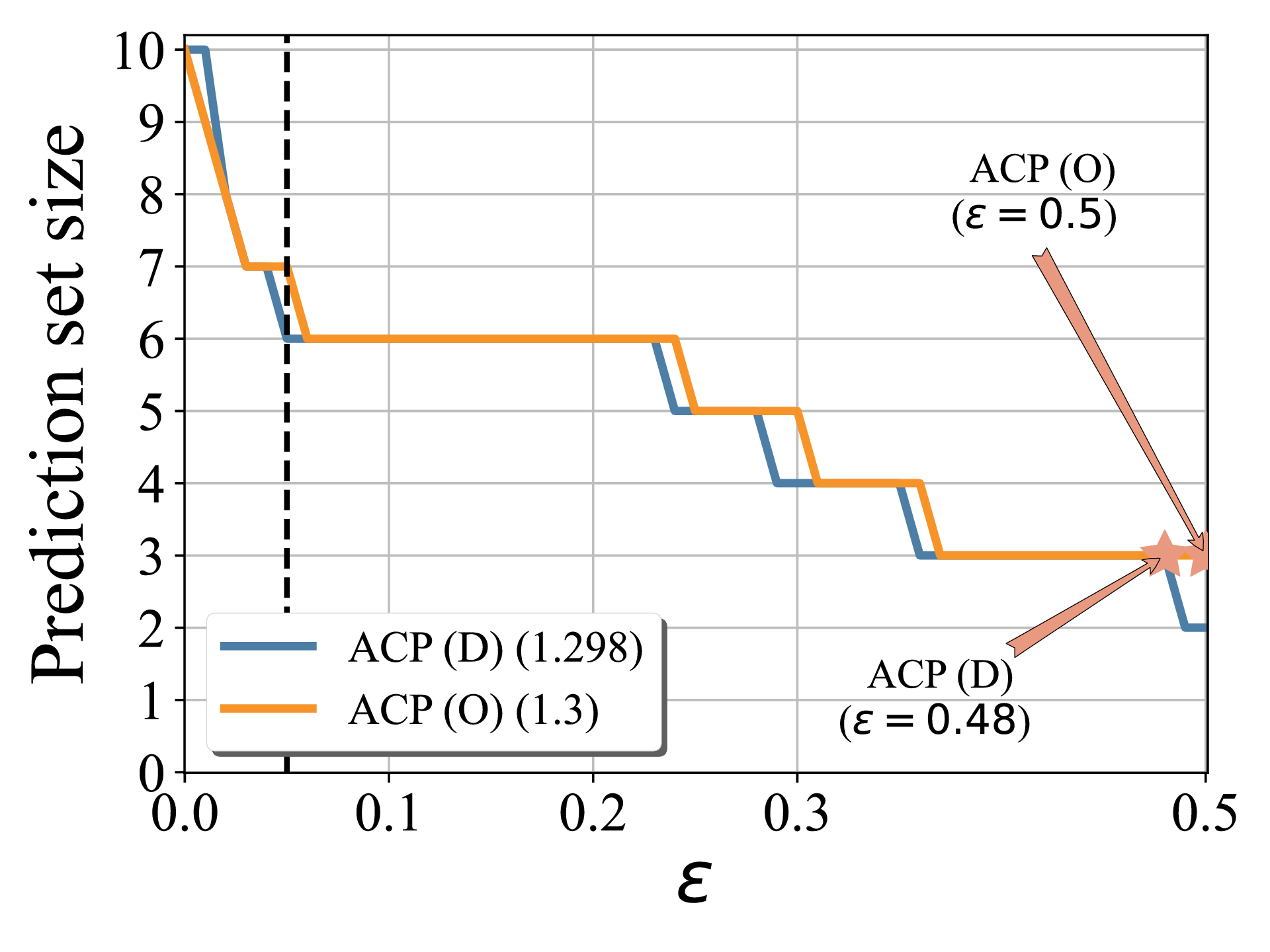}
    \caption{ Prediction set size (\ifcp{})}
    \label{fig:motivating-example-acp}
    \end{subfigure}
\end{minipage}%
\begin{minipage}[t!]{.31\textwidth}
    \begin{subfigure}[t]{\textwidth}
      \includegraphics[width=\linewidth]{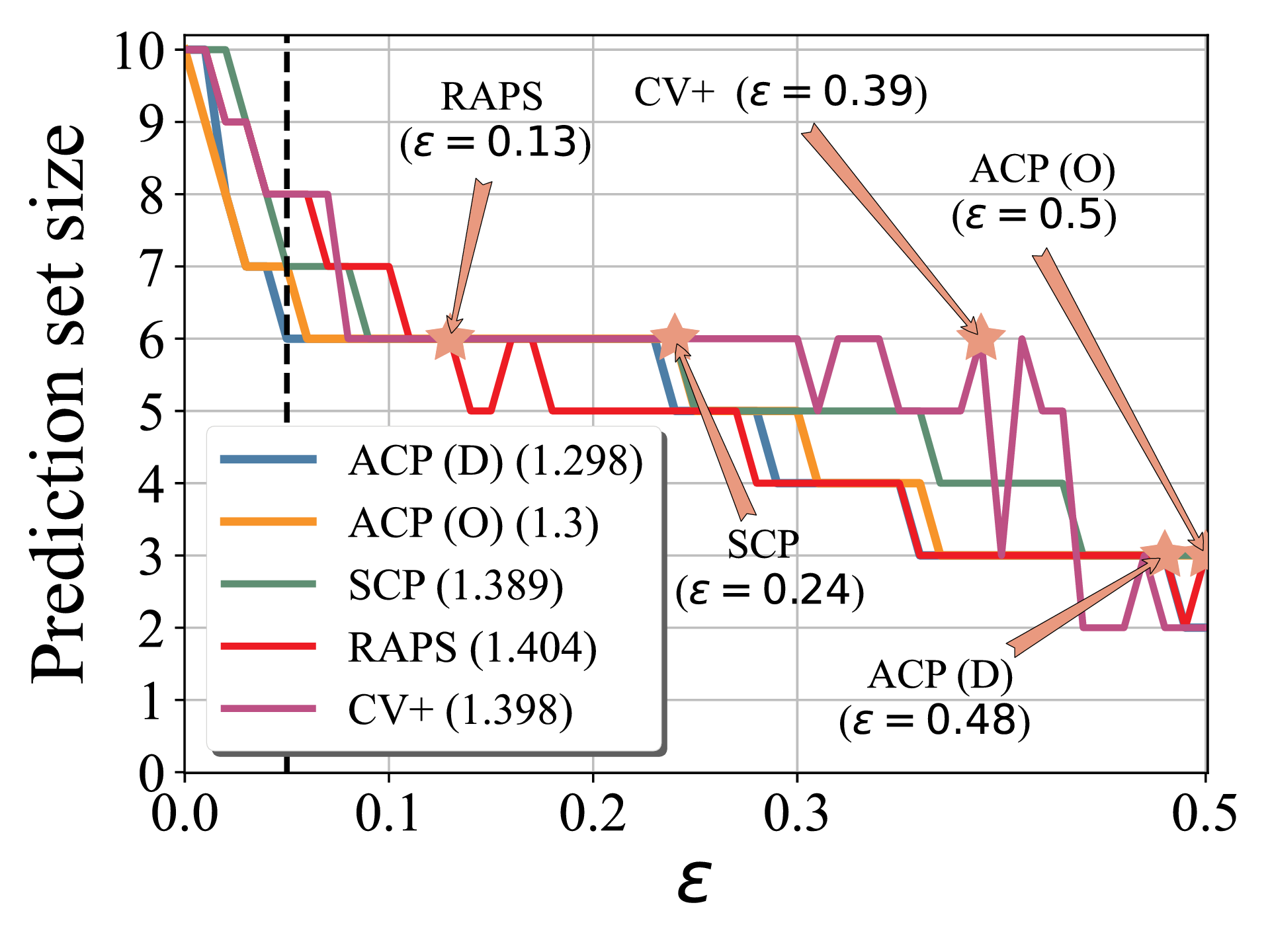}
    \caption{ Prediction set size (all)}
    \label{fig:motivating-example-all}
    \end{subfigure}
\end{minipage}%
\caption{Prediction set for a fixed
    $\varepsilon$ (\ref{fig:motivating-example-table}), and
    prediction set size w.r.t.
    $\varepsilon$ for \ifcp{} (\ref{fig:motivating-example-acp})
    and comparing methods (\ref{fig:motivating-example-all}).
    For each method,
    $\textcolor{star}{\bigstar}$ is the highest
    $\varepsilon$ for which the prediction set includes
    the true label; higher is better. We also show the AUC in the interval $\varepsilon \in [0, 0.2]$; lower is better.
    }
\label{fig:motivating-example}
\end{figure*}

\subsection{A warm-up example}

We run an illustrative experiment
for a CIFAR-10 test point picked uniformly at random.
We consider a neural network with
3 layers of 100, 50, and 20 neurons;
we refer to this network as \mlpC. \Cref{appendix:implementations} gives implementation details.

\Cref{fig:motivating-example-table} shows
the prediction set for a fixed $\varepsilon=0.05$;
we observe that, while all the methods output
the true label, the prediction set of
\ifcp{} (deleted) is the tightest
(i.e., more efficient).
We also report how the prediction set size
changes w.r.t. $\varepsilon$,
for \ifcp{} (\Cref{fig:motivating-example-acp})
and for all methods (\Cref{fig:motivating-example-all}).
As a way of comparing the curves, we
include the AUC for the interval
$\varepsilon\in[0,0.2]$.
ACP (deleted and ordinary) have the smallest AUC.
Finally, for each method we report the highest
$\varepsilon$ for which the prediction set
contains the true label.
A higher value indicates that, for this
test example, the method
would still be accurate with an $\varepsilon$ larger than $0.05$,
which would correspond to tighter prediction
sets.
Once again, \ifcp{} (deleted and ordinary)
have the largest values. 
We show similar instances in \Cref{appendix:additional:examples}.
In the next part, we observe this behavior
generalizes to larger test sets.

We observe an unstable behavior in the predictions of RAPS and CV+: their prediction
set size considerably oscillates as $\varepsilon$
increases.
The reason is that their prediction sets are not guaranteed
to be nested; that is, $\varepsilon > \varepsilon'$
does not imply that the prediction sets $\Gamma^\varepsilon \subseteq \Gamma^{\varepsilon'}$.
Specifically, because RAPS and CV+ use randomness to
decide whether to include a label in
the set, the true label may appear
and then disappear for a smaller significance level.
This may not be desirable in some practical
applications.\footnote{RAPS allows a non-randomized version, although with a more conservative behavior and considerably larger prediction sets.}
The prediction set for \ifcp{} and SCP monotonically decreases with $\varepsilon$, by construction.

\subsection{Experimental setup}
We evaluate the methods for five
underlying models: three multilayer perceptrons with architectures (neurons per layer): 20-10 (\mlpA), 100 (\mlpB) and 100-50-20 (\mlpC);
logistic regression (LR); and a convolutional neural network (CNN). In experiments with MNIST and CIFAR-10, the dimensionality is first reduced with an autoencoder (AE) in all settings except the CNN. We defer implementation details
to \Cref{appendix:implementations}-\ref{designs}.

Considering these five settings enables
comparing
the methods both for underparametrized
regimes (e.g., LR) and for better performing models (e.g., CNN).
Note that CP's guarantees hold regardless of
whether the underlying model is misspecified.
Further, observe that most of these models
are non-convex, where the ERM optimization problem does not have a unique solution;
this contradicts the IF assumption
(\Cref{sec:back}).
Nonetheless, our empirical results show that \ifcp{} works well -- it performs better than
the other proposals;
in \Cref{sec:discussion} we discuss
relaxations of this assumption.

\subsection{Results}
For each experiment and method, we report averaged metrics over 100 test points;
we also run statistical tests to check if differences are (statistically) significant.

\paragraph{Prediction set size.}
We measure efficiency as the
average prediction set size.
We report this as a function of
$\varepsilon$, discretized with a step $\Delta \varepsilon = 0.01$.
\Cref{fig:curve_MNIST,fig:curve_CIFAR-10}
show the average prediction set size
in MNIST and CIFAR-10
for \mlpC.
\ifcp{} (deleted and ordinary)
consistently outperform all other methods;
deleted is better than ordinary.
\ifcp{} is significantly more efficient than RAPS and CV+.
SCP and \ifcp{} (O) perform similarly on
MNIST, but their difference is remarked on
CIFAR-10; this suggests that SCP is a cheap
effective alternative to \ifcp{}
on relatively easier tasks.
\Cref{appendix:additional_curves} reports
results for the rest of the
models and for the US Census, showing similar behavior.

\begin{figure}[htb!]
        \centering
        \begin{subfigure}[hb]{0.21\textwidth}
             \includegraphics[width=\linewidth]{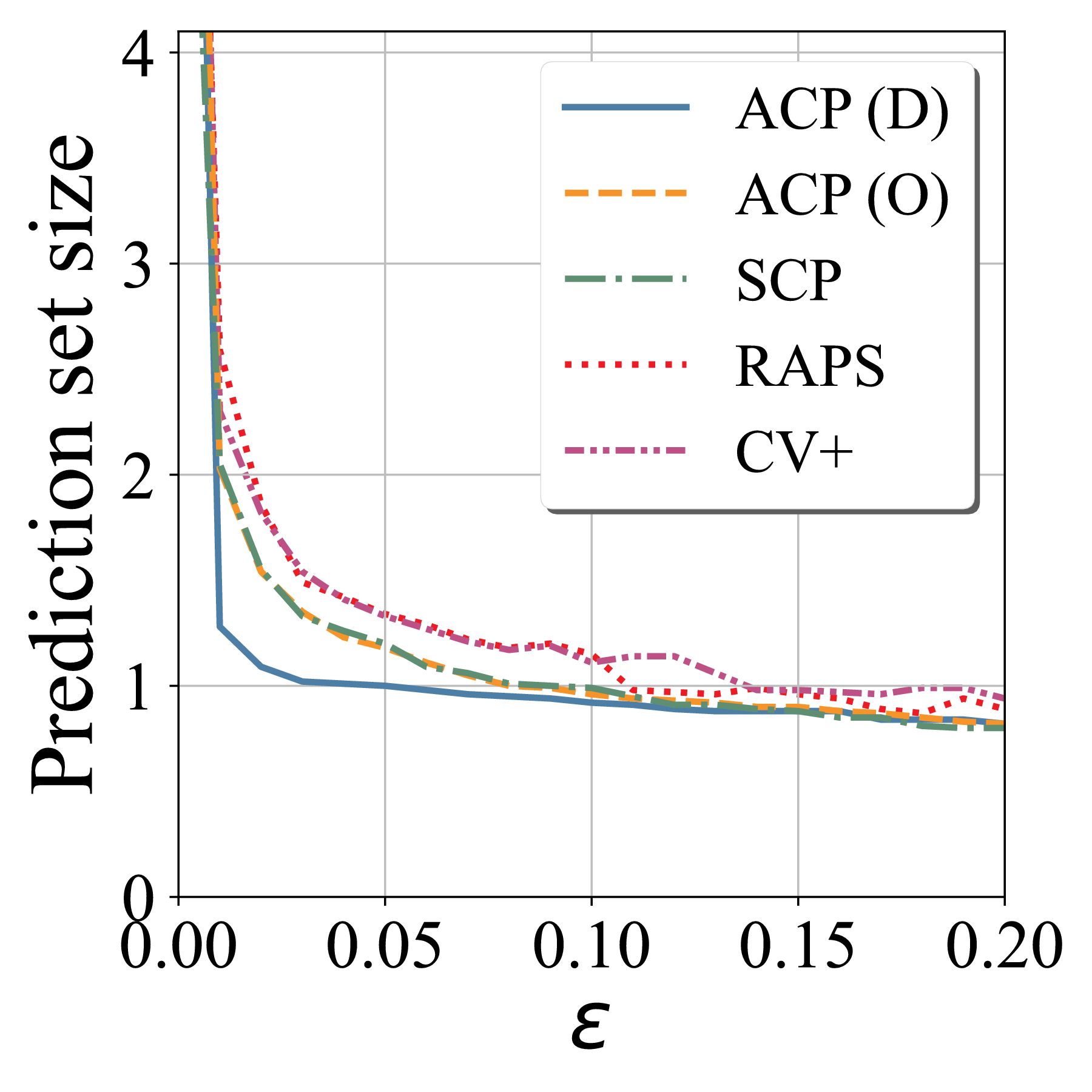}
                \caption{MNIST} \label{fig:curve_MNIST}
        \end{subfigure}
        \begin{subfigure}[hb]{0.21\textwidth}
                \includegraphics[width=\linewidth]{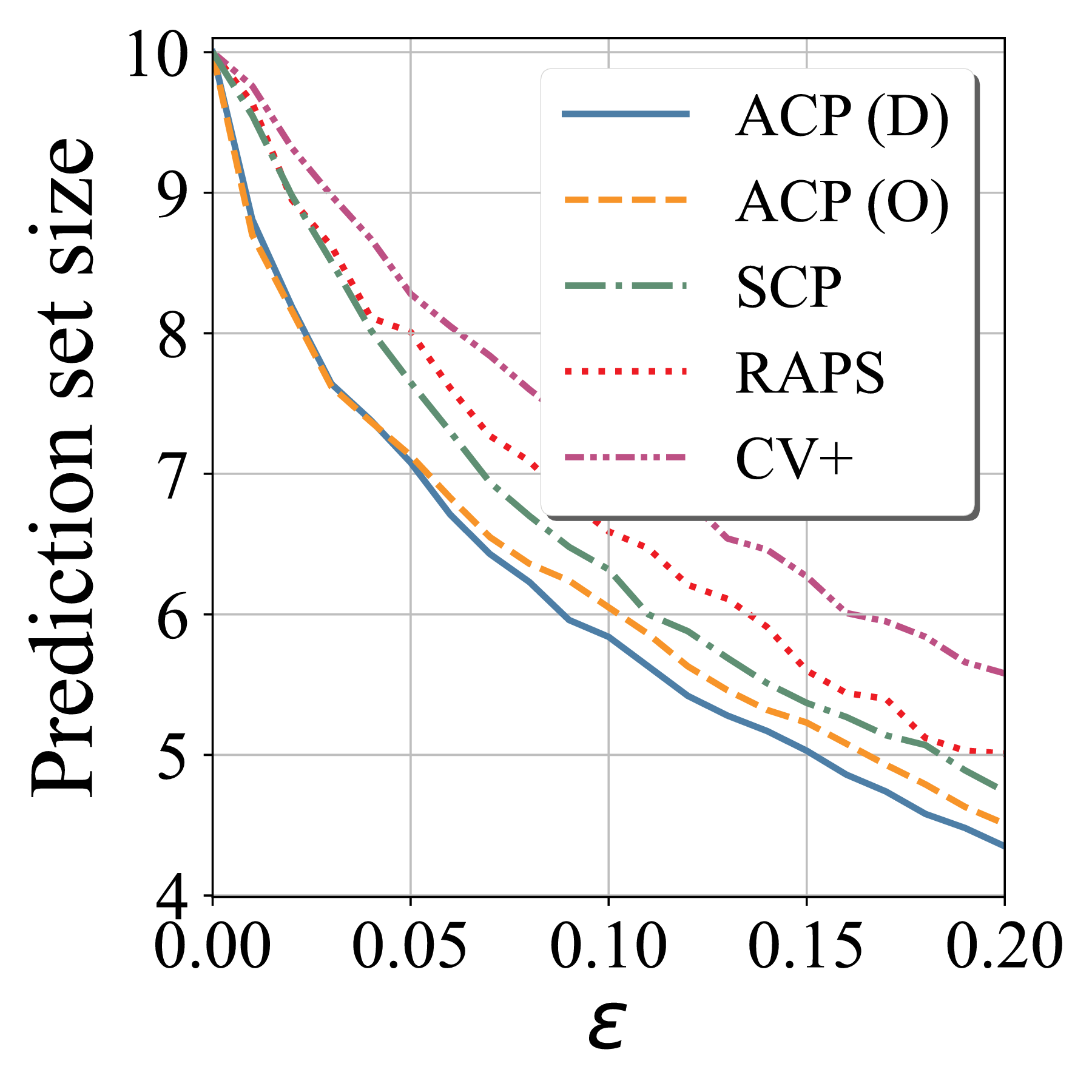}
                \caption{CIFAR-10} \label{fig:curve_CIFAR-10}
        \end{subfigure} 

        \caption{Average prediction set size as a function of the significance level $\varepsilon$ in MNIST (\ref{fig:curve_MNIST}) and CIFAR-10 (\ref{fig:curve_CIFAR-10}) for \mlpC.
        Smaller prediction set size indicates better efficiency.
        }
        \label{fig:global_sizes_C} 
\end{figure}

\begin{table}[!htb]
    \centering
    \begin{footnotesize}
    \begin{tabular}{l|lllll}
    \multicolumn{1}{c}{}  & \multicolumn{5}{c}{US Census}\\  \midrule 
     Model & ACP (D) & ACP (O) & SCP & RAPS & CV+  \\ \midrule
     \mlpA & $.301$ & $\mathbf{.281}$ & $.302^\dagger$ & $.318^{* \dagger}$ & $.374^{* \dagger}$ \\
     \mlpB & $.306$ & $\mathbf{.304}$ & $.351^{* \dagger}$ & $.351^{* \dagger}$ & $.377^{* \dagger}$ \\
     \mlpC & $\mathbf{.273}$ & $.275$ & $.280$ & $.299^{* \dagger}$ & $.324^{* \dagger}$ \\
     LR & $.276$ & $.276$ & $.284^{* \dagger}$ & $\mathbf{.183}$ & $.344^{* \dagger}$ \\
    \bottomrule
    \multicolumn{6}{c}{}\\[-1.5ex]
     \multicolumn{1}{c}{} & \multicolumn{5}{c}{MNIST}\\  \midrule 
    Model & ACP (D) & ACP (O) & SCP & RAPS & CV+  \\ \midrule
   \mlpA & $\mathbf{.252}$ & $.253$ & $.261$ & $.299^{* \dagger}$ & $.322^{* \dagger}$ \\
    \mlpB & $\mathbf{.220}$ & $.230$ & $.233$ & $.266^{* \dagger}$ & $.277^{* \dagger}$ \\
    \mlpC & $\mathbf{.198}$ & $.231$ & $.230^*$ & $.258^{* \dagger}$ & $.267^{* \dagger}$ \\
    LR   & $.385$ & $.386$ & $\mathbf{.379}$ & $.438^{* \dagger}$ & $.467^{* \dagger}$ \\
    CNN  & $\mathbf{.175}$ & $.182$ & $.237^{* \dagger}$ & $.197^{* \dagger}$ & $.199^{* \dagger}$ \\
    \bottomrule
    \multicolumn{6}{c}{}\\[-1.5ex]
    \multicolumn{1}{c}{} & \multicolumn{5}{c}{CIFAR-10}\\  \midrule 
     Model & ACP (D) & ACP (O) & SCP & RAPS & CV+  \\ \midrule
     \mlpA & $1.261$ & $\mathbf{1.259}$ & $1.281^{* \dagger}$ & $1.311^{* \dagger}$ & $1.373^{* \dagger}$ \\
     \mlpB & $1.385$ & $\mathbf{1.364}$ & $1.397^\dagger$ & $1.416^{* \dagger}$ & $1.514^{* \dagger}$ \\
     \mlpC & $\mathbf{1.226}$ & $1.250$ & $1.327^{* \dagger}$ & $1.377^{* \dagger}$ & $1.475^{* \dagger}$ \\
     LR   & $\mathbf{1.409}$ & $1.411$ & $1.419$ & $1.436^{* \dagger}$ & $1.476^{* \dagger}$ \\
     CNN  & $\mathbf{.976}$ & $1.108$ & $1.019$ & $1.110^{*}$ & $1.467^{* \dagger}$ \\
    \bottomrule
    \end{tabular}
    \end{footnotesize}
    \caption{
    Efficiency AUC ($\varepsilon\in[0, 0.2]$).
    The table indicates that differences in the AUC are statistically significant as compared to ACP (D) ($*$) and ACP (O) ($\dagger$).}
    \label{tab:results_CIFAR_MNIST_CENSUS}
\end{table}

\paragraph{Efficiency AUC.}
We use an $\varepsilon$-independent
metric to aid this comparison:
the area under the curve (AUC)
of the prediction set size in the
interval $\varepsilon \in [0, 0.2]$;
a smaller AUC means better efficiency.
\ifcp{} (deleted or ordinary) prevails
on all methods, datasets, and model combinations (\Cref{tab:results_CIFAR_MNIST_CENSUS}).
An exception is
LR on US Census, where RAPS has a better efficiency than
\ifcp{} and SCP;
simpler tasks and
models may be well served by the computationally
efficient RAPS.
Welch one-sided tests (reject with p-value $<0.1$)
confirm that both deleted and ordinary \ifcp{}
are better than RAPS and CV+;
they further show that either
deleted or ordinary \ifcp{} are
better than SCP on most tasks.
We refer to \Cref{sec:discussion}
for directions to improve
\ifcp{}'s performance.

\paragraph{Validity.}
As a way of interpreting why \ifcp{}
performed better than the other methods,
we measure their empirical error rate with a fixed $\varepsilon=0.2$.
Indeed, whilst
all methods guarantee a probability of
error of at most
$\varepsilon$, a more conservative
(i.e., smaller) empirical error may correspond to
larger prediction sets.
\Cref{fig:all_errors_CIFAR_A} shows the difference between the expected and the observed error rate on CIFAR-10 (i.e., $\varepsilon-\hat{\varepsilon}$).
We observe that \ifcp{} achieves an error rate
close to the significance level in most
cases; this indicates this method fully exploits
its error margin to optimize efficiency.\footnote{Note a randomized version of full CP, called smooth CP, ensures the error rate is exactly $\varepsilon$ (instead of ``at most $\varepsilon$'').
Herein, we use the standard, conservatively valid definition
per 
Algorithm 1.
}
SCP shows a similar behavior.
On the other hand, RAPS and CV+ are more conservative and tend to commit fewer errors. This means that RAPS and CV+ generate unnecessarily large prediction sets that, despite including the true label on average, result in lower efficiency.  This behavior is consistent with analogous experiments on the
US Census and MNIST datasets.

\begin{table}[htb!]
\centering
\footnotesize
\begin{tabular}{l|rrrrr}
\toprule
 Model  & ACP (D) & ACP (O) & SCP & RAPS & CV+  \\ \midrule
\mlpA  & $0.01$ & $0.01$ & $0$ & $0.05$ & $0.09$ \\ 
\mlpB   & $0.02$ & $0.03$ & $0.04$ & $0.05$ & $0.13$ \\ 
\mlpC & $-0.01$ & $0$ & $0.02$ & $-0.03$ & $0.09$ \\ 
LR    & $0$ & $0.03$ & $-0.01$ & $0.02$ & $0.04$ \\
CNN & $-0.01$ & $0.01$ & $-0.01$ & $0.04$ & $0.10$\\
\toprule
\end{tabular}
        \caption{ Difference between the expected
            error ($\varepsilon=0.2$) and the empirical
            error rate on CIFAR-10. A positive value indicates a conservative prediction set: the method could output tighter prediction sets
            without losing validity;
            a negative value is a validity violation, which may be due to statistical fluctuation (100 test points).
}
        \label{fig:all_errors_CIFAR_A}
\end{table}

\begin{table}
    \centering
    \setlength{\tabcolsep}{3pt}
    \begin{footnotesize}
    \begin{tabular}{l|lll|lll}
      \multicolumn{1}{c}{} & \multicolumn{3}{c}{MNIST} & \multicolumn{3}{c}{CIFAR-10}\\  \midrule 
     Model & ACP (D) & ACP (O) & SCP &
        ACP (D) & ACP (O) & SCP\\ \midrule
     \mlpA & $\mathbf{.057}$ & $.059$ & $.068^*$ & $\mathbf{1.796}$ & $1.832$ & $1.849$ \\
     \mlpB & $\mathbf{.032}$ & $.041$ & $.043^{*}$ & $\mathbf{1.967}$ & $2.050$ & $2.132^{*}$ \\
     \mlpC & $\mathbf{.012}$ & $.040$ & $.044^{*}$ & $\mathbf{1.728}$ & $1.843$ & $1.878^*$ \\
     LR & $.187$ & $.189$ & $\mathbf{.183}$ & $\mathbf{2.269}$ & $2.280$ & $2.332$ \\
     CNN & $\mathbf{.002}$ & $.002$ & $.002$ & $\mathbf{1.186}$ & $1.645$ & $1.235$ \\
    \toprule
    \end{tabular}
    \end{footnotesize}
    \caption{
    Fuzziness of \ifcp{} and SCP on MNIST and CIFAR-10.
    A smaller fuzziness corresponds to more
    efficient prediction sets. The table indicates that differences in fuzziness are statistically significant as compared to ACP (D) ($*$) and ACP (O) ($\dagger$).}
    \label{tab:fuzz_CIFAR_MNIST}
\end{table}

\paragraph{P-values fuzziness.}
Average prediction set size is a coarse
criterion of efficiency:
small variations in the algorithm may
lead to substantially different prediction sets.
An alternative efficiency criterion is fuzziness~\citep{vovk2016criteria}, which measures how
p-values distribute in the label space.
Of the methods we considered, fuzziness
can only be measured for \ifcp{} and SCP
(RAPS and CV+ do not produce
p-values).
For every test point $x$, and computed p-values
$\{p_{(x, \hat{y})}\}_{\hat{y}\in \labelspace}$,
fuzziness is the sum of the p-values minus the
largest p-value:
$\sum_{\hat{y}\in \labelspace} p_{(x, \hat{y})}
    - \max_{\hat{y}\in \labelspace} p_{(x, \hat{y})}$;
a small fuzziness is desirable.
In \Cref{tab:fuzz_CIFAR_MNIST}, we observe
that \ifcp{} (D) has a consistently better
fuzziness than \ifcp{} (O) and SCP.
We observe statistical significance of the
results for 3 models in MNIST, and
2 models in CIFAR-10.
Results for the US Census (\Cref{appendix:fuzziness}) are similar.
As previously observed in \Cref{tab:results_CIFAR_MNIST_CENSUS},
SCP performs better in logistic regression (MNIST);
this makes it a good alternative for
simpler underlying models.

\section{Related work}

Full CP~\citep{vovk2005algorithmic} is notoriously expensive. Many alternatives have been
proposed.
Arguably the most prominent are SCP~\cite{vovk2005algorithmic},  CV+~\cite{Romano2020ClassificationWV}, RAPS~\citep{Angelopoulos2021UncertaintySF},
Cross-CP~\cite{vovk2012cross}, aggregated CP~\cite{carlsson2014aggregated}, APS~\citep{Romano2020ClassificationWV}, and
the jackknife~\citep{Miller1974TheJR, Efron1979BootstrapMA}.
We compared ACP with the former three.
Unlike full CP, all the above methods have weaker validity guarantees or they tend to attain
less efficiency.

\citet{pmlr-v139-cherubin21a} introduced exact optimizations
for full CP for classification.
Their method saves an order of magnitude in
time complexity for many models, but is alas only
applicable to models supporting incremental
and decremental learning (e.g., k-NN).
Crucially, it is unlikely extendable to neural networks

The closest in spirit to our approach
is the Discriminative Jackknife (DJ)
by \citet{alaa2020discriminative},
which uses IF to approximate jackknife+
confidence intervals for regression~\cite{barber2021predictive}.
We
note several differences between ACP and DJ,
besides their different goals (classification
vs regression).
While DJ approximates LOO w.r.t. the parameters, we
introduce a \textit{direct approach} to approximate
the nonconformity scores in~\Cref{eq:score-approx-loss}. 
Whereas DJ only does decremental learning, which allows them to exploit Hessian-Vector-Products~\citep{pearlmutter1994fast},
approximating full CP requires us to do both
incremental \textit{and} decremental learning (\Cref{algo:ifcp}). 
We also prove that our approximation error
decreases w.r.t. the size of the training set.

\section{Conclusion and Future Work}
\label{sec:discussion}
Full CP is a statistically sound method for providing performance guarantees on the outcomes of ML models. For classification tasks, CP generates prediction sets which contain the true label with a user-specified probability.
Unfortunately, full CP is impractical to run for
more than a few hundred training points.
In this work, we develop \ifcp{}, a computationally efficient method
which approximates
full CP
via influence functions;
this strategy
avoids the numerous recalculations that full CP requires. We prove that \ifcp{} is consistent: it approaches full CP as the training set grows. Our
experiments support the use of \ifcp{} in practice.

There are many directions 
to improve ACP.
For example, we assumed that the ERM solution is unique. While our approximation works well in practice, it would be a fruitful endeavour to relax this  
assumption; initial work towards this was done for IF
by \citet{koh2017understanding}. 
Another direction 
is to
build nonconformity scores on the
\textit{studentized} scheme, a middle way
between deleted and ordinary~\cite{vovk2017nonparametric};
this may further improve \ifcp{}'s prediction power. Future work can investigate when \Cref{thm:approx-loss-vs-model} holds,
and they can try to apply \Cref{thm:consistency} to get direct error bounds on the coverage gap.

Finally, while we scale full CP to relatively large models, computing and inverting the Hessian becomes
very expensive as the number of parameters increases. Recent tools to approximate the Hessian, like the Kronecker-factored
Approximate Curvature (K-FAC) method~\citep{martens2015optimizing, ba2017distributed, tanaka2020pruning},
might help further scale \ifcp{} to larger models like ResNets.

In conclusion,
\ifcp{} 
helps scaling full CP to large datasets and ML models, for which running full CP would be impractical. Although split-based approaches like SCP and RAPS are less expensive to run, they do not attain the same efficiency as \ifcp{}. This makes the adoption of our method particularly appealing for critical real-world applications.

\section*{Acknowledgments}
\label{sec:ack}
 JA acknowledges support from the ETH AI Center. UB acknowledges support from DeepMind and the Leverhulme Trust via the Leverhulme Centre for the Future of Intelligence (CFI), and from the Mozilla Foundation. AW acknowledges support from a Turing AI Fellowship under grant EP/V025279/1, The Alan Turing Institute, and the Leverhulme Trust via CFI. GC acknowledges support from The Alan Turing Institute. The authors are grateful to Volodya Vovk, Pang Wei Koh and Manuel Gomez Rodriguez for useful comments and pointers.
 
\bibliography{main}
\appendix
\onecolumn
\input{appendix.tex}

\end{document}

%% file: figures/algo-full-cp.tex
\begin{algorithm}[H]
    \caption{Full CP}
	\label{algo:cp}
	\begin{algorithmic}[1]
	\FOR {$x$ in test points}
	    \FOR {$\hat{y} \in \labelspace$}
		    \STATE $\hat{z} = (x, \hat{y})$
    		\FOR{$z_i \in \zset \cup \{\hat{z}\}$}
    		  
    	   \small{\STATE \textcolor{oldalgo}{$\modelncm = \argmin_{\hat{\model} \in \modelspace} \risk(\zset\cup\{\hat{z}\}\setminus\{z_i\}, \hat{\model})$}
    			\STATE \textcolor{oldalgo}{$\score_i = \loss(z_i, \modelncm)$}}
    		\ENDFOR
    		\STATE $\pval_{(x, \hat{y})} = \frac{\card\{ i=1, ..., N+1 \suchthat \score_i \geq \score_{N+1}\}}{N+1}$
    	\STATE If $\pval_{(x, \hat{y})} > \varepsilon$, include
    	$\hat{y}$
    	in set $\predset^\varepsilon_x$
    	\ENDFOR
    \ENDFOR
	\end{algorithmic}
\end{algorithm}

%% file: figures/algo-our-method.tex
\begin{algorithm}[H]
    \centering
    \caption{Approximate full CP (ACP)}
\label{algo:ifcp}
	\begin{algorithmic}[1]
	\STATE \textcolor{algochange}{$\model_{\zset} = \argmin_{\hat{\model} \in \modelspace} \risk(\zset, \hat{\model})$}
	\FOR {$x$ in test points}
	    \FOR {$\hat{y} \in \labelspace$}
    	    \STATE $\hat{z} = (x, \hat{y})$
    		\FOR{$z_i \in \zset \cup \{\hat{z}\}$}
    		\small \STATE \textcolor{algochange}{$\ifscore_i =$ Approximate via Eq.~\ref{eq:score-approx-model} or \ref{eq:score-approx-loss}}
    		\ENDFOR
    	\STATE $\pval_{(x, \hat{y})} = \frac{\card\{ i=1, ..., N+1 \suchthat \tilde{\score}_i \geq \tilde{\score}_{N+1}\}}{N+1}$
    	\STATE If $\pval_{(x, \hat{y})} > \varepsilon$, include
    	$\hat{y}$
    	in set $\predset^\varepsilon_x$
		\ENDFOR
	\ENDFOR
	\end{algorithmic}
\end{algorithm}

%% file: appendix.tex
\section{Additional information about ACP}
\subsection{Ordinary Full CP and \ifcp{}}
\label{sec:ordinary-cp}

\begin{algorithm}[H]
	\begin{algorithmic}[1]
	\FOR {$x$ in test points}
	    \FOR {$\hat{y} \in \labelspace$}
		    \STATE $\hat{z} = (x, \hat{y})$
		    \STATE $\modelzhat = \argmin\limits_{\hat{\model} \in \modelspace} \risk(\zset \cup \{\hat{z}\}, \hat{\model})$
    		\FOR{$z_i \in \zset \cup \{\hat{z}\}$}
    		\STATE $\score_i = \loss(z_i, \modelzhat)$
    		\ENDFOR
    		\STATE $\pval_{(x, \hat{y})} = \frac{\card\{ i=1, ..., N+1 \suchthat \score_i \geq \score_{N+1}\}}{N+1}$
    	\STATE If $\pval_{(x, \hat{y})} > \varepsilon$, include
    	$\hat{y}$
    	in prediction set $\predset^\varepsilon_x$
    	\ENDFOR
    \ENDFOR
	\end{algorithmic}
	\caption{Full CP - Ordinary}
	\label{algo:cp-ordinary}
\end{algorithm}

Algorithm~\ref{algo:cp-ordinary} shows the Full CP
algorithm constructed for a nonconformity measure
defined on an ordinary scheme.
\ifcp{} with the ordinary prediction can
be defined similarly to the deleted one
as follows.
Once again, we can use the direct or indirect
approach.

\paragraph{Indirect approach.}
Let
$\ifmodelzhat \equiv \model_{\zset}
+\influence_{\model_\zset}(\hat{z})$.
Then:
\begin{equation}
    \label{eq:score-approx-model-ordinary}
    \score_i \approx \loss(z_i, \ifmodelzhat) \,.
\end{equation}

\paragraph{Direct approach.}
\begin{align}
    \label{eq:score-approx-loss-ordinary}
    \score_i
    &\approx
        \ifloss(z_i, \modelzhat) \nonumber \\
    &\equiv
        \loss(z_i, \model_\zset) +
        \influence_\loss(z_i, \hat{z})
\end{align}

Observe that the only difference w.r.t.
\ifcp{} (deleted) is that here we only need to
add the influence of $\hat{z}$; that is,
we do not need to remove the influence of
$z_i \in \zset \cup \{\hat{z}\}$.
This spares the very costly LOO procedure of
retraining the model, which makes this method
significantly more computationally efficient.

\subsection{Time complexity:  Full CP Vs ACP}

Predicting $m$ test points via  full CP
in an $l$-label classification setting 
costs $\bigoh{nlm(T_n+P)}$,
where $T_n$ is the cost of training a model on
$n$ training points, and $P$ is the cost of making a prediction with it.

The \ifcp{} procedure is split into two steps: training
and prediction.
Training costs
$\bigoh{T_n + H_n + n(G_n+P)}$,
where $H_n$ and $G_n$ are the costs of computing the Hessian inverse and a gradient, respectively. We denote the dependency on the training set size with the subscript $n$. 
Computing a prediction for $m$ points
costs $\bigoh{lm(G_n+nI_n+P)}$, where $I_n$ is the (usually small)
cost of computing the influence;
this gives an order speed-up over  CP.

\begin{table}[htb!]
    \centering
    \begin{normalsize}
    \begin{tabular}{lll}
    \toprule
          & \textbf{Train} & \textbf{Predict}\\ \midrule
         Full CP & N/A & $nlm(T_n+P)$ \\
         \textbf{ACP}  & $T_n + H_n + n(G_n+P)$ & $lm(G_n+nI_n+P)$ \\
         \bottomrule
    \end{tabular}
    \end{normalsize}
    \caption{Time complexities for full CP and ACP.}
    \label{tab:complexities}
\end{table}
As an illustrative example, let us assume a standard classification problem with $l$ labels, a training set with size $N$, and a smaller test set $M$, where we fit a 3-layers neural network. The  full CP algorithm would retrain the model $l\times N$ times per \textbf{each} test point in $M$. This is infeasible for relatively large datasets with, e.g., more than $1000$ points, either in the training or the test set. Conversely, ACP  only needs to train the model once, compute the Hessian inverse and the $N$ gradients (training step). Although inverting the Hessian requires $\bigoh{W^3}$ operations, this is expected to be significantly faster than full CP's retraining procedure for this standard scenario. ACP only needs access to the gradient of the test point to inexpensively compute the IF (prediction step). 

We measured the running times in the synthetic dataset (\Cref{app:exp}) with 50 features,  1000 training points, and 100 test points. As indicative figures, running full CP in this setting took approximately 5 days, whereas ACP took less than 1 hour (CPU times with an Intel Core i7-8750H). 

\subsection{Space complexity}

ACP requires storing the Hessian inverse, i.e., $\bigoh{W^2}$ with $W$ being the number of parameters, and one gradient per training point, i.e., $\bigoh{nW}$.  Full CP stores the full training set.

\subsection{Additional discussion on the nonconformity score}

Any nonconformity measure $\ncm: (\objspace\times\labelspace) \times (\objspace\times\labelspace)^N \rightarrow \mathbf{R}$ defines a conformal predictor (CP); given an example $x \in \objspace$ and a significance level $\varepsilon$, this predictor generates a  set that contains the true label $y \in \labelspace$ with probability $1-\varepsilon$. The intuition here is that $y$ will have a value that makes the pair $(x, y)$ conform with the previous examples; the validity is therefore conditioned on the \textit{training set}. The nonconformity score  is simply a measure of how different a new example is from old examples. In a general setting, the nonconformity measure is arbitrary; the validity guarantees always hold, e.g., a random score will still generate valid prediction sets, although likely with low efficiency. Whether a function $\ncm$ is an appropriate nonconformity score will always be open to discussion since it greatly depends on the context~\citep{vovk2005algorithmic}. 

 \ifcp{} (Algorithm 2) restricts the nonconformity score to be the loss at a point. Although the instance-wise loss might be suboptimal, the underlying algorithm is arbitrary. \citet{shafer2008tutorial} point out that the score choice is relatively
unimportant and that the critical step in the CP framework is determining the underlying classifier. It is therefore not expected
that ACP degrades its efficiency by using the loss as score as long as the wrapped classifier is appropriate.

\section{Proofs}
\label{proofs}

\subsection{\ifcp{} consistency}

We first declare the following set of assumptions introduced and
thoroughly discussed by \citet{giordano2019swiss}.

\begin{assumption}
\label{assumption:giordano}
~
\begin{itemize}
	\item For all $\model \in \modelspace$ and all $z_i \in \zset$,
		$\nabla_\model \loss(z_i, \model)$ is continuously
		differentiable in $\model$.
    \item For all $\model \in \modelspace$, the Hessian
		$H_\model$ is non-singular with
		$\sup_{\model \in \modelspace} ||H_\theta^{-1}||_{op} < \infty$.
	\item $\exists$ finite constants $C_g, C_h$ s.t.
			$\sup\limits_{\model \in \modelspace} \frac{1}{\sqrt{N}}
				||\nabla_\model\loss(z, \model)||_2 \leq
					C_g$ and
			$\sup\limits_{\model \in \modelspace} \frac{1}{\sqrt{N}}
			||\nabla^2_\model\loss(z, \model)||_2 \leq
			C_h$.
	\item 
	Let $h(\model) = \nabla^2_{\model}\loss(z, \model)$.
	There exists $\Delta_\model > 0$ and a finite constant $L_h$ such that $||\model-\modelz||_2 \leq \Delta_\model$
	implies that $\frac{||h(\model)-h(\modelz)||_2}{\sqrt{N}} \leq
	L_h ||\model - \modelz||_2$.
\end{itemize} 
\end{assumption}

To prove \Cref{thm:consistency} we exploit the following
result by \citet{giordano2019swiss}:
Let $\ifmodelloo = \modelz - \influence_\model(z_i, z_i)$
be the approximation of $\modelloo$ obtained from
\Cref{eq:if-model}.

\begin{theorem}[Consistency of $\ifmodelloo$ \citep{giordano2019swiss}]
\label{thm:giordano-consistency}
Under \Cref{assumption:giordano},
for every $N$
there is a constant $C$ s.t., for every
$z_i \in \zset$:
\begin{align*}
	||\ifmodelloo - \modelloo||_2 & \leq C \frac{||g||^2_\infty}{N^2}\\
		& \leq C \frac{\max \{C_g, C_h\}^2}{N}
\end{align*}

(Note, the result by Giordano et al. is stated more generally for
the case of leave $k$ out; we only report its LOO version.
Also, the general case requires 2 further assumptions,
which we omit as they are always satisfied for LOO.)
\end{theorem}

We now restate and prove the consistency of
\ifcp{}.

\consistency*
\begin{proof}
First, from \Cref{thm:giordano-consistency} \citep{giordano2019swiss} we have that, for every $N$
there exists constants $C$ and $C'$ such that:
\begin{align*}
    &||\modelzhat -\influence_\model(z_i) - \modelncm||_2 \leq C \frac{\max \{C_g, C_h\}^2}{N} \\
    &||\modelzhat - \influence_\model(\hat{z}) - \modelz||_2 \leq C' \frac{\max \{C_g, C_h\}^2}{N}\,;
\end{align*}
observe that the second expression is obtained by
removing $\hat{z}$ from the training data of
a model trained on $\zset \cup \{\hat{z}\}$.

Fix $z_i \in \zset \cup \{\hat{z}\}$.
\begin{align*}
    |\ifloss(z_i, \modelncm) - \loss(z_i, \modelncm)|
    &\leq
        |\loss(z_i, \ifmodelncm) - \loss(z_i, \modelncm)|\\
    &= |\loss(z_i, \ifmodelncm) - \loss(z_i, \modelncm)|\\
    &\leq K||\ifmodelncm - \modelncm||_2\\
    &= K||\modelz + \influence_\model(\hat{z}) -
        \influence_\model(z_i) -
        \modelncm||_2\\
    &\leq K ||\modelzhat -\influence_\model(z_i) - \modelncm||_2 + ||\modelzhat -\influence_\model(\hat{z}) - \modelz ||_2 \\
    &\leq K(C+C')\frac{\max \{C_g, C_h\}^2}{N}
\end{align*}
First step applies \Cref{thm:approx-loss-vs-model},
third step uses Lipschitz continuity,
fourth step uses triangle inequality,
and the final step uses \Cref{thm:giordano-consistency}.
\end{proof}

\subsection{Dependence on the regularization parameter $\reg$}
\label{appendix:theory-regularization}

Our result on the effect of the regularization
parameter on the approximation error of \ifcp{}
(\Cref{thm:approx-regularizer}) is an extension
of a result by \citet{koh2019accuracy}
on the approximation error of IF for LOO.

\subsubsection*{Background on result by \citet{koh2019accuracy}}

The LOO loss can be
written as
$$\loss(z_i, \modelloo) = \loss(z_i, \modelz) - \influence_\loss^*(z_i, z_i) \approx \loss(z_i, \modelz) - \influence_\loss(z_i, z_i) \,,$$
where $\influence_\loss^*(z_i, z_i)$ is defined to
be the true influence
of training point $z_i \in \zset$ on the loss of the model at $z_i$.
For simplicity, we omit the arguments of
$\influence_\loss$ when clear from the context.
In the following, we assume $\model$ is the minimizer of
an ERM problem with regularization parameter $\reg$.

\citet{koh2019accuracy} decompose the
influence function error
into Newton-actual error ($E_A$) and
Newton-influence error ($E_\influence$):
$$\influence^*_\loss-\influence_\loss =
    E_A + E_\influence$$
with $E_A \equiv \influence_\loss^* - \influence_\loss^\mathcal{N}$, and
$E_\influence \equiv
\influence_\loss^\mathcal{N} -
\influence_\loss$ . Here $\influence_\loss^\mathcal{N}$ is the Newton approximation of the true influence $\influence_\loss^*$.

\citet{koh2019accuracy} show that $E_A$
decreases with $\bigoh{\frac{1}{\lambda^3}}$,
and they observed
that its value tends to be negligible in practice.

\begin{proposition}[\citet{koh2019accuracy}]
\label{thm:koh-self-loss}
Assume that $\loss$ is $C_f$-Lipschitz and that the
Hessian is $C_H$-Lipschitz.
Further, assume the third derivative of $\loss$ exists and
that it is bounded in norm by $C_{f,3}$.
The influence on the self-loss is such that:
$$\influence_\loss(z_i) + E_{f,3} \leq
    \influence^\mathcal{N}_\loss(z_i) \leq
    \left(1 + \frac{3\sigma_{max}}{2\reg} +
        \frac{\sigma_{max}^2}{2\reg^2}\right)\influence_\loss(z_i)
    + E_{f,3}(z_i) \,.$$

Furthermore:
$$|E_{f,3}(z_i)| \leq
    \frac{C_{f,3}C^3_\loss}
        {6(\sigma_{min}+\reg)^3} \,.$$
\end{proposition}

\subsubsection*{Extension to \ifcp{}}

If we assume $E_A$ to be negligible, we obtain that
$\influence^*_\loss \approx \influence^\mathcal{N}_\loss$.
Further, ignoring $O(\reg^{-3})$ terms, Proposition~\ref{thm:koh-self-loss} gives:
$$\influence_\loss \leq \influence^*_\loss \leq g(\reg) \influence_\loss \,,$$
that is (by subtracting each term from $\loss(z_i, \modelz)$):
$$\loss(z_i, \modelz)-g(\reg)\influence_\loss(z_i, z_i) \leq
    \loss(z_i, \modelloo) \leq
    \loss(z_i, \modelz)-\influence_\loss(z_i, z_i) =
    \ifloss(z_i, \modelloo)\,.$$

Note, however, that this result only applies to the
LOO loss.
We use 
the following simplifying assumption to extend
the above result to \ifcp{}.

\begin{assumption}
For all $z_i$ and $\hat{z}$:
$$\influence^*_\loss(z_i, \hat{z}) \approx
\influence_\loss(z_i, \hat{z})$$
\label{assumption:influence-zhat}
\end{assumption}

We can now restate and prove our result.

\regularizer*

\begin{proof}

Koh et al. show that (ignoring $O(\lambda^{-3})$ terms):
    $$\influence_\loss(z_i) \leq
    \influence^\mathcal{N}_\loss(z_i) \leq
    \left(1 + \frac{3\sigma_{max}}{2\reg} +
        \frac{\sigma_{max}^2}{2\reg^2}\right)\influence_\loss(z_i) \,.$$
(For a definition of all the terms please refer to
\Cref{appendix:theory-regularization}.)
Ignoring $O(\reg^{-3})$ terms, we apply this to the loss:
    $$\loss(z_i, \modelz) - g(\reg)\influence_\loss(z_i, z_i) \leq
    \loss(z_i, \modelloo) \leq
    \loss(z_i, \modelz) - \influence_\loss(z_i, z_i)$$

We apply this to the augmented training set
$\zset \cup \{\hat{z}\}$ (now $z_i$ ranges in
$\zset \cup \{\hat{z}\}$):

$$\loss(z_i, \modelzhat) - g(\reg)\influence_\loss(z_i, z_i) \leq \loss(z_i, \modelncm) \leq
    \loss(z_i, \modelzhat) - \influence_\loss(z_i, z_i) \,.$$

By using Assumption~\ref{assumption:influence-zhat}, replace
    $\loss(z_i, \modelzhat)$ with
    $\loss(z_i, \modelz) + \influence_\loss(z_i, \hat{z})$.
This concludes the proof.
\end{proof}

\section{Experimental details}
\label{app:exp}

\subsection{Synthetic data (\Cref{sec:synthetic})}
We generate synthetic data for a binary classification problem using scikit-learn’s \texttt{make\_classification()} \cite{scikit-learn}. We sample points from four Gaussian-distributed clusters (two per class).

\subsection{Real data (\Cref{sec:experiments})}
\label{appendix:datasets}
Follows a brief description of the datasets we used and of their pre-processing. 

\begin{itemize}
\item \textbf{MNIST}: we use 60,000 images for training and 10,000 for testing from the MNIST dataset \cite{lecun1998mnist} – from which we take the first 100 test points. The 28×28 grayscale images represent handwritten digits between 0 and 9. In all settings except the CNN, these images are standardized in the interval $[0,1]$ by dividing the pixel intensities by $255$.
\item \textbf{CIFAR-10}:  we use 50,000 images for training and 10,000 for testing from the CIFAR-10 dataset \cite{Cifar} – from which we take the first 100 test points. These correspond to images from 10 mutually exclusively classes. We also standardize them in the interval $[0,1]$ (except in setting CNN).
\item \textbf{US Census}: we use the ACSIncome data for the state of New York from the US Census dataset \cite{ding2021retiring}. Each data point represents an individual above the age of 16, who reported usual working hours of at least 1 hour per week in the past year, and an income of at least \$100. Ten features characterize each point, and the task consists of predicting whether the individual’s income is above \$50,000. We split the dataset of 103,021 points and keep the 90\% (92,718) for training. We take 100 test points from the remaining. 
\end{itemize}

\subsection{Implementation details} \label{appendix:implementations}
We use dense layers with Rectified Linear Units (ReLU) as nonlinear activation functions in all the models. Settings \mlpA, \mlpB and \mlpC are based on multilayer perceptrons with different widths and depths, LR uses logistic regression, and CNN a convolutional neural network.
All models are trained with Adam for a maximum of 200 epochs. We fix a learning rate of 0.001, minibatches of size 100, and an early stopping based on a 20\% validation split. We also set a regularization $\lambda=10^{-5}$. We use a cross-entropy loss function.

We use an autoencoder (AE) to reduce the dimensionality of MNIST and CIFAR-10 for settings \mlpA, \mlpB, \mlpC, and LR. The AE uses a symmetric encoder-decode architecture with two dense layers of 128 and 64 neurons, separated by dropout with  $p=0.2$ and ReLU activations. 

We now give specific details of each setting:
 
\begin{itemize}
\item \textbf{\mlpA}: two layers, the first with 20 and the second with 10 neurons. The embedding size for the AE is 8.
\item \textbf{\mlpB}: one layer with 100 neurons. Embedding size of 16.
\item \textbf{\mlpC}: three layers with 100, 50, and 20 neurons. Embedding size of 32.
\item \textbf{LR}: embedding size of 8.
\item \textbf{CNN}: two convolutional layers with 16 and 32 kernels of size 5, respectively. They include a ReLU activation and max-pool layer with kernel size 4. A multilayer-perceptron makes the classification. 
\end{itemize}

\subsection{Design choices for the methods} \label{designs}
Here we comment on some design choices that we made when including the competing methods.

\begin{itemize}
\item \textbf{SCP}: we fix a calibration split of 20\%, as advised in \cite{Linusson2014EfficiencyCO}. We obtain the nonconformity scores by fitting the model in the proper training set and computing the loss for the calibration data.

\item \textbf{RAPS}: we fix a calibration split of 20\%. We use the randomized version of the algorithm, and we do not allow zero sets (i.e., we use the default version in \cite{Angelopoulos2021UncertaintySF}). The randomized version, although slightly unstable, provides a less conservative coverage. The parameters ``kreg'' and ``lamda'' are picked in an optimal fashion that minimizes the prediction set size. We therefore compare ACP to the most efficient version of RAPS.
\item \textbf{CV+}: we set $K=5$ as the number of folds for the cross-validation.
\end{itemize}

For ACP, we add a damping term of $\lambda=0.01$ to the Hessian to enforce positive eigenvalues. This is equivalent to an l2-regularisation.

\section{Fuzziness comparison}
\label{appendix:fuzziness}

\Cref{tab:fuzz_Census} shows the fuzziness
of \ifcp{} and SCP on the US Census dataset.
Results match the behavior observed for MNIST
and CIFAR-10 (\Cref{sec:experiments}).

\begin{table}[htb!]
    \centering
    \begin{tabular}{l|lll}
    \toprule
      & \multicolumn{3}{c}{US Census}\\  \midrule 
     Model & ACP (D) & ACP (O) & SCP \\ \midrule
     \mlpA & $.1086$ & $\mathbf{.0844}$ & $.1121^\dagger$ \\
     \mlpB & $\mathbf{.1144}$ & $.1366$ & $.2173^{*}$ \\
     \mlpC & $\mathbf{.0733}$ & $.0768$ & $.0741$ \\
     LR & $\mathbf{.0781}$ & $.0782$ & $.0885$ \\
    \toprule
    \end{tabular}
    \caption{Fuzziness of \ifcp{} and SCP on US Census.
    A smaller fuzziness corresponds to more
    efficient prediction sets.\\
     $^*$Differences in the fuzziness are statistically significant compared to ACP (D). $^\dagger$Differences are significant compared to ACP (O).}
    \label{tab:fuzz_Census}
\end{table}

\section{Further results}

\subsection{Additional experiments with synthetic data}
We run additional experiments on synthetic data to further support  the theory in \Cref{sec:theory}. Specifically, we show how ACP's p-values approximate Full CP's as the training set increases, for different numbers of features and regularization strengths. We also compute the Kendall tau distance between the p-values rankings.

\begin{figure}[htb]
\begin{subfigure}{.235\textwidth}
  \centering
  \includegraphics[width=0.99\linewidth]{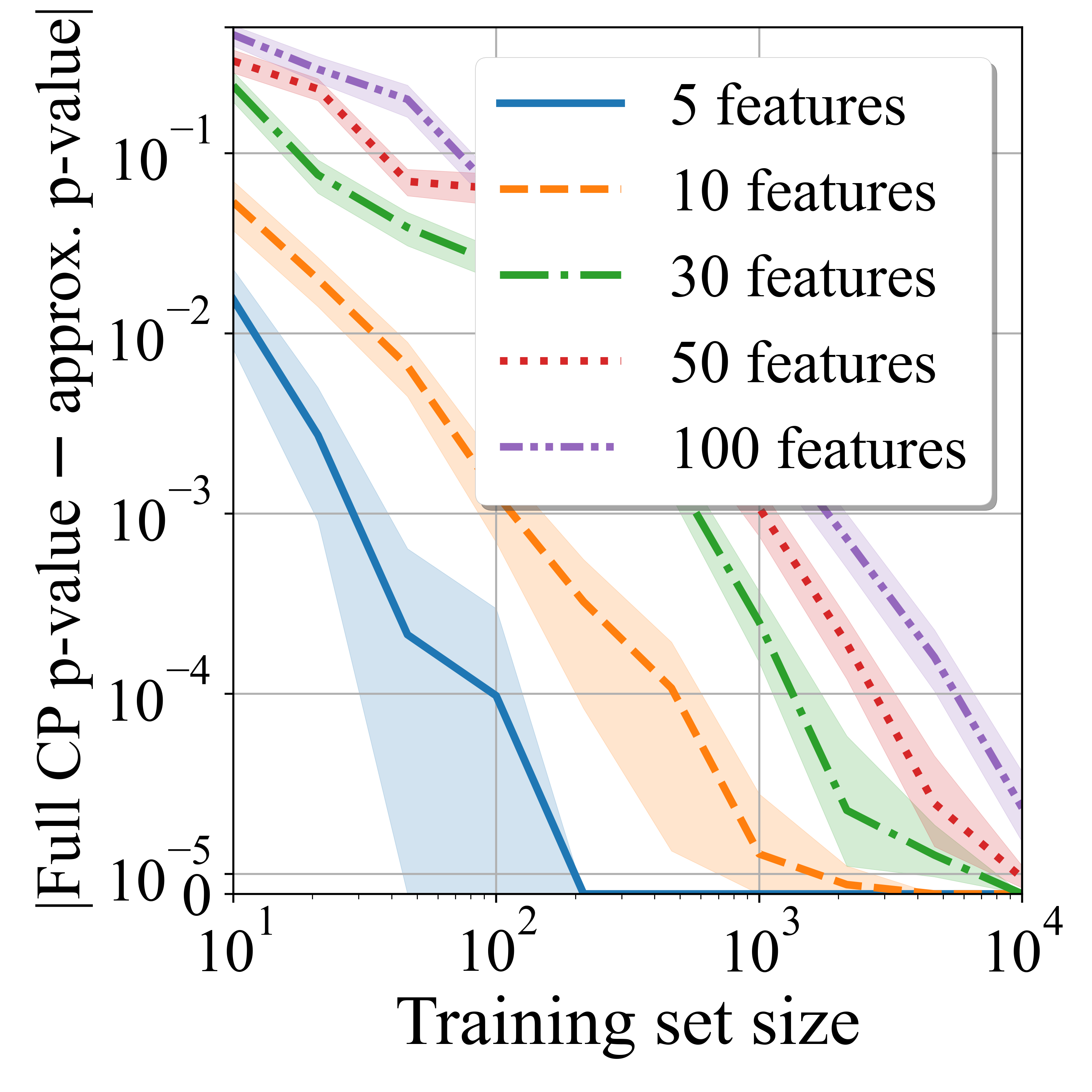}
\end{subfigure}
\begin{subfigure}{.235\textwidth}
  \centering
  \includegraphics[width=0.99\linewidth]{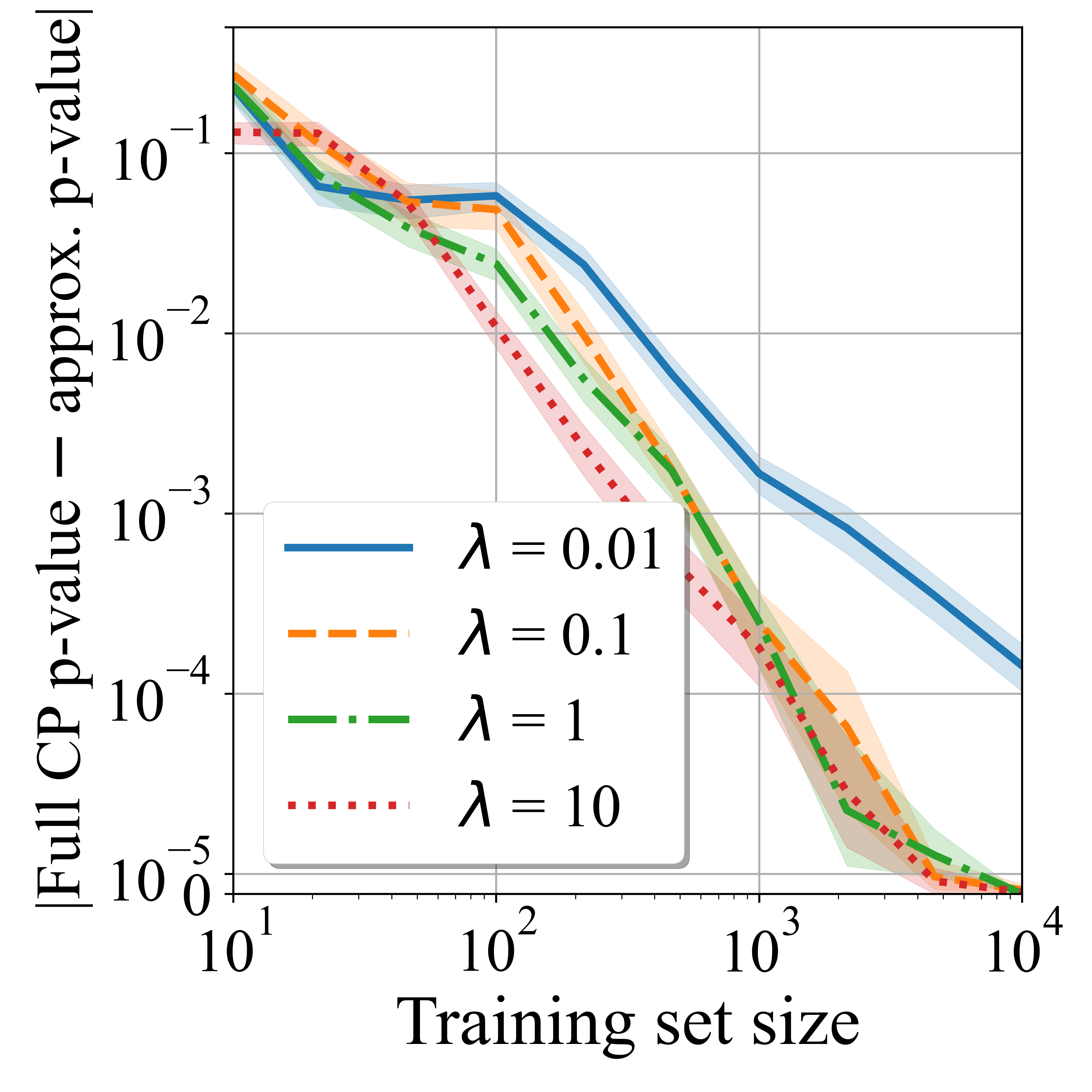}
\end{subfigure}
\begin{subfigure}{.235\textwidth}
  \centering
  \includegraphics[width=0.99\linewidth]{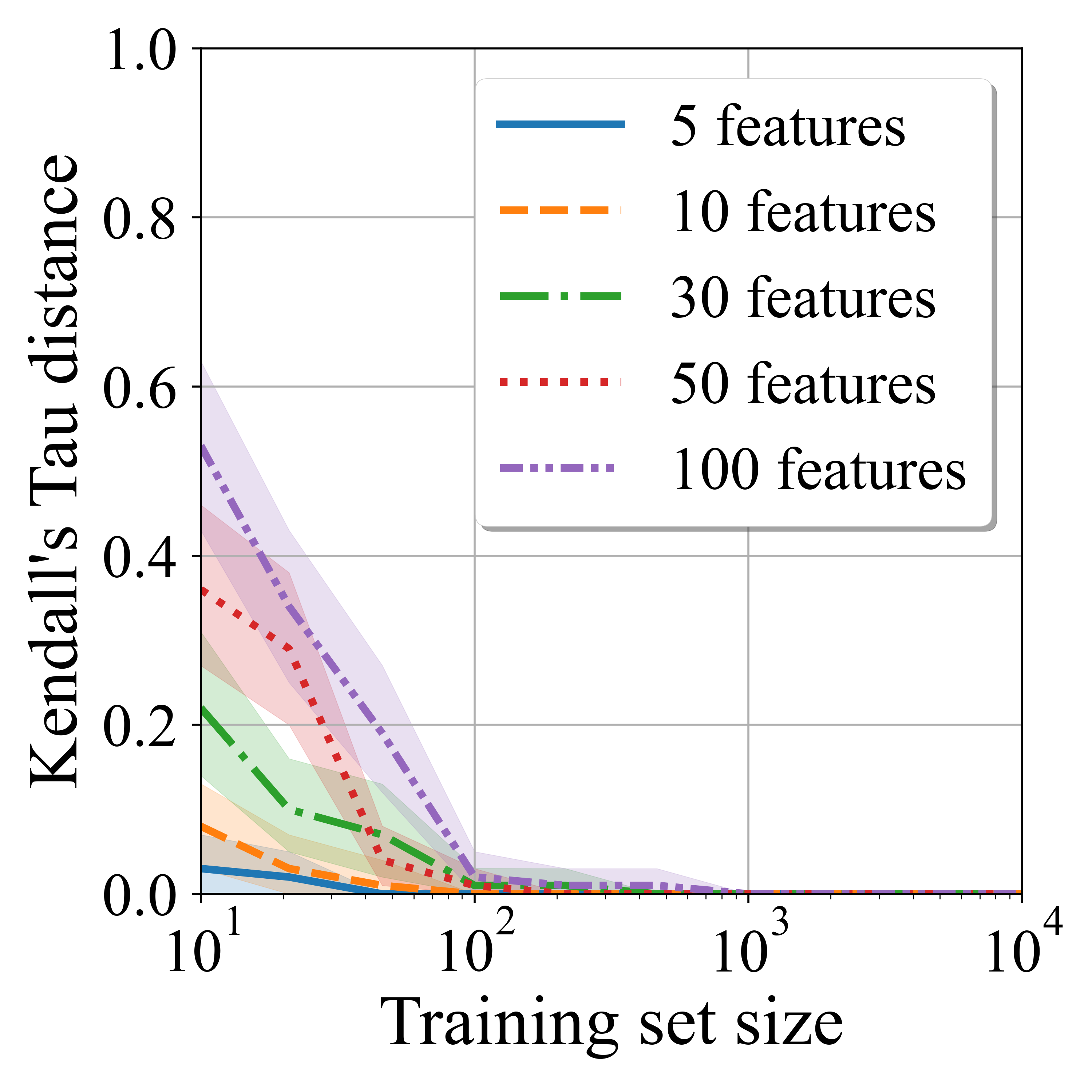}
\end{subfigure}
\begin{subfigure}{.235\textwidth}
  \centering
  \includegraphics[width=0.99\linewidth]{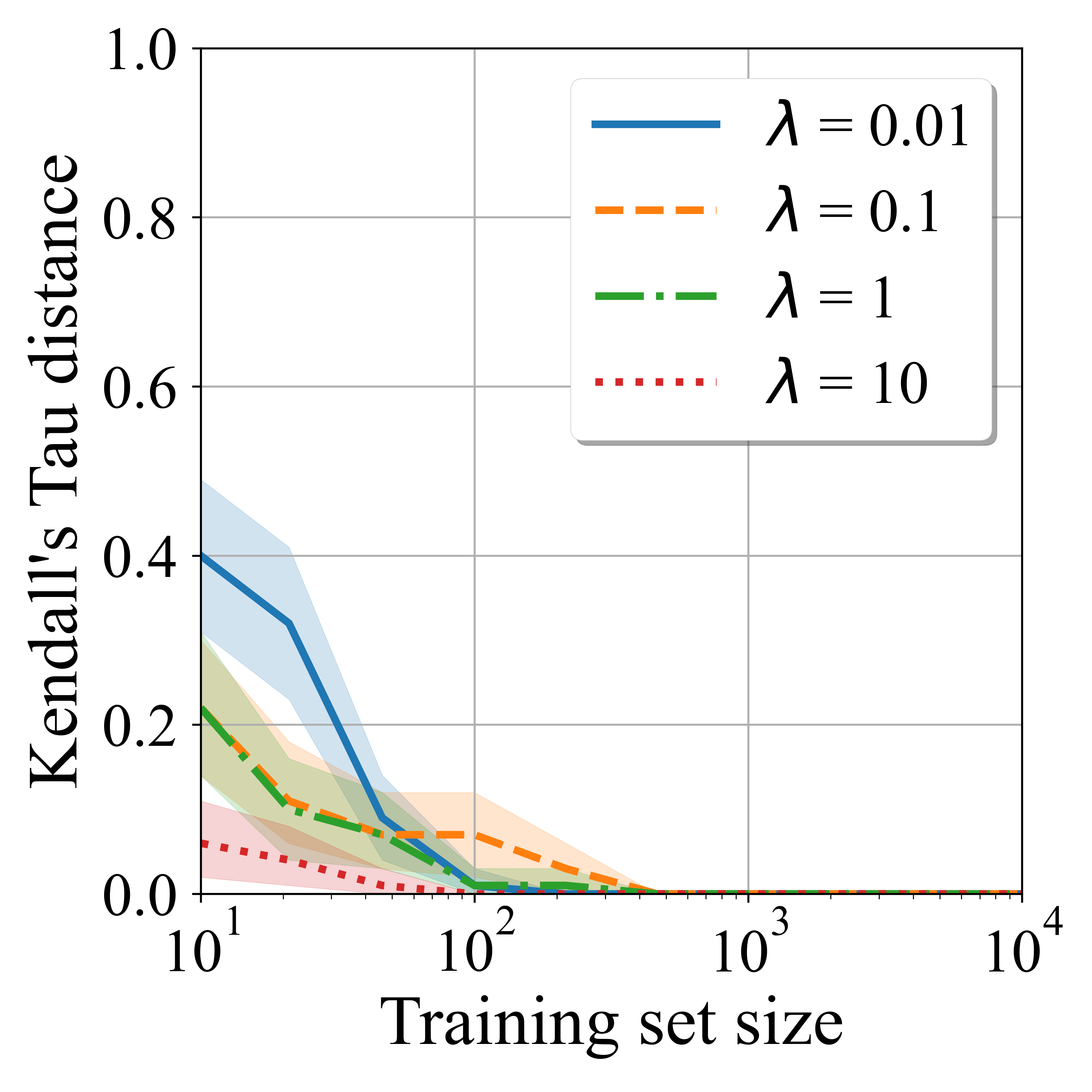}
\end{subfigure}
\caption{Additional experiments on synthetic data. We further support that ACP generates the same p-values as full CP as the training set increases.}
\end{figure}

\subsection{Additional motivating examples} \label{appendix:additional:examples}
We include additional motivating examples with specific instances from CIFAR-10 and MNIST for setting \mlpC.

\begin{figure*}[!htb]
\centering
\sbox{\bigpicturebox}{%
  \begin{subfigure}[b]{.6\textwidth}
  \adjustbox{raise=-5pc}{\scalebox{2}[2]{\includegraphics[width=0.5\textwidth]{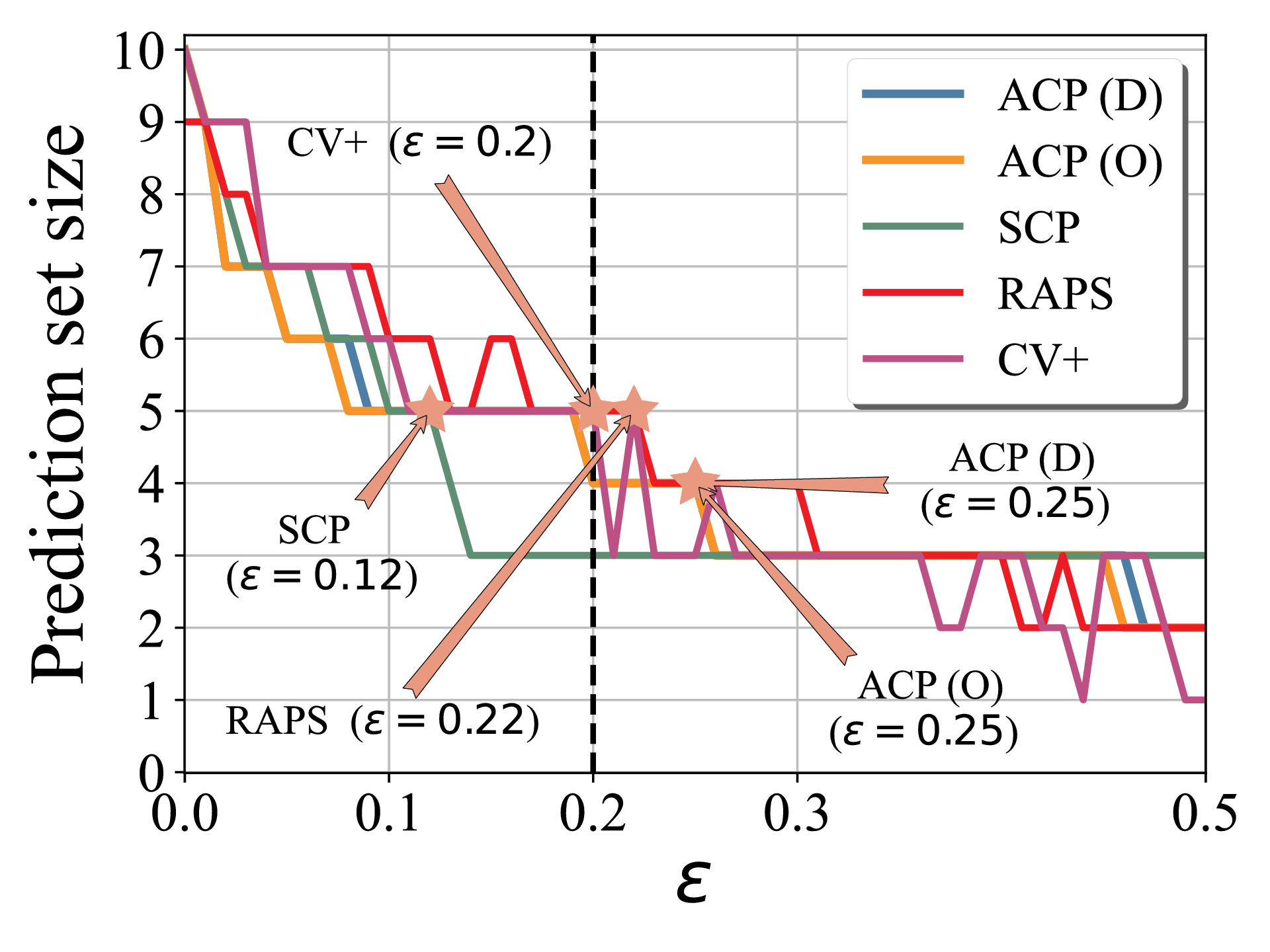}}}%
\end{subfigure}}
\usebox{\bigpicturebox}
\begin{minipage}[b][\ht\bigpicturebox][s]{.3\textwidth}
\begin{subfigure}{.3\textwidth}
\adjustbox{raise=-5pc}{\scalebox{4}[4]{\includegraphics[width=0.5\linewidth]{figures/CIFAR_91.JPG}}}
\end{subfigure}
\vfill
\begin{subfigure}[b]{.3\textwidth}
\begin{center}
\small 
\begin{tabular}{ll}
    \toprule
    \multicolumn{2}{c}{$\varepsilon = 0.2$}\\
    \toprule
     Method & Prediction set \\ \midrule
     ACP (D) & bird, \underline{\textbf{cat}}, deer, frog \\ 
     ACP (O) & bird, \underline{\textbf{cat}}, deer, frog \\  
     SCP & bird, deer, frog \\ 
     RAPS & bird, \underline{\textbf{cat}}, deer, dog, frog \\ 
     CV+ & bird, \underline{\textbf{cat}}, deer, dog, frog
 \end{tabular}
\end{center}
\end{subfigure}
\end{minipage}
\caption{Prediction set size as a function of $\varepsilon$ for a specific instance from CIFAR-10. We indicate with a star the point ($\varepsilon$) at which each method starts containing the true label. ACP (D) and ACP (O) are the \textit{fastest}, meaning that they include the true label earlier than the rest of methods (i.e., at higher $\varepsilon$). For a typical significance $\varepsilon=0.2$, the prediction sets generated by both methods are the smallest that contain the true label. RAPS and CV+ are more conservative, while SCP fails to include it. }
\label{examples_cat}
\end{figure*}

\begin{figure*}[!htb]
\centering
\sbox{\bigpicturebox}{%
  \begin{subfigure}[b]{.6\textwidth}
  \adjustbox{raise=-5pc}{\scalebox{2}[2]{\includegraphics[width=0.5\textwidth]{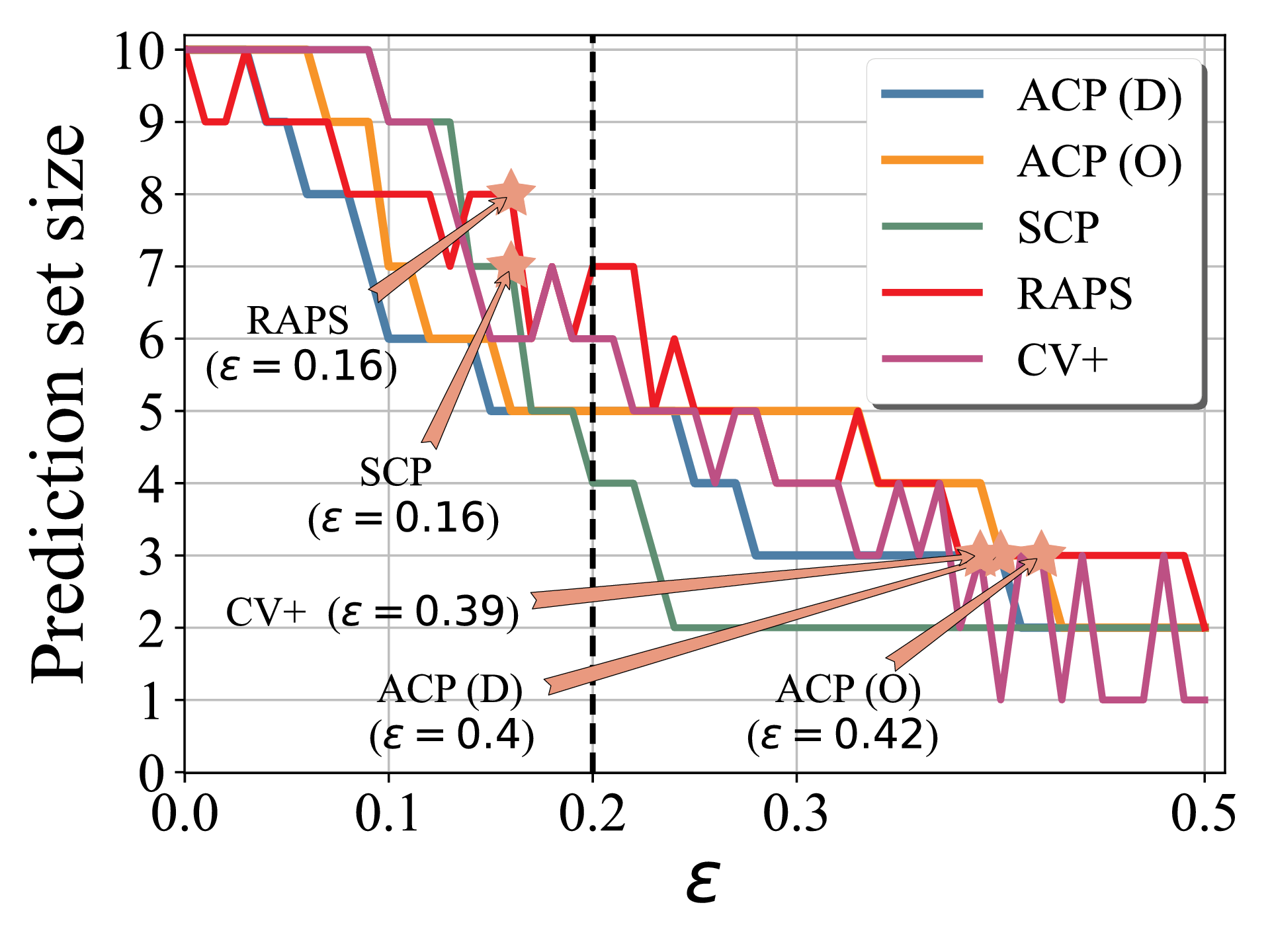}}}%
\end{subfigure}}
\usebox{\bigpicturebox}
\begin{minipage}[b][\ht\bigpicturebox][s]{.3\textwidth}
\begin{subfigure}{.3\textwidth}
\adjustbox{raise=-5pc}{\scalebox{4}[4]{\includegraphics[width=0.5\linewidth]{figures/CIFAR_13.JPG}}}
\end{subfigure}
\vfill
\begin{subfigure}[b]{.3\textwidth}
\begin{center}
\small 
\begin{tabular}{ll}
    \toprule
    \multicolumn{2}{c}{$\varepsilon = 0.2$}\\
    \toprule
     Method & Prediction set \\ \midrule
     ACP (D) & auto, cat, frog, \underline{\textbf{horse}}, truck \\ 
     ACP (O) & auto, cat, frog, \underline{\textbf{horse}}, truck \\ 
     SCP & auto, deer, frog, truck \\ 
     RAPS & plane, auto, bird, deer, frog, ship, truck \\ 
     CV+ & plane, auto, deer, frog, \underline{\textbf{horse}}, truck
 \end{tabular}
\end{center}
\end{subfigure}
\end{minipage}
\caption{Additional motivating example in CIFAR-10. For a typical significance $\varepsilon=0.2$, ACP (D) and ACP (O) yield the smallest prediction sets that contain the true label. ACP (D) and ACP (O) are the first to consider the true label.}
\label{examples_horse}
\end{figure*}

\begin{figure*}[!htb]
\centering
\sbox{\bigpicturebox}{%
  \begin{subfigure}[b]{.6\textwidth}
  \adjustbox{raise=-5pc}{\scalebox{2}[2]{\includegraphics[width=0.5\textwidth]{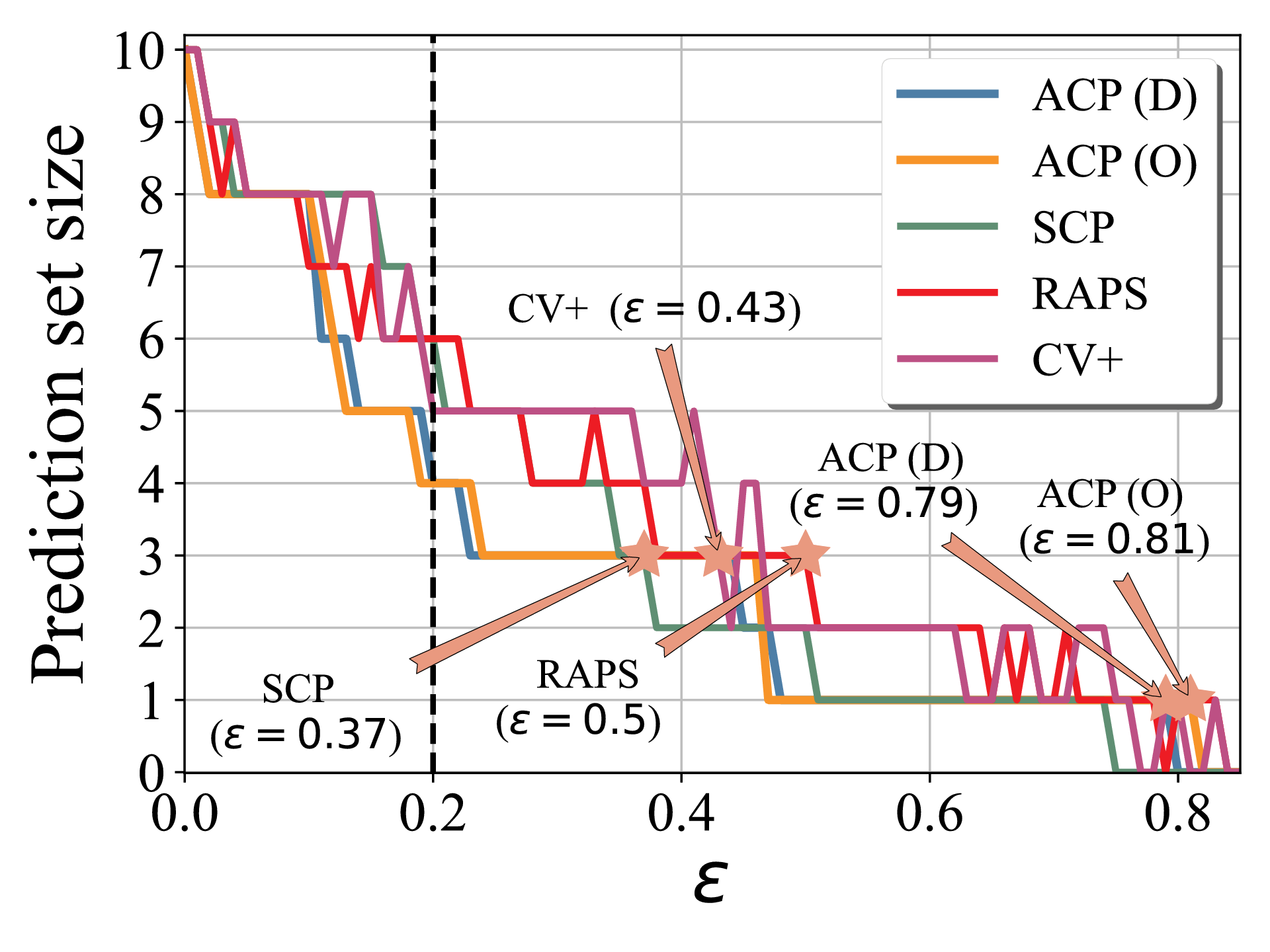}}}%
\end{subfigure}}
\usebox{\bigpicturebox}
\begin{minipage}[b][\ht\bigpicturebox][s]{.3\textwidth}
\begin{subfigure}{.3\textwidth}
\adjustbox{raise=-5pc}{\scalebox{4}[4]{\includegraphics[width=0.5\linewidth]{figures/CIFAR_43.JPG}}}
\end{subfigure}
\vfill
\begin{subfigure}[b]{.3\textwidth}
\begin{center}
\small 
\begin{tabular}{ll}
    \toprule
    \multicolumn{2}{c}{$\varepsilon = 0.2$}\\
    \toprule
     Method & Prediction set \\ \midrule
     ACP (D) & cat, deer, \underline{\textbf{frog}}, horse \\ 
     ACP (O) & cat, deer, \underline{\textbf{frog}}, horse \\  
     SCP & cat, deer, dog, \underline{\textbf{frog}}, horse, truck \\  
     RAPS & cat, deer, dog, \underline{\textbf{frog}}, horse, truck \\ 
     CV+ & cat, deer, dog, \underline{\textbf{frog}}, horse
 \end{tabular}
\end{center}
\end{subfigure}
\end{minipage}
\caption{Additional motivating example in CIFAR-10. For a typical significance $\varepsilon=0.2$, all methods include the true label, ACP (D) and ACP (O) yielding the smallest prediction sets. ACP (D) and ACP (O) are the first to consider the true label.}
\label{examples_frog}
\end{figure*}

\begin{figure*}[!htb]
\centering
\sbox{\bigpicturebox}{%
  \begin{subfigure}[b]{.6\textwidth}
  \adjustbox{raise=-5pc}{\scalebox{2}[2]{\includegraphics[width=0.5\textwidth]{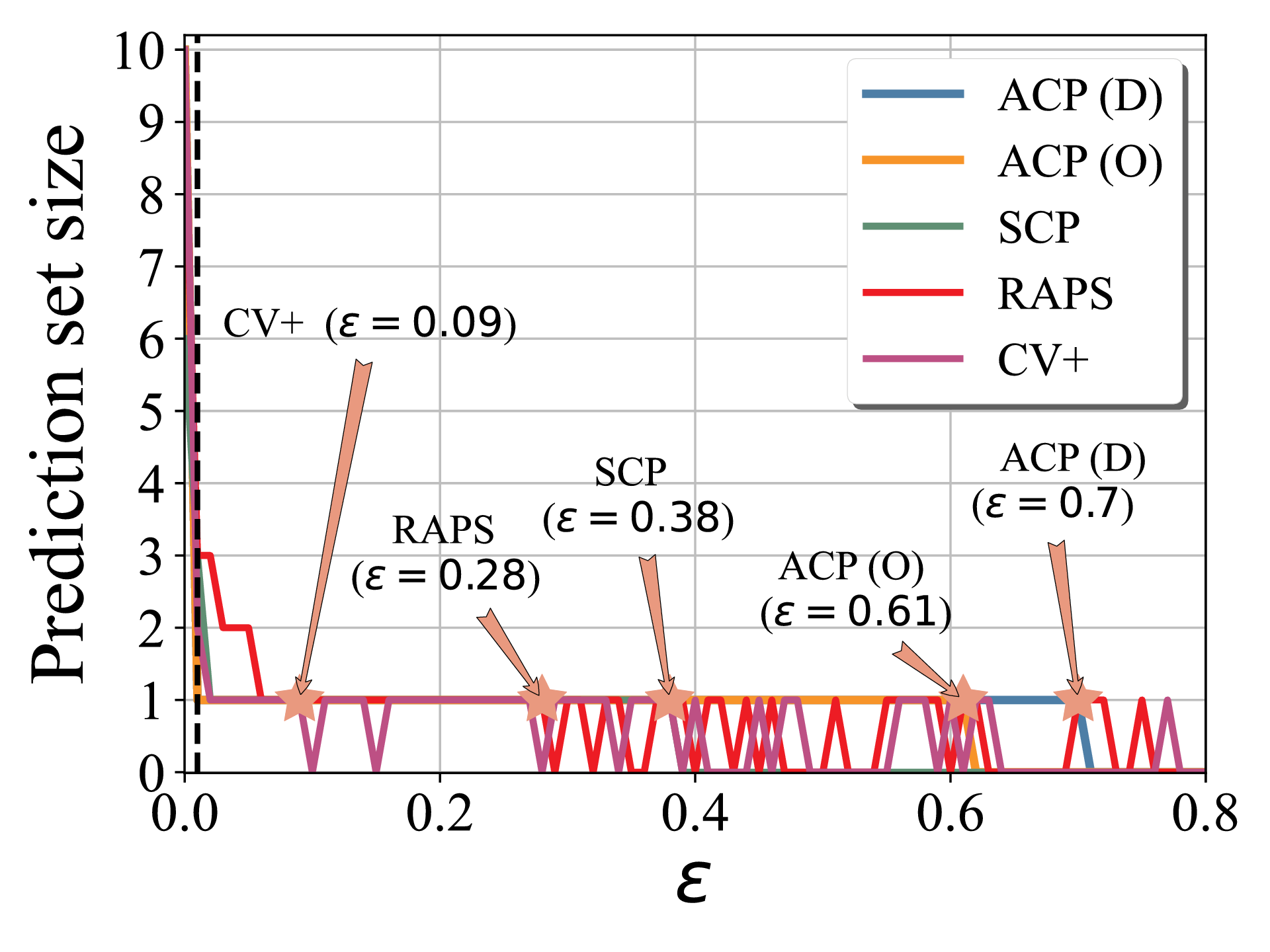}}}%
\end{subfigure}}
\usebox{\bigpicturebox}
\begin{minipage}[b][\ht\bigpicturebox][s]{.3\textwidth}
\begin{subfigure}{.3\textwidth}
\adjustbox{raise=-5pc}{\scalebox{4}[4]{\includegraphics[width=0.5\linewidth]{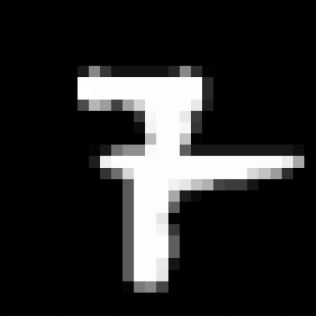}}}
\end{subfigure}
\vfill
\begin{subfigure}[b]{.3\textwidth}
\begin{center}
\small
\begin{tabular}{ll}
    \toprule
    \multicolumn{2}{c}{$\varepsilon = 0.01$}\\
    \toprule
     Method & Prediction set \\ \midrule
     ACP (D) & \underline{\textbf{7}} \\ 
     ACP (O) & \underline{\textbf{7}} \\ 
     SCP & 4, \underline{\textbf{7}}, 8 \\ 
     RAPS & 4, \underline{\textbf{7}}, 8 \\ 
     CV+ & 1, \underline{\textbf{7}}
 \end{tabular}
\end{center}
\end{subfigure}
\end{minipage}
\caption{Additional motivating example in MNIST. For a typical significance $\varepsilon=0.01$, all methods include the true label, ACP (D) and ACP (O) yielding the smallest prediction sets. ACP (D) and ACP (O) are also the first to consider the true label.}
\label{}
\end{figure*}

\begin{figure*}[!htb]
\centering
\sbox{\bigpicturebox}{%
  \begin{subfigure}[b]{.6\textwidth}
  \adjustbox{raise=-5pc}{\scalebox{2}[2]{\includegraphics[width=0.5\textwidth]{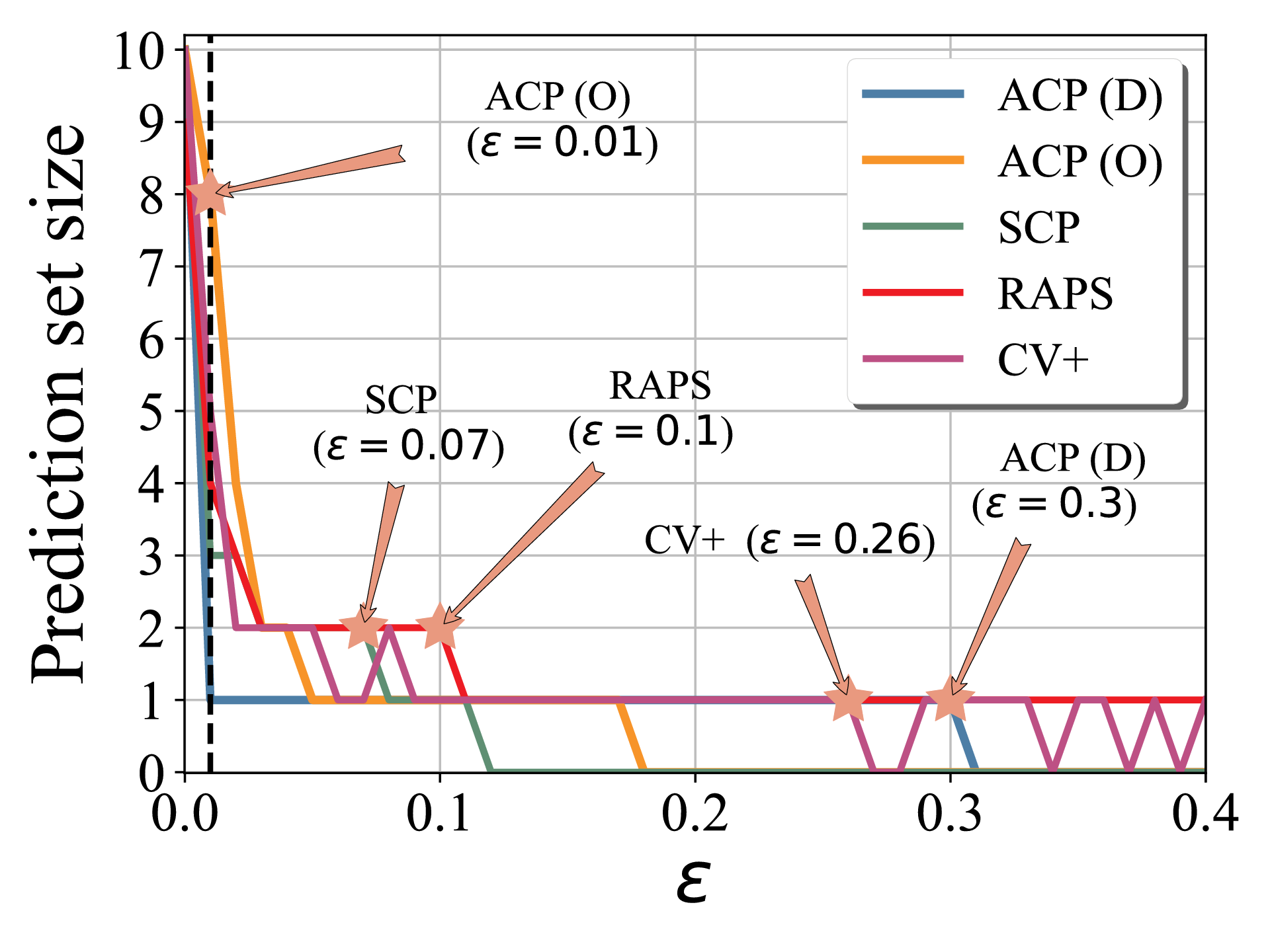}}}%
\end{subfigure}}
\usebox{\bigpicturebox}
\begin{minipage}[b][\ht\bigpicturebox][s]{.3\textwidth}
\begin{subfigure}{.3\textwidth}
\adjustbox{raise=-5pc}{\scalebox{4}[4]{\includegraphics[width=0.5\linewidth]{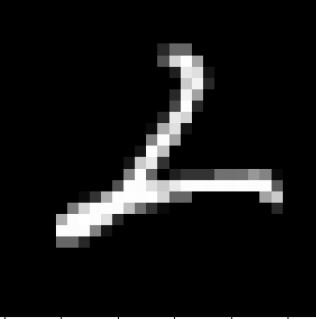}}}
\end{subfigure}
\vfill
\begin{subfigure}[b]{.3\textwidth}
\begin{center}
\small
\begin{tabular}{ll}
    \toprule
    \multicolumn{2}{c}{$\varepsilon = 0.01$}\\
    \toprule
     Method & Prediction set \\ \midrule
     ACP (D) & \underline{\textbf{2}} \\ 
     ACP (O) & 1, \underline{\textbf{2}}, 3, 5, 6, 7, 8, 9 \\ 
     SCP & 1, \underline{\textbf{2}}, 7 \\ 
     RAPS & 1, \underline{\textbf{2}}, 6, 7 \\ 
     CV+ & \underline{\textbf{2}}, 3, 6
 \end{tabular}
\end{center}
\end{subfigure}
\end{minipage}
\caption{Additional motivating example in MNIST. For a typical significance $\varepsilon=0.01$, all methods include the true label, ACP (D) yielding the smallest prediction sets. ACP (D) is the first method to consider the true label, followed by CV+. ACP (O) builds the largest set and is the \textit{slowest} to consider the true label.}
\label{}
\end{figure*}

\subsection{Additional curves}\label{appendix:additional_curves}
\Cref{fig:all_curves} plots the average prediction set size w.r.t the significance $\varepsilon$ for all combinations of settings (\mlpA, \mlpB, \mlpC, LR, CNN) and datasets (MNIST, CIFAR-10, US Census).

\begin{figure*}[htb!]
        \centering
        \begin{subfigure}[hb]{0.16\textwidth}
             \includegraphics[width=\linewidth]{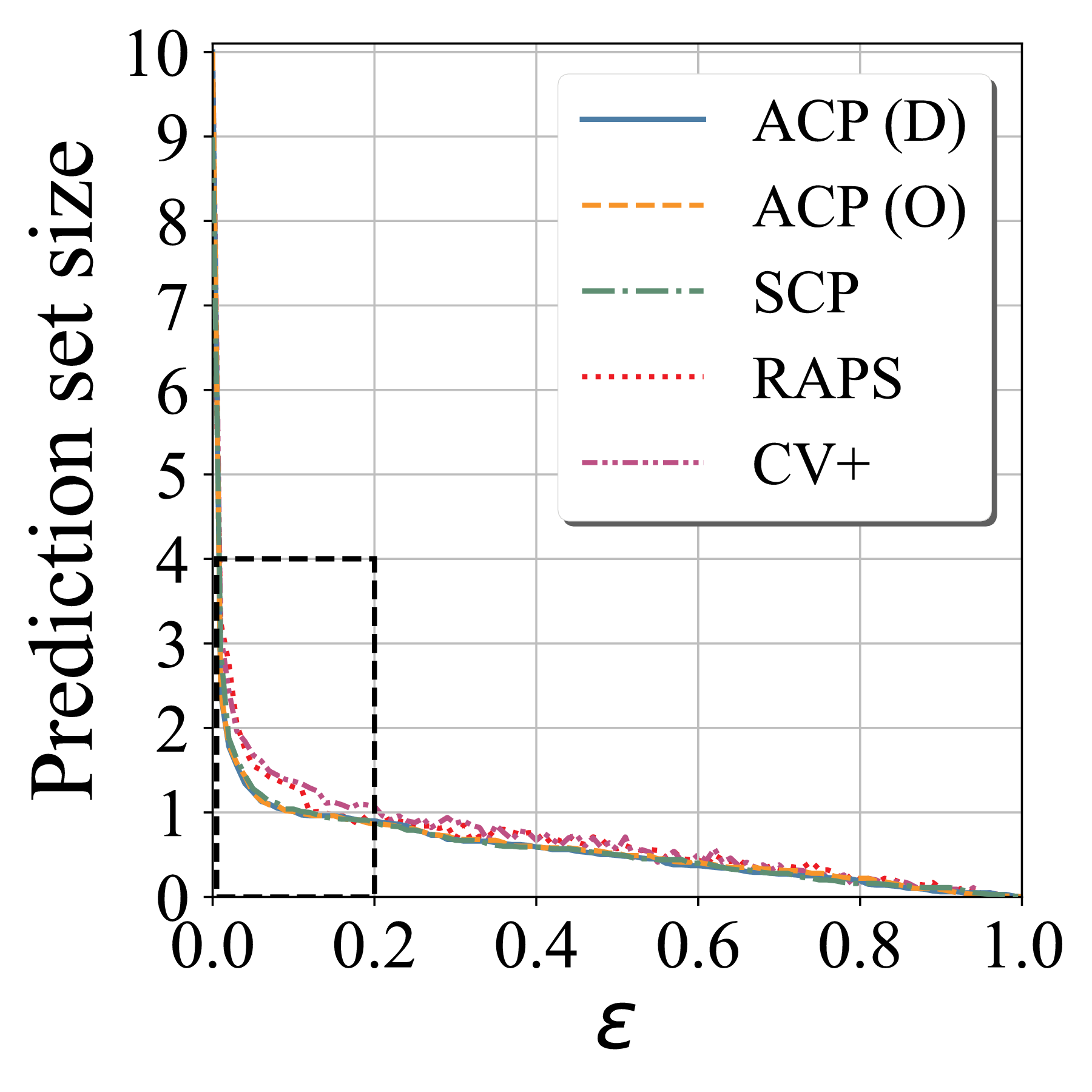}
             \caption*{MNIST (\mlpA)}
        \end{subfigure} 
        \begin{subfigure}[hb]{0.16\textwidth}
                \includegraphics[width=\linewidth]{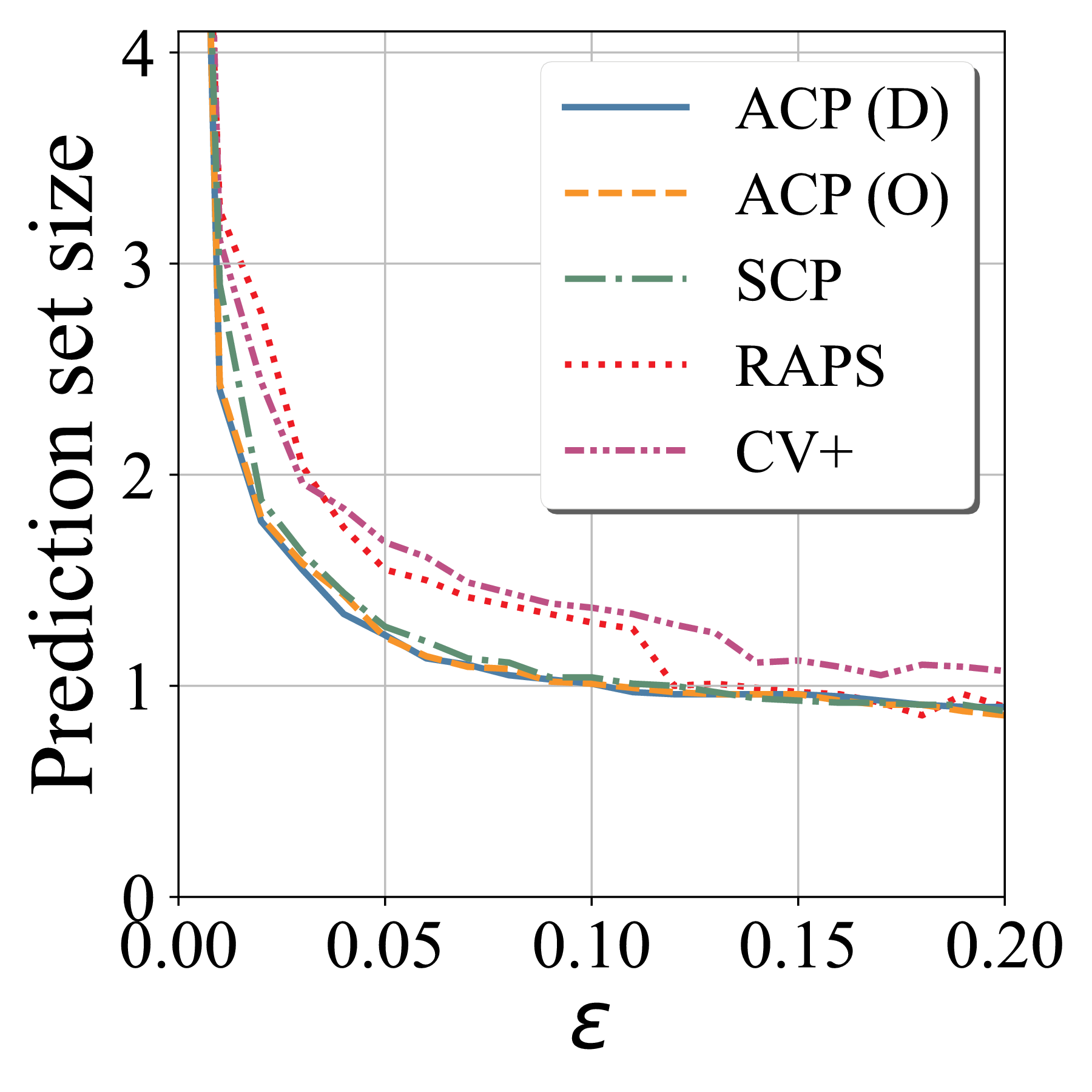}
                \caption*{MNIST (\mlpA)}
\end{subfigure} 
\begin{subfigure}[hb]{0.16\textwidth}
     \includegraphics[width=\linewidth]{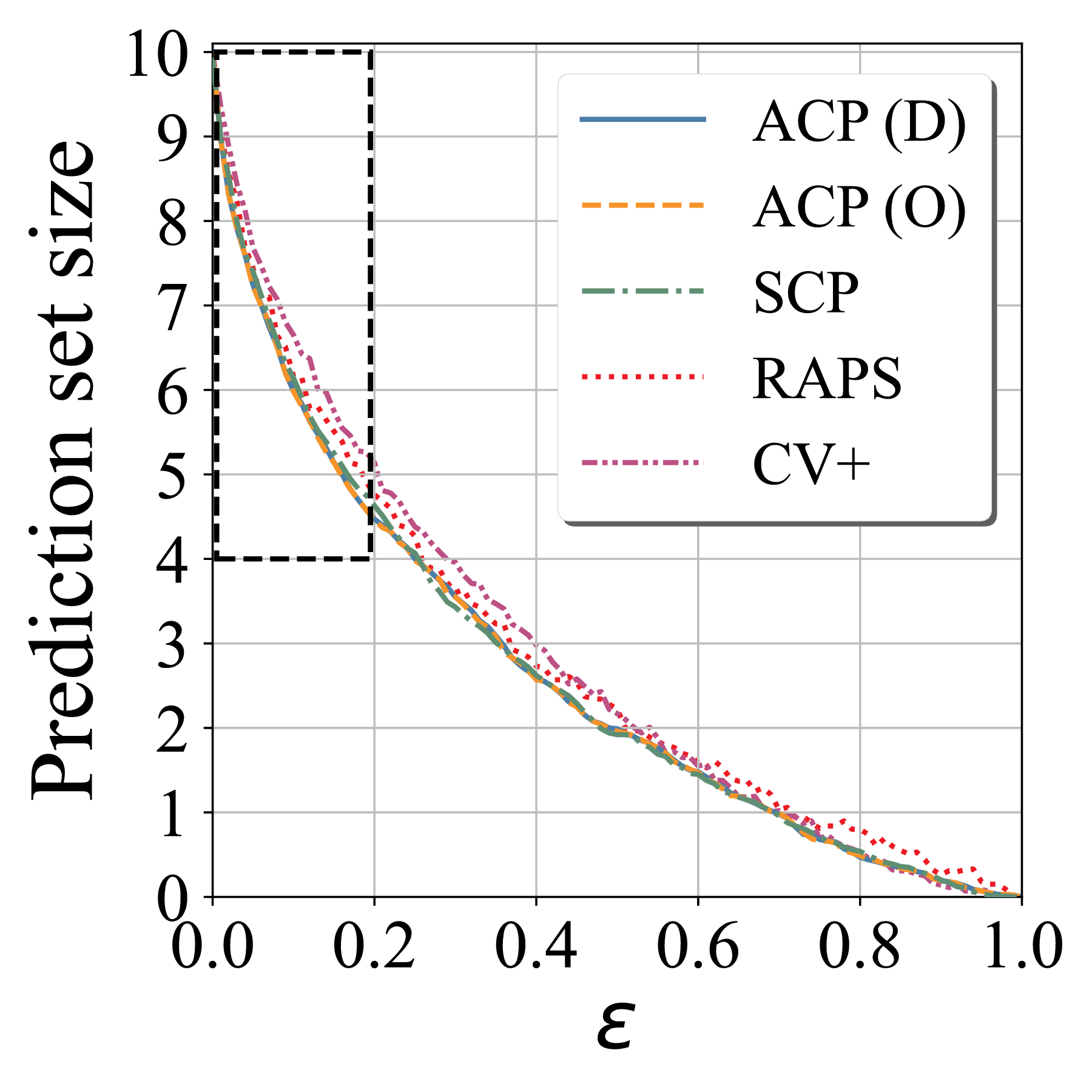}
     \caption*{CIFAR (\mlpA)}
\end{subfigure} 
\begin{subfigure}[hb]{0.16\textwidth}
        \includegraphics[width=\linewidth]{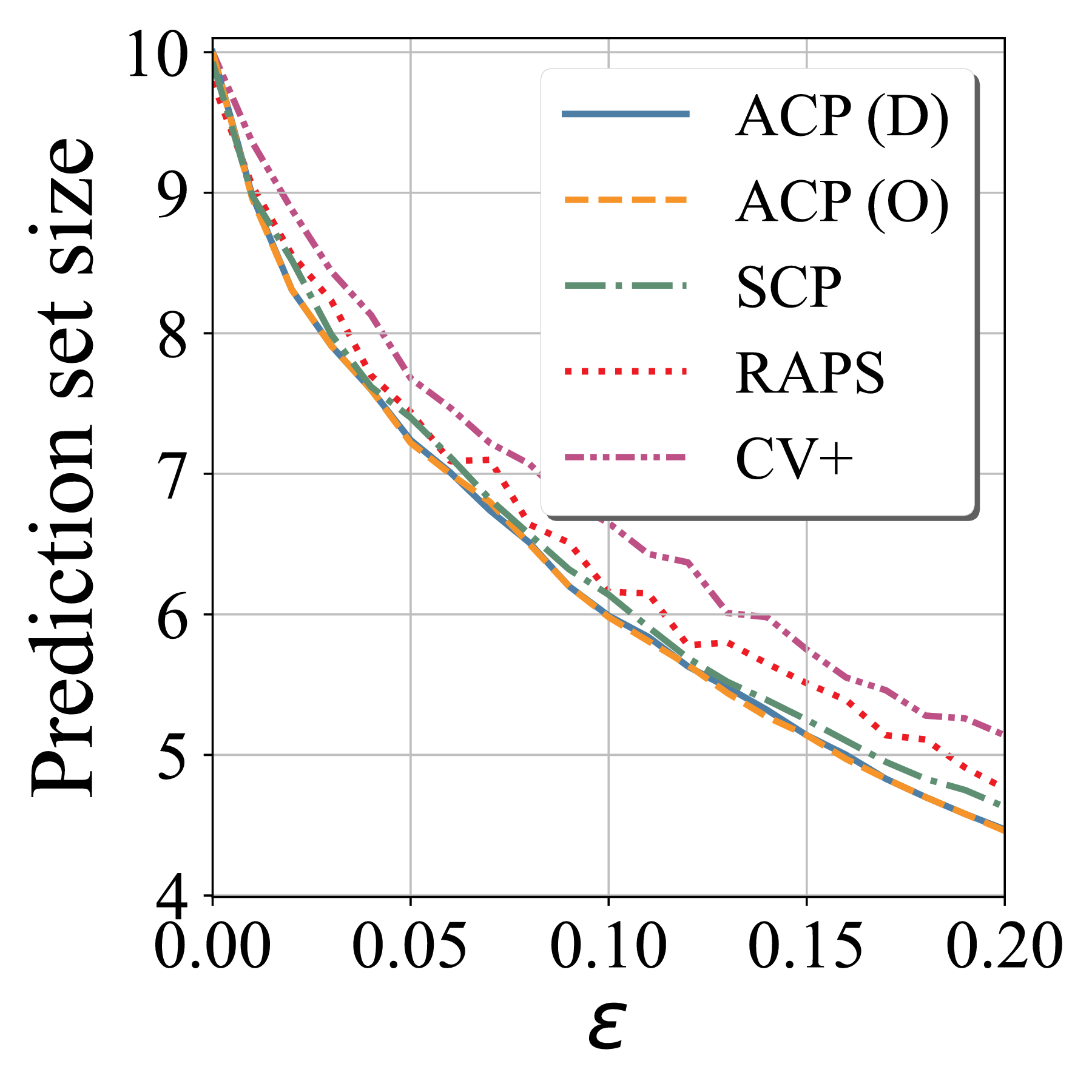}
        \caption*{CIFAR (\mlpA)}
\end{subfigure}
\begin{subfigure}[hb]{0.16\textwidth}
     \includegraphics[width=\linewidth]{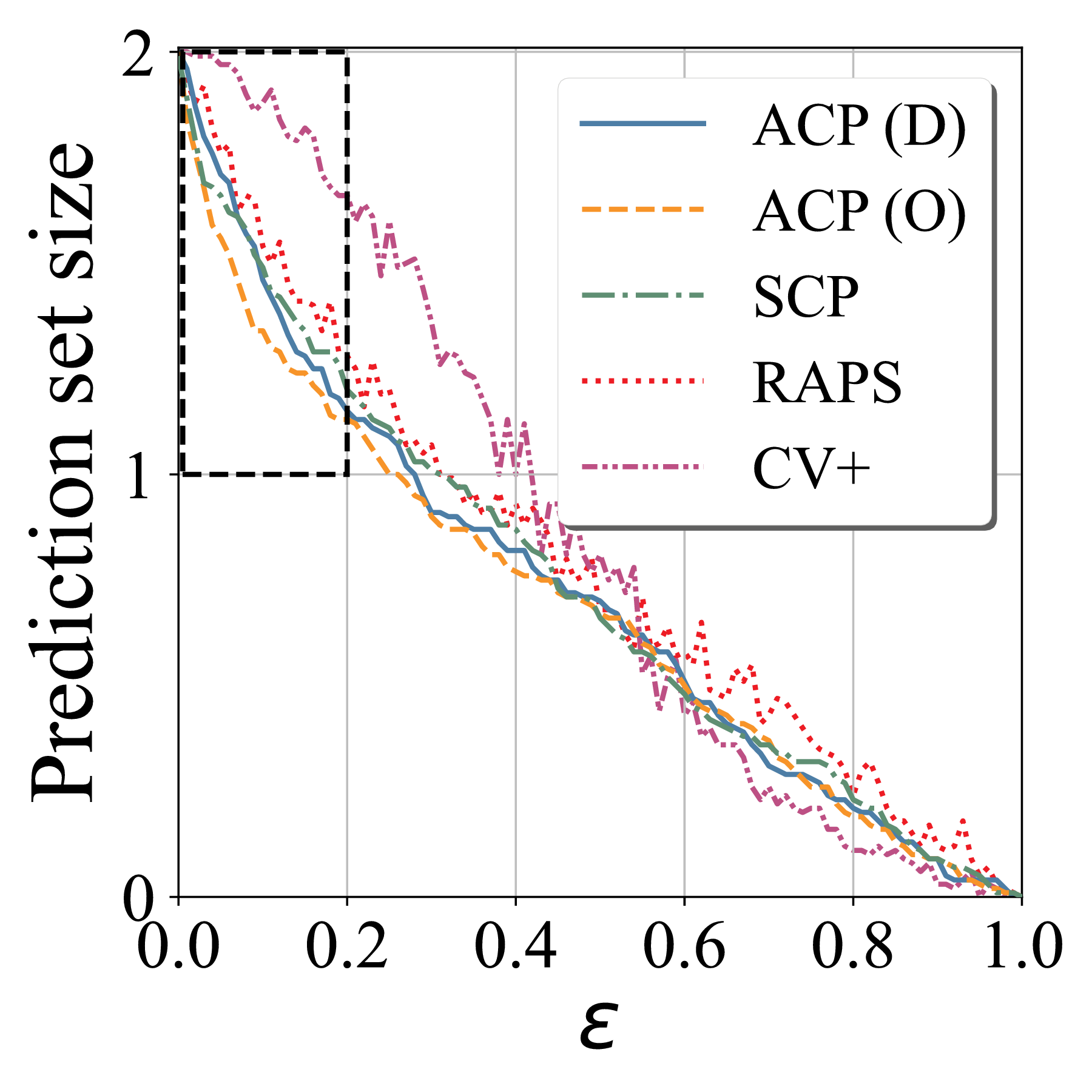}
     \caption*{US Census (\mlpA)}
\end{subfigure} 
\begin{subfigure}[hb]{0.16\textwidth}
        \includegraphics[width=\linewidth]{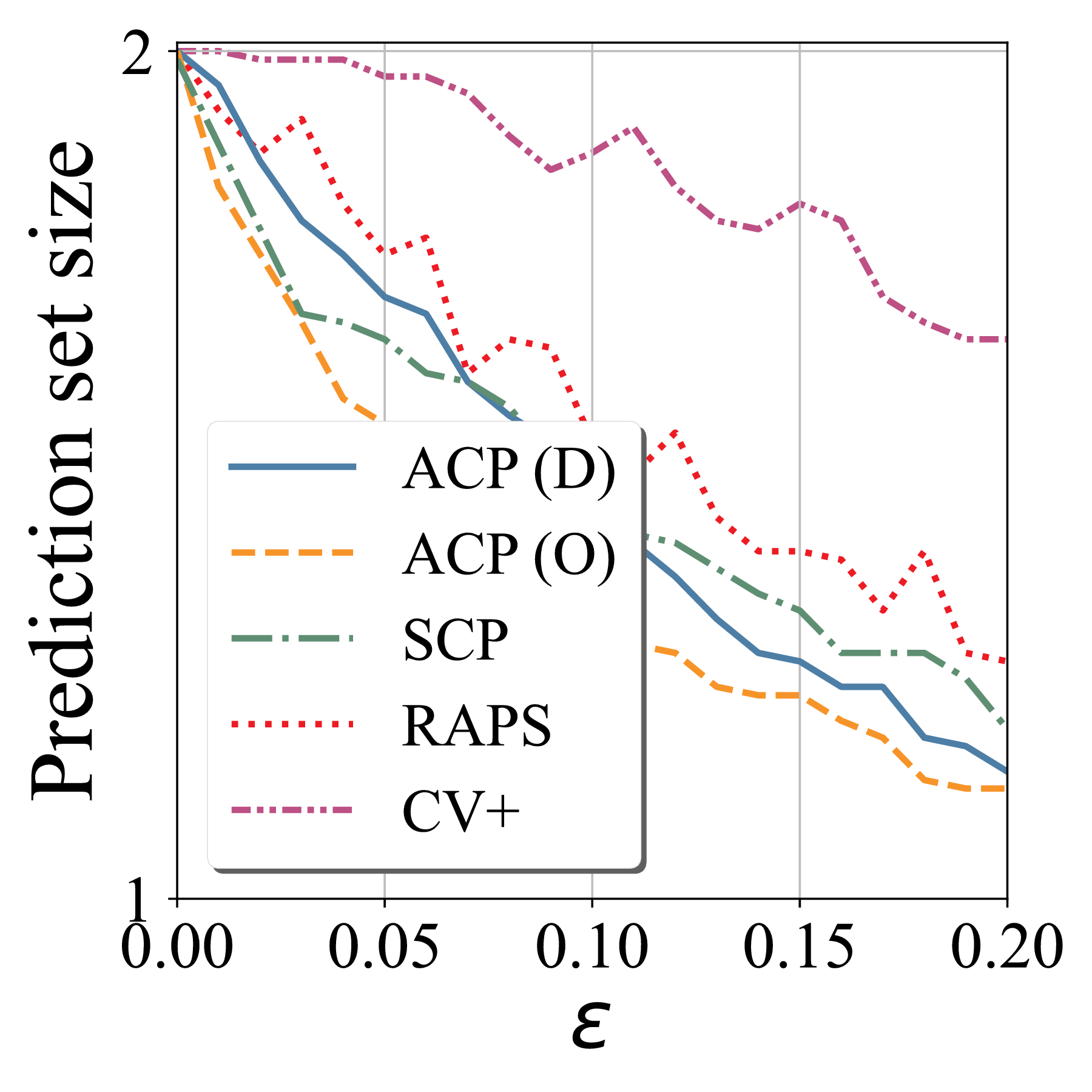}
        \caption*{US Census (\mlpA)}
\end{subfigure} \\
\begin{subfigure}[hb]{0.16\textwidth}
     \includegraphics[width=\linewidth]{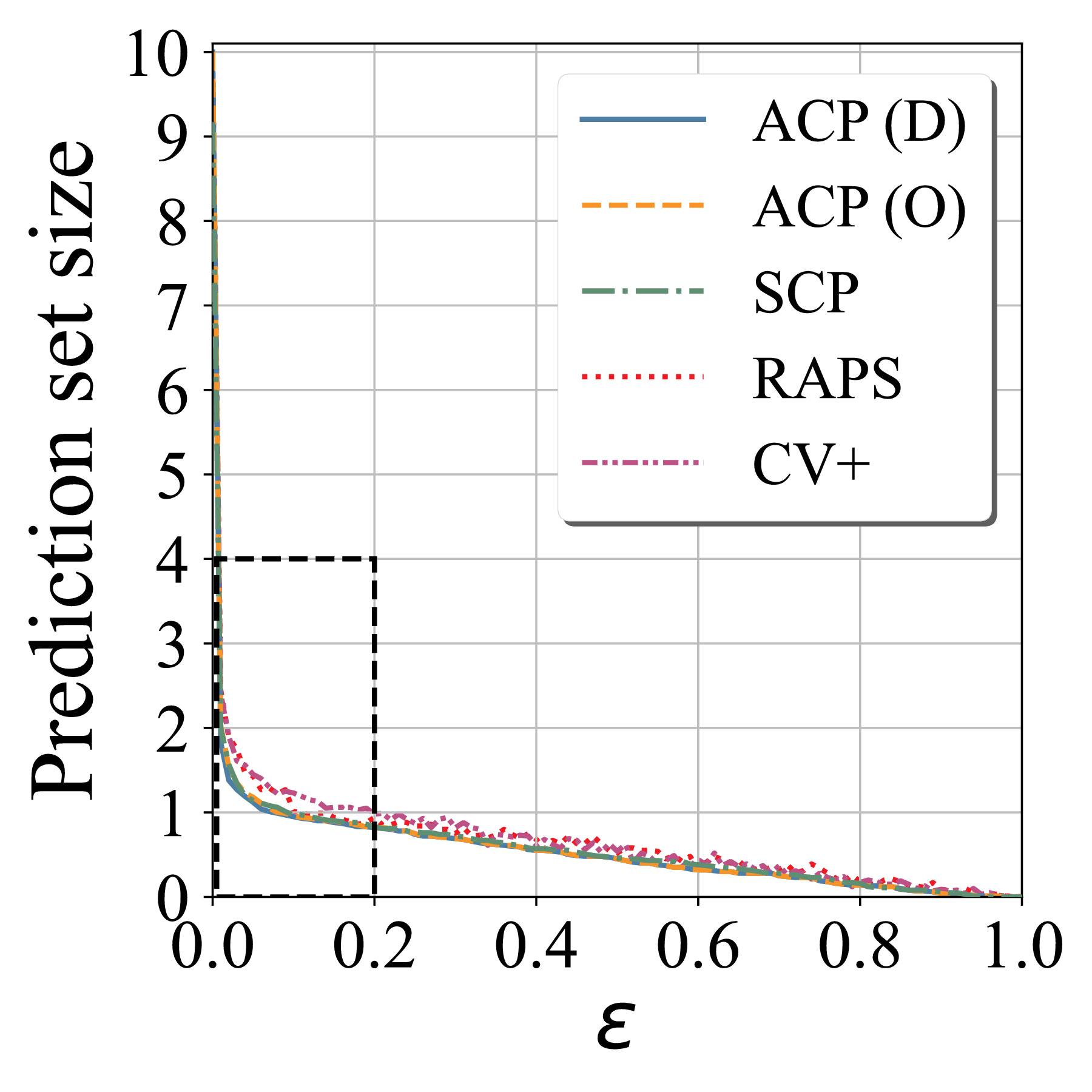}
     \caption*{MNIST (\mlpB)}
\end{subfigure} 
\begin{subfigure}[hb]{0.16\textwidth}
        \includegraphics[width=\linewidth]{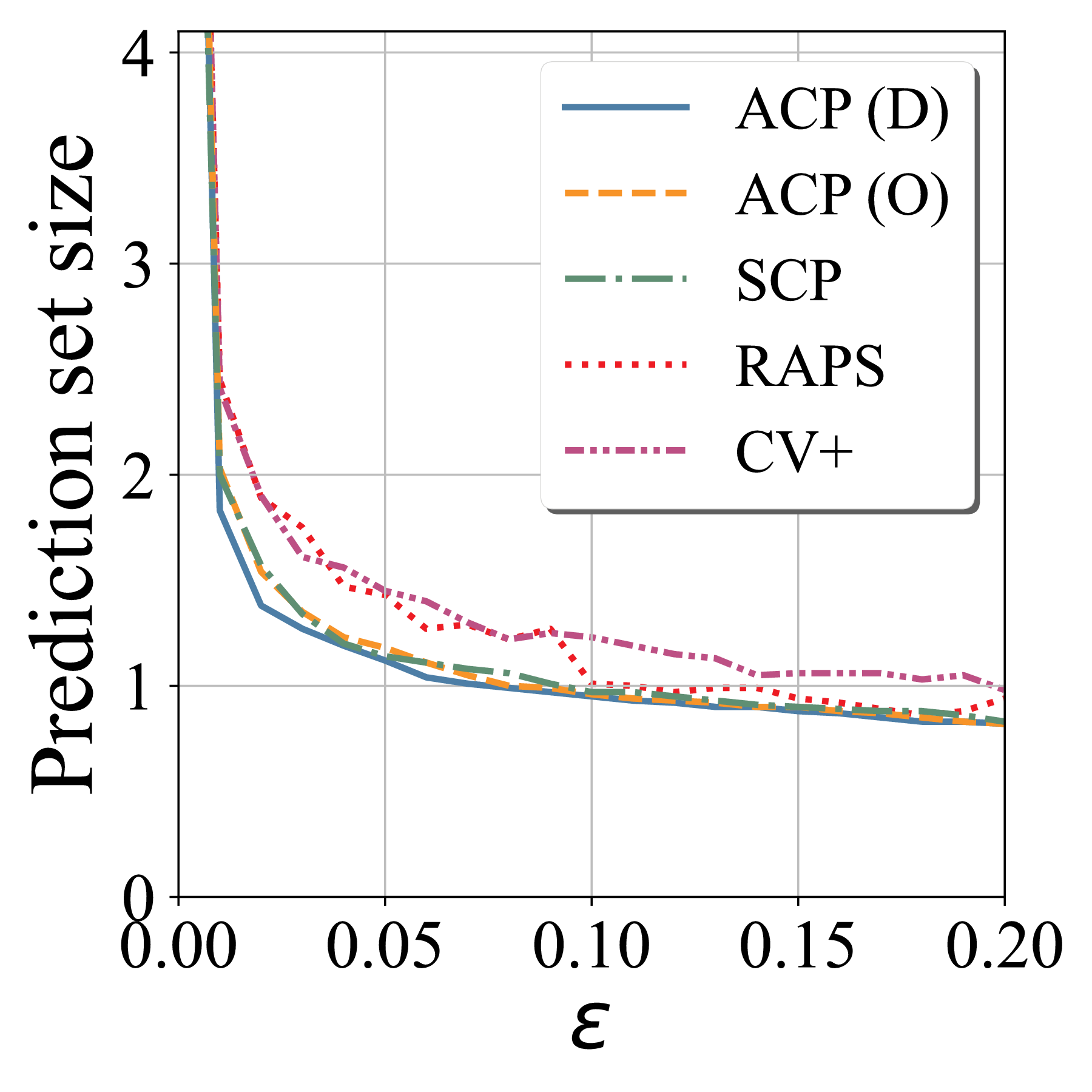}
        \caption*{MNIST (\mlpB)}
\end{subfigure} 
\begin{subfigure}[hb]{0.16\textwidth}
     \includegraphics[width=\linewidth]{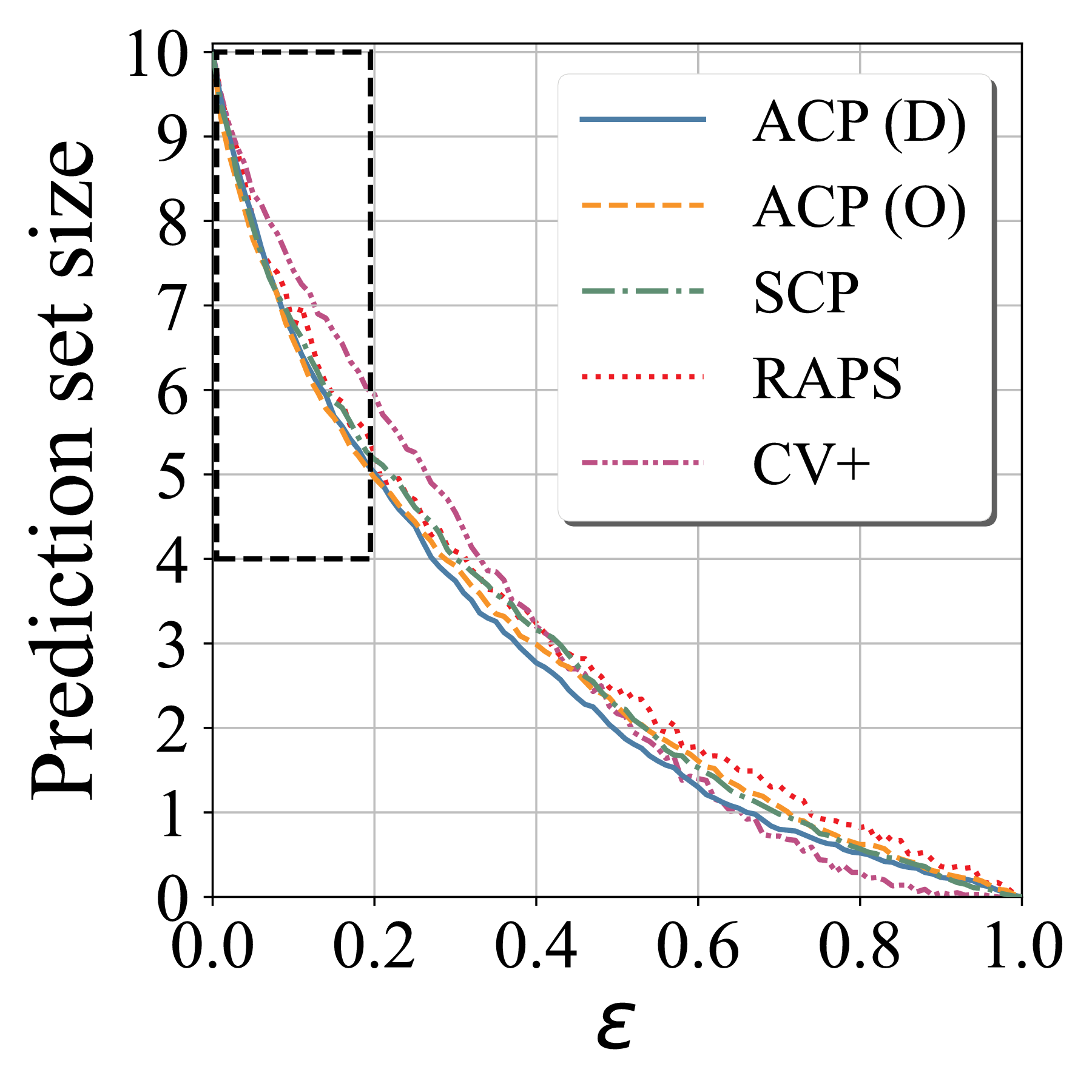}
     \caption*{CIFAR (\mlpB)}
\end{subfigure} 
\begin{subfigure}[hb]{0.16\textwidth}
        \includegraphics[width=\linewidth]{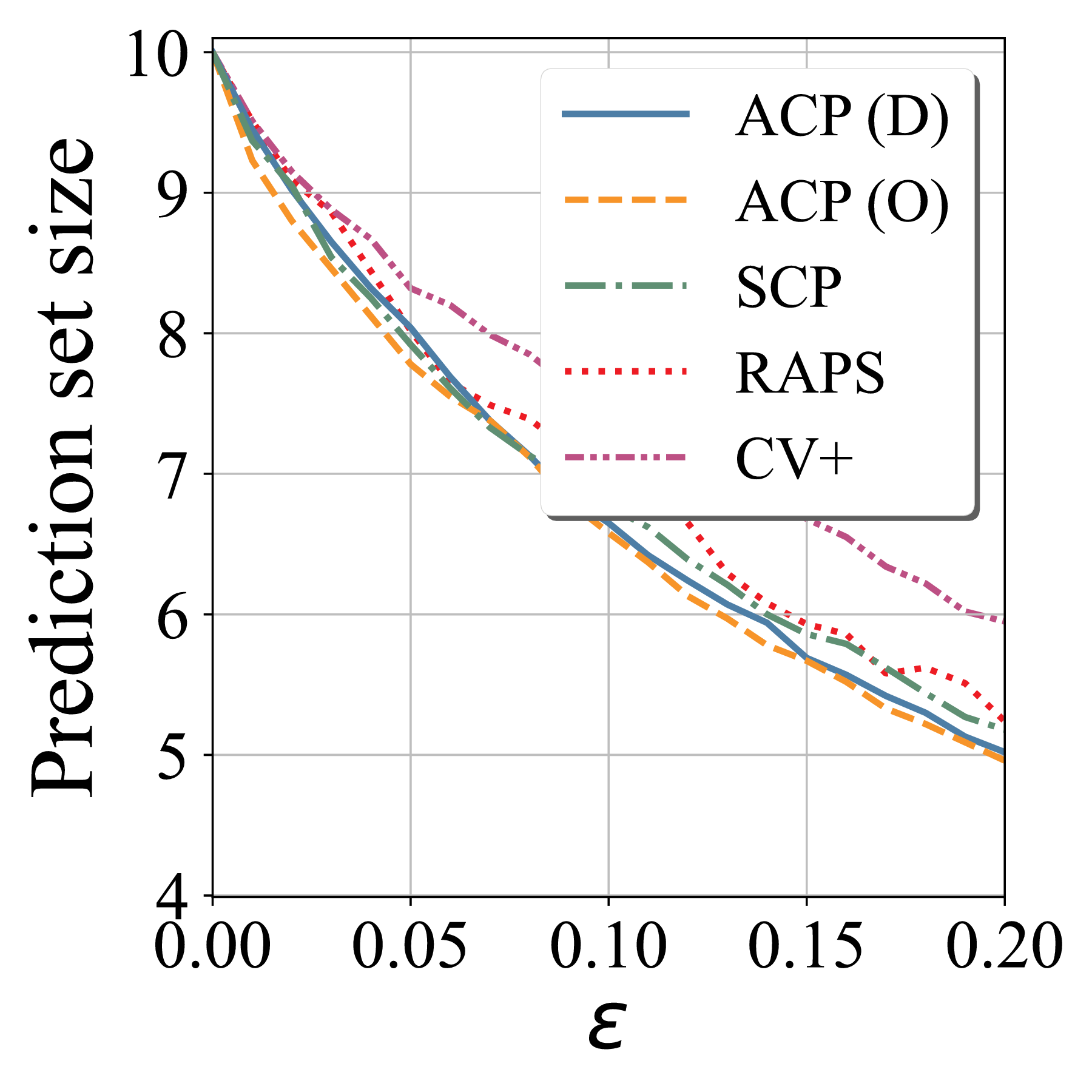}
        \caption*{CIFAR (\mlpB)}
\end{subfigure}
\begin{subfigure}[hb]{0.16\textwidth}
     \includegraphics[width=\linewidth]{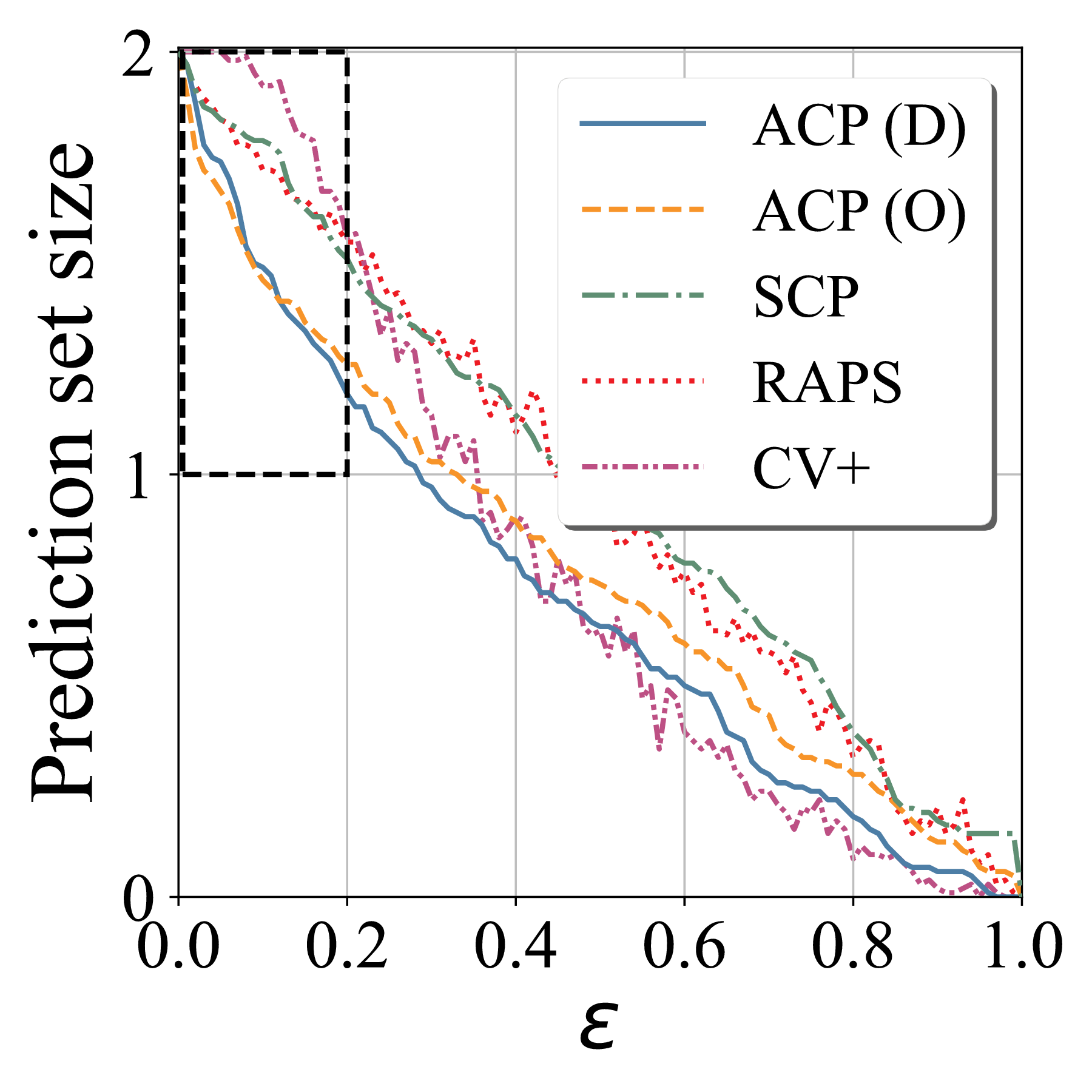}
     \caption*{US Census (\mlpB)}
\end{subfigure} 
\begin{subfigure}[hb]{0.16\textwidth}
        \includegraphics[width=\linewidth]{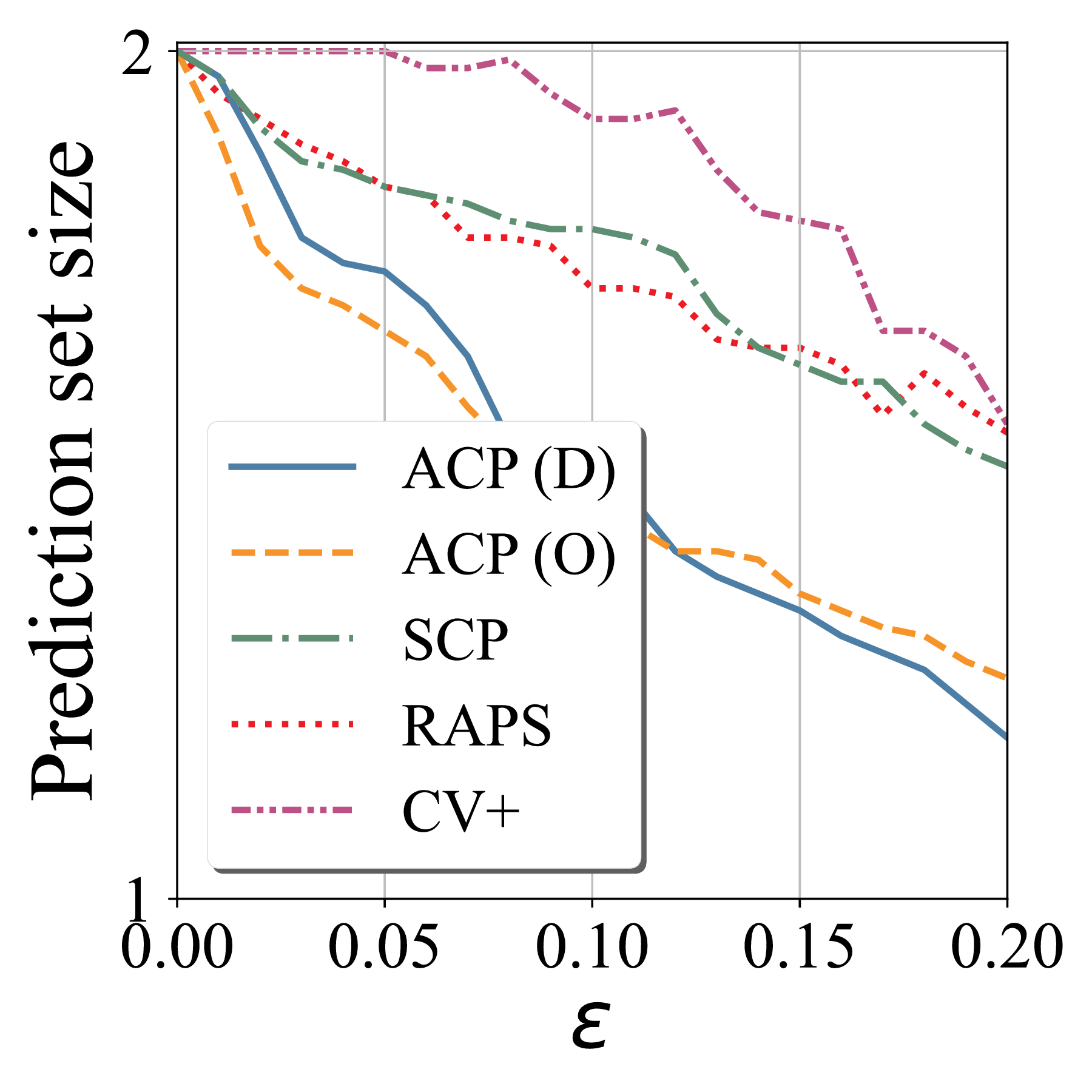}
        \caption*{US Census (\mlpB)}
\end{subfigure} 
\begin{subfigure}[hb]{0.16\textwidth}
     \includegraphics[width=\linewidth]{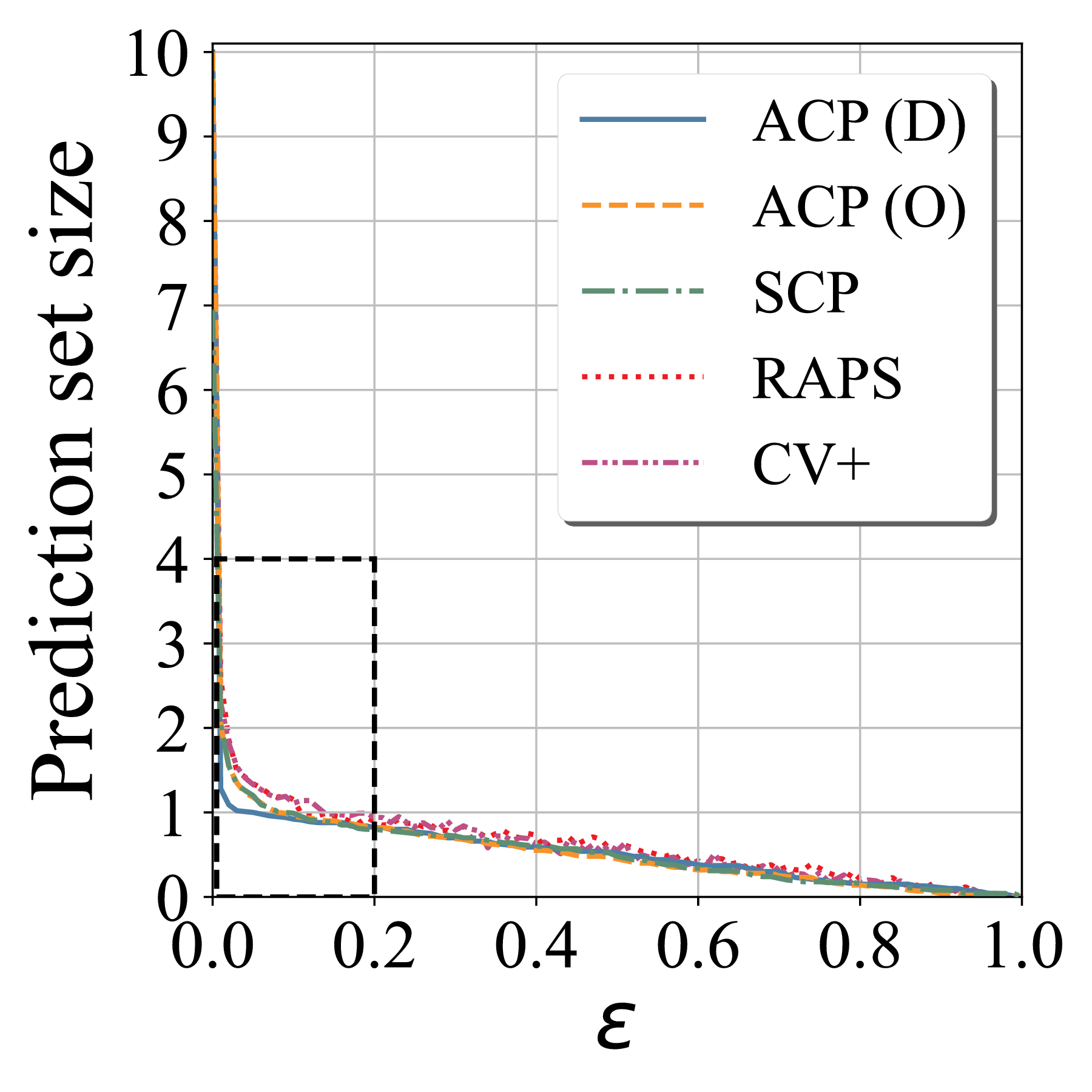}
     \caption*{MNIST (\mlpC)}
\end{subfigure} 
\begin{subfigure}[hb]{0.16\textwidth}
        \includegraphics[width=\linewidth]{figures/MNIST_global_sizes_short_C.png}
        \caption*{MNIST (\mlpC)}
\end{subfigure} 
\begin{subfigure}[hb]{0.16\textwidth}
     \includegraphics[width=\linewidth]{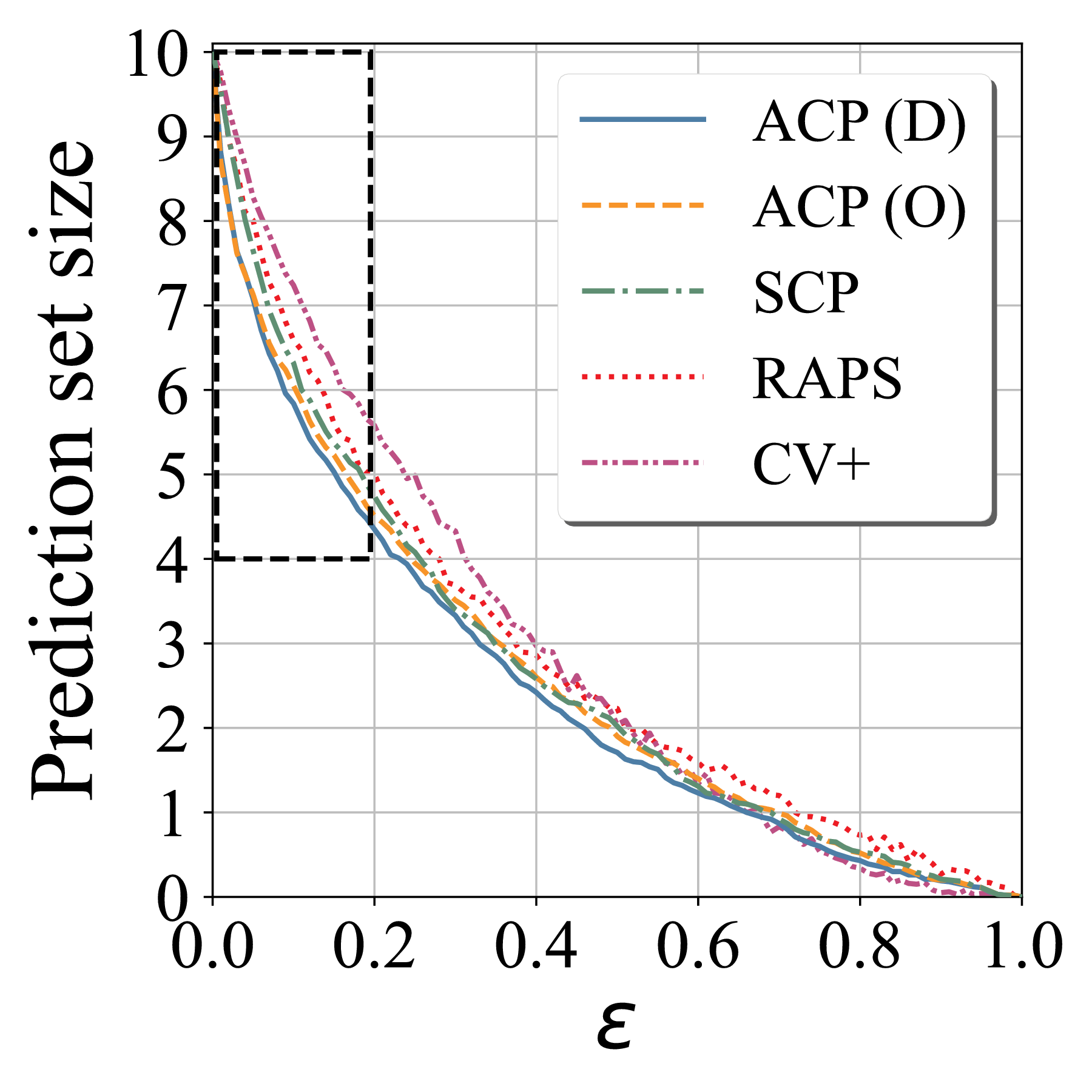}
     \caption*{CIFAR (\mlpC)}
\end{subfigure} 
\begin{subfigure}[hb]{0.16\textwidth}
        \includegraphics[width=\linewidth]{figures/CIFAR_global_sizes_short_C.png}
        \caption*{CIFAR (\mlpC)}
\end{subfigure}
\begin{subfigure}[hb]{0.16\textwidth}
     \includegraphics[width=\linewidth]{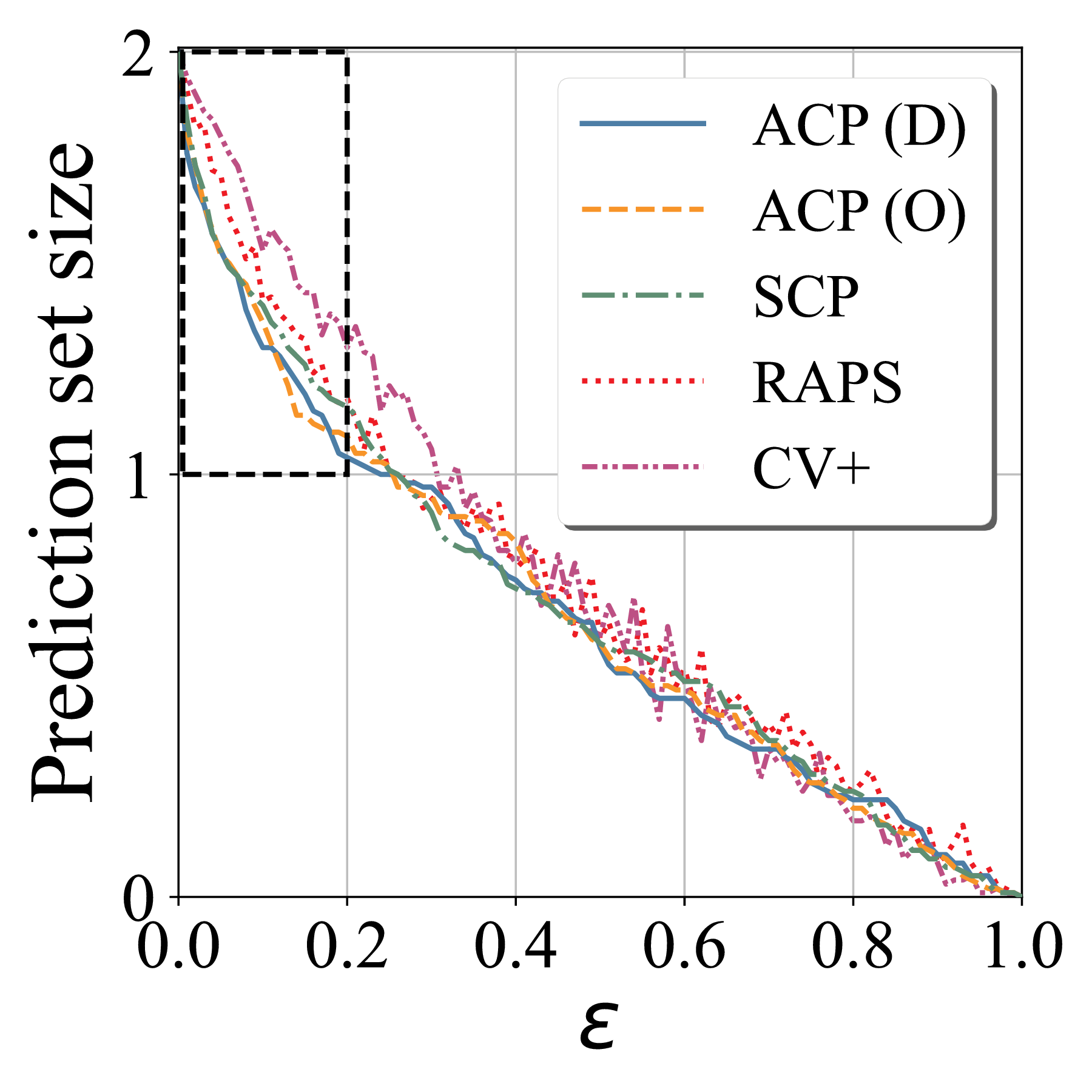}
     \caption*{US Census (\mlpC)}
\end{subfigure} 
\begin{subfigure}[hb]{0.16\textwidth}
        \includegraphics[width=\linewidth]{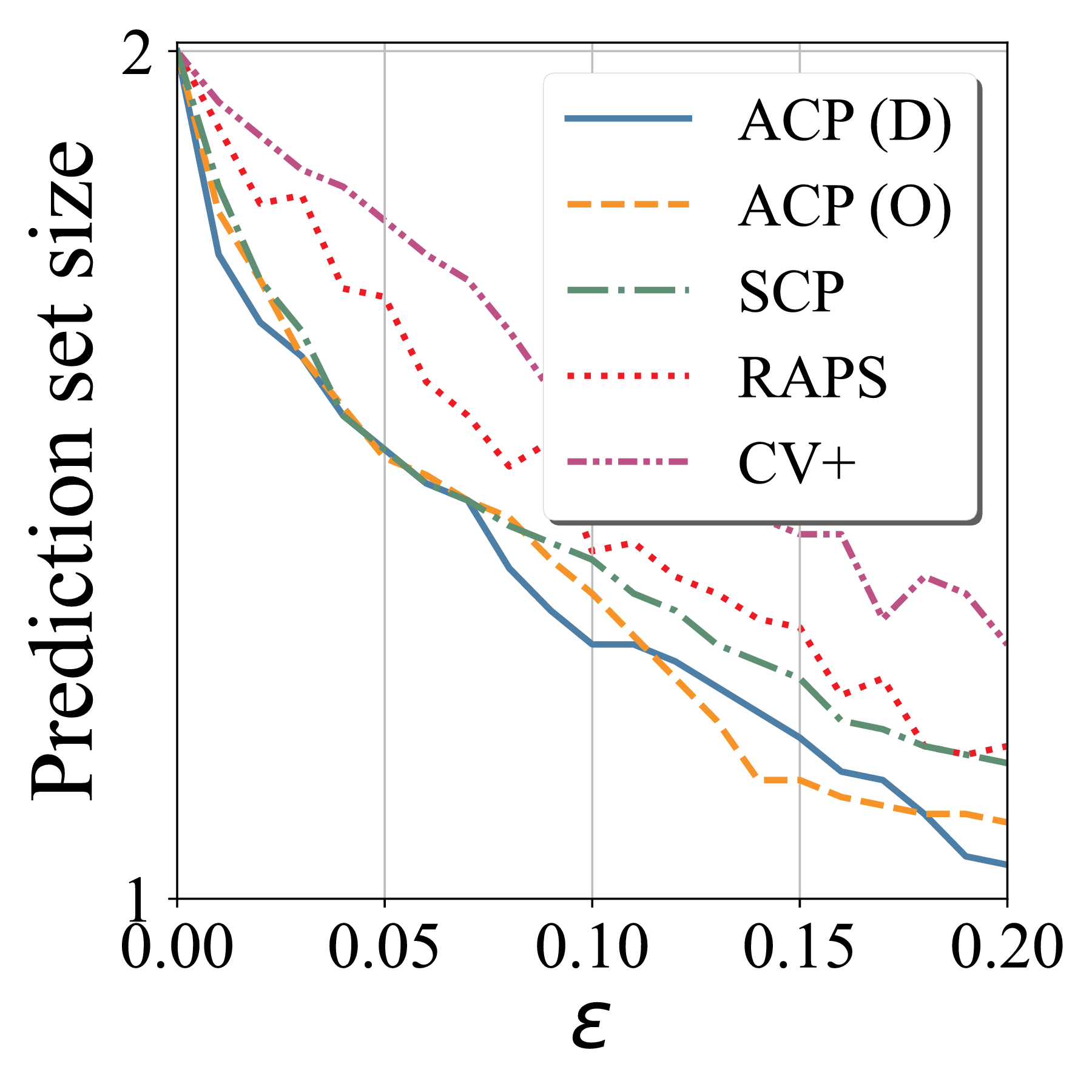}
        \caption*{US Census (\mlpC)}
\end{subfigure} 
\begin{subfigure}[hb]{0.16\textwidth}
     \includegraphics[width=\linewidth]{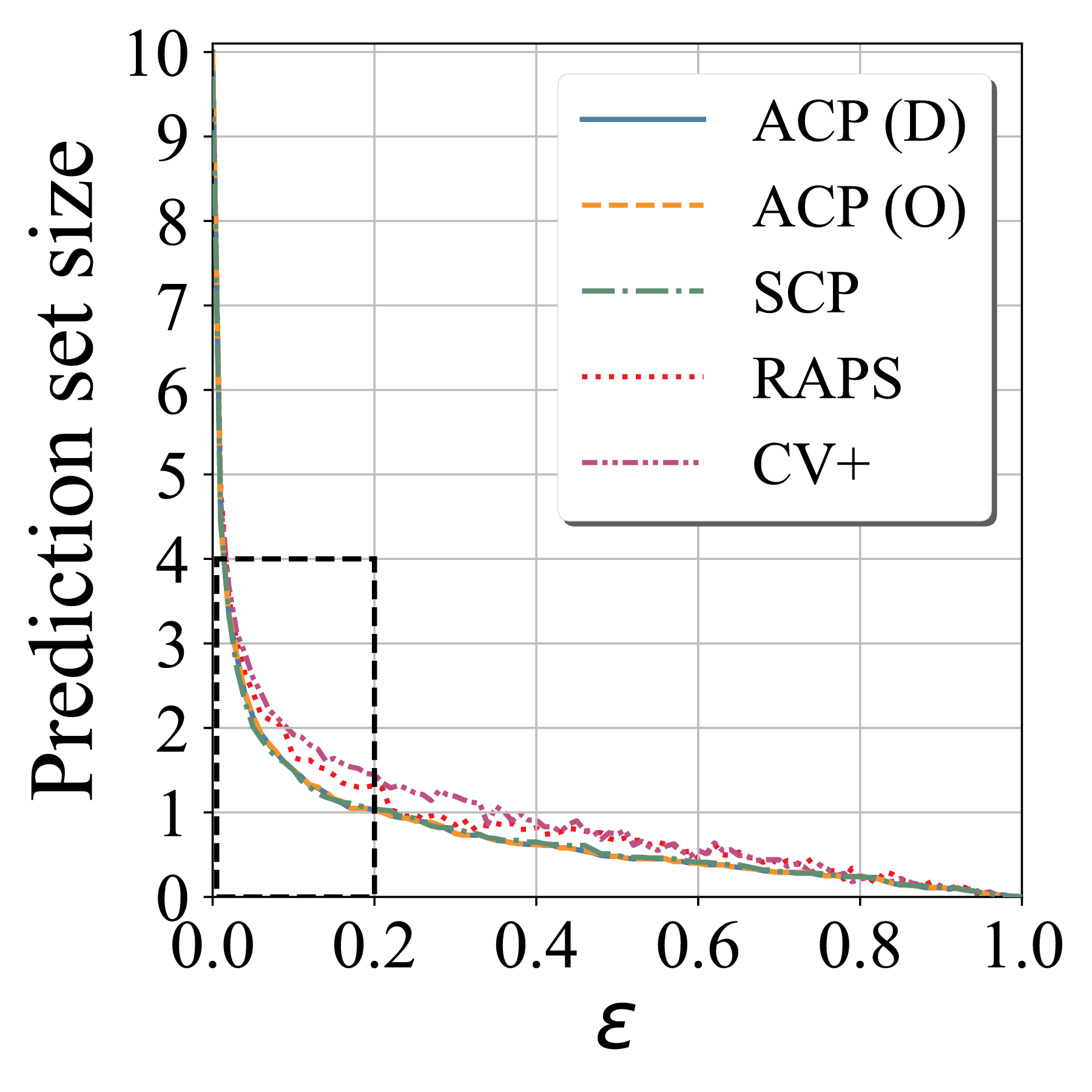}
     \caption*{MNIST (LR)}
\end{subfigure} 
\begin{subfigure}[hb]{0.16\textwidth}
        \includegraphics[width=\linewidth]{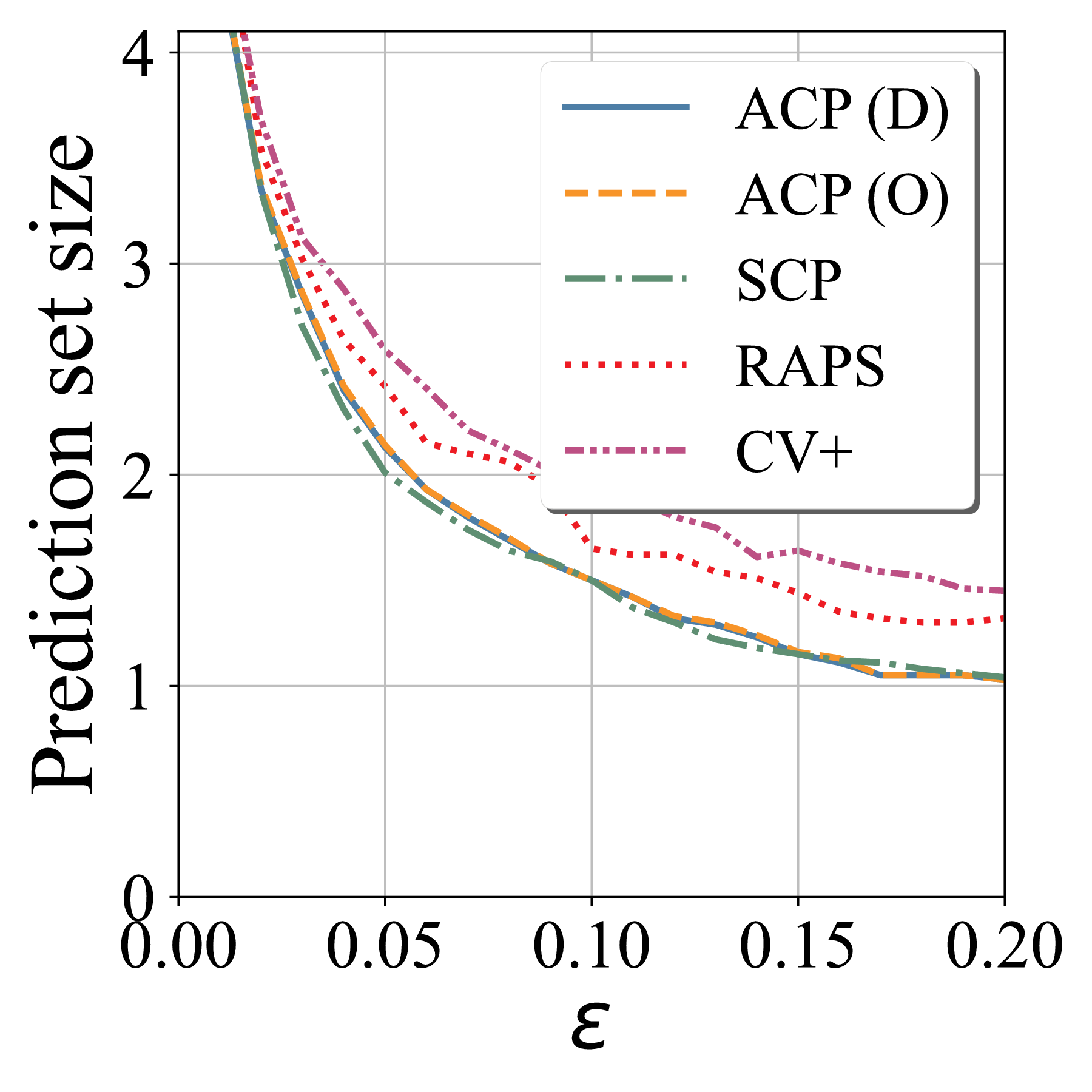}
        \caption*{MNIST (LR)}
\end{subfigure} 
\begin{subfigure}[hb]{0.16\textwidth}
     \includegraphics[width=\linewidth]{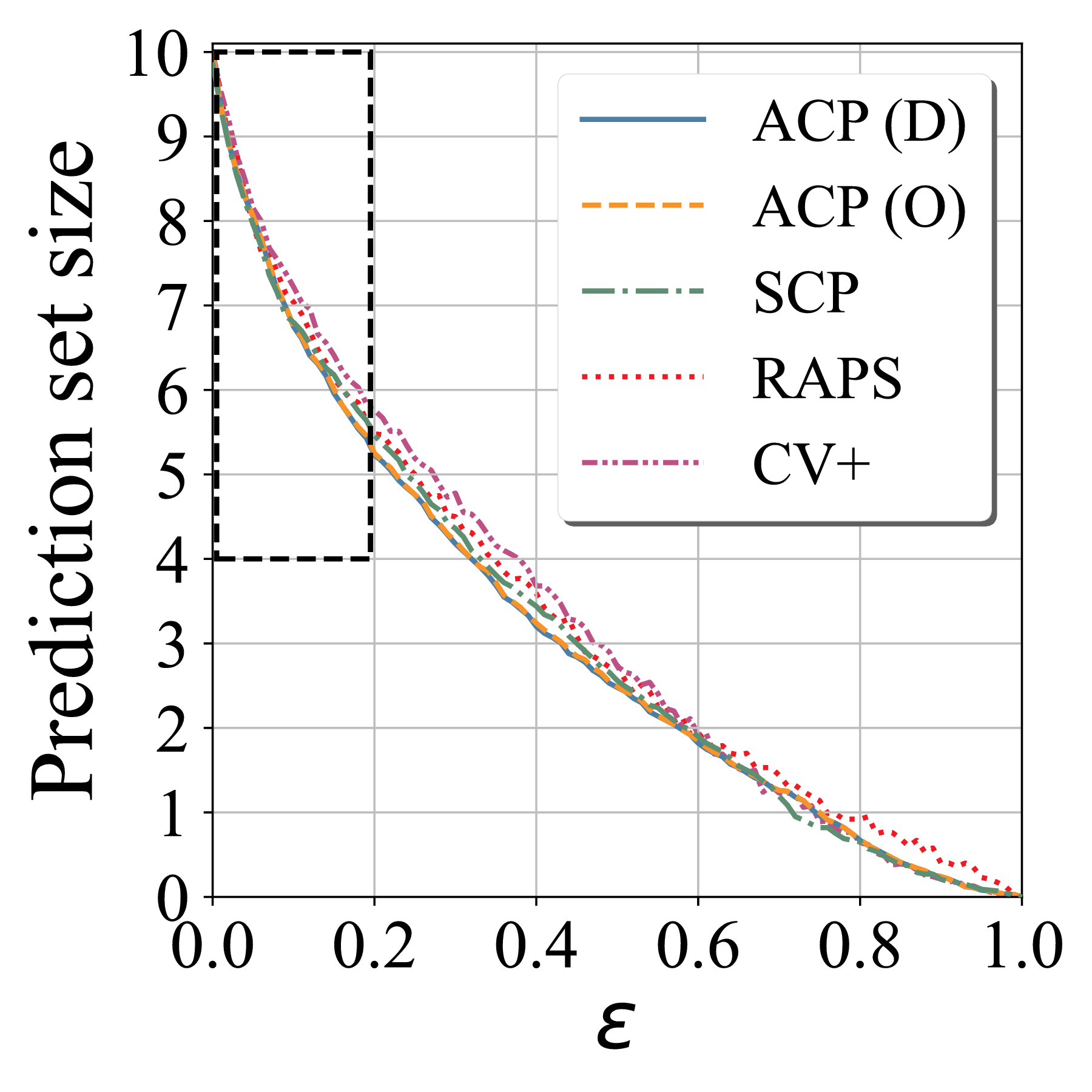}
     \caption*{CIFAR (LR)}
\end{subfigure} 
\begin{subfigure}[hb]{0.16\textwidth}
        \includegraphics[width=\linewidth]{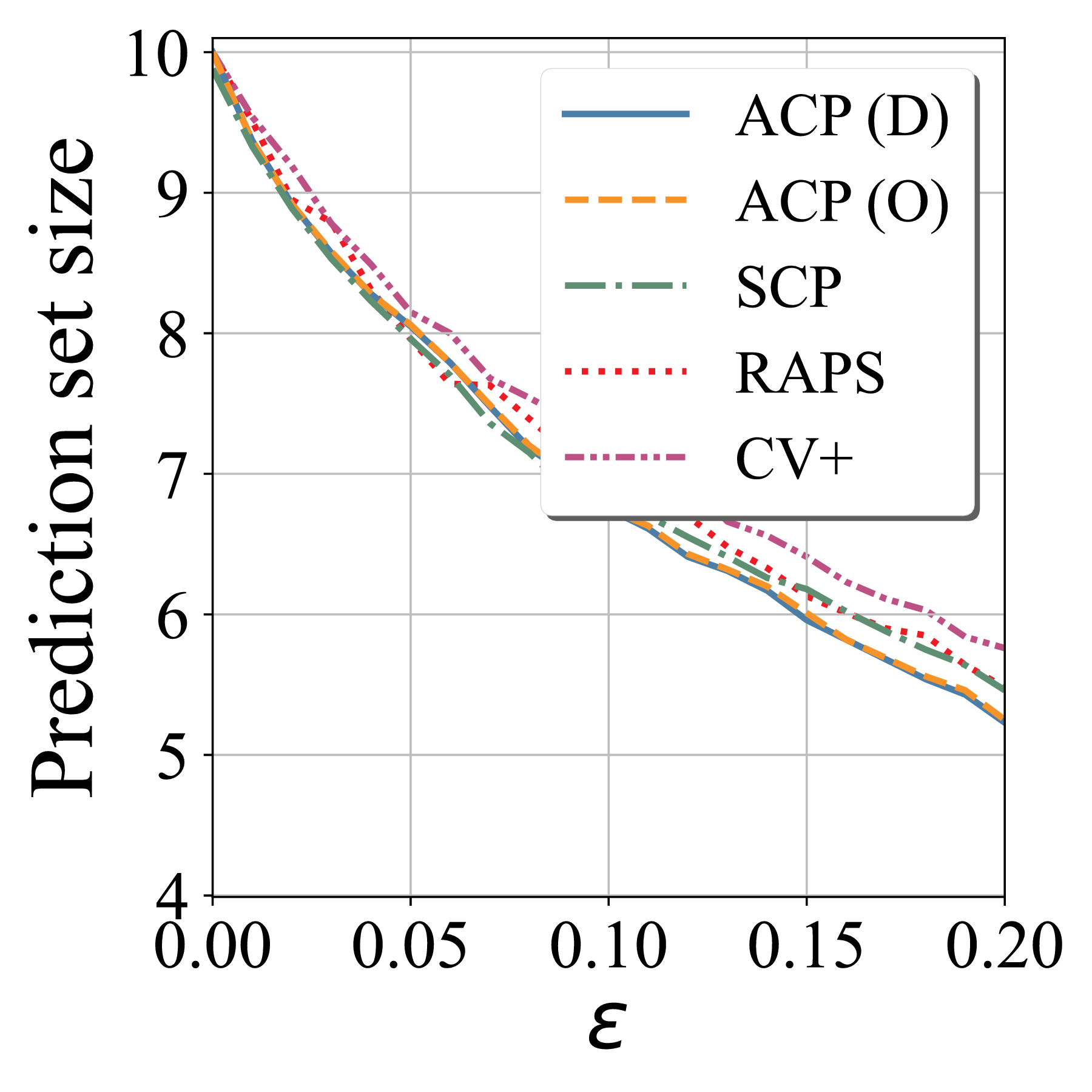}
        \caption*{CIFAR (LR)}
\end{subfigure}
\begin{subfigure}[hb]{0.16\textwidth}
     \includegraphics[width=\linewidth]{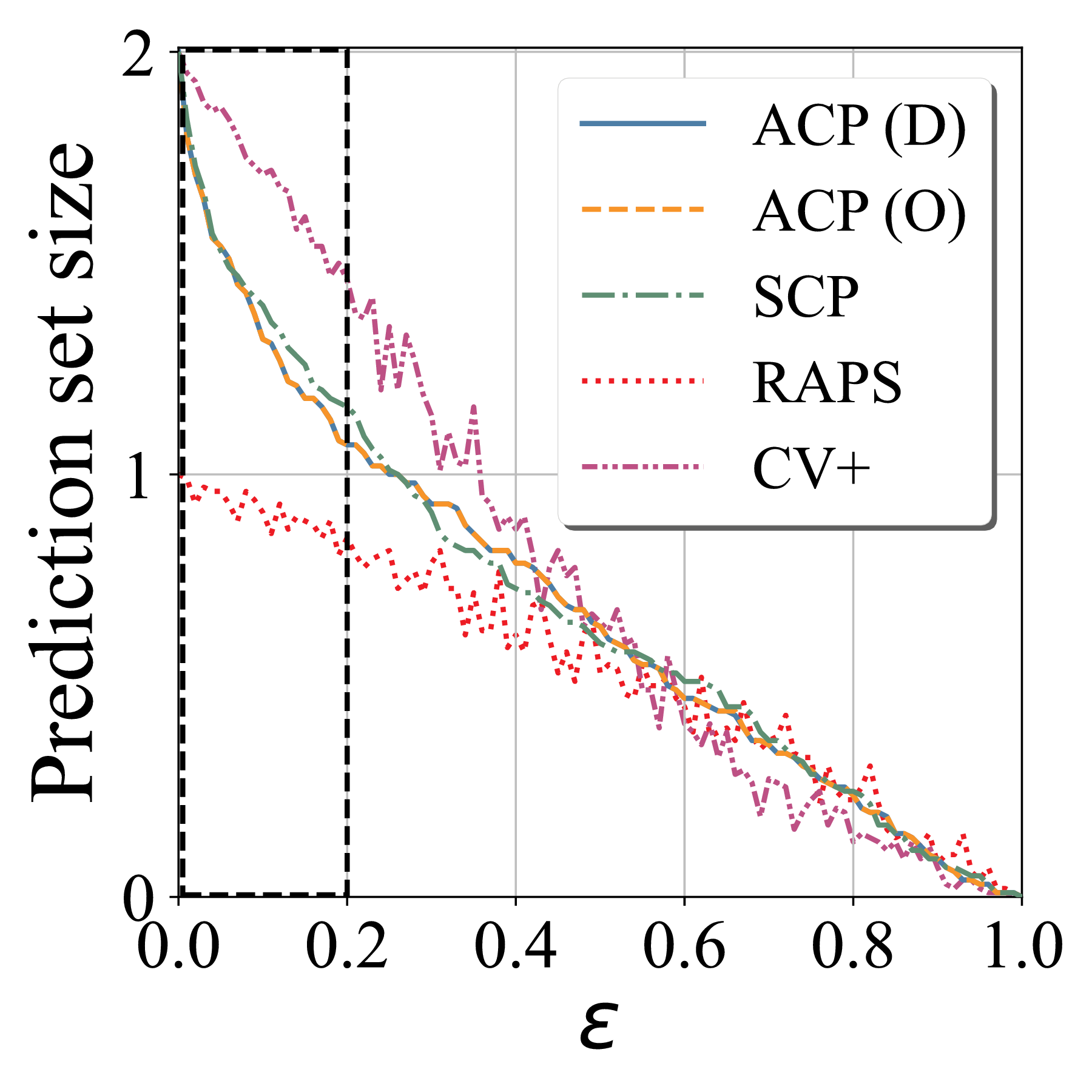}
     \caption*{US Census (LR)}
\end{subfigure} 
\begin{subfigure}[hb]{0.16\textwidth}
        \includegraphics[width=\linewidth]{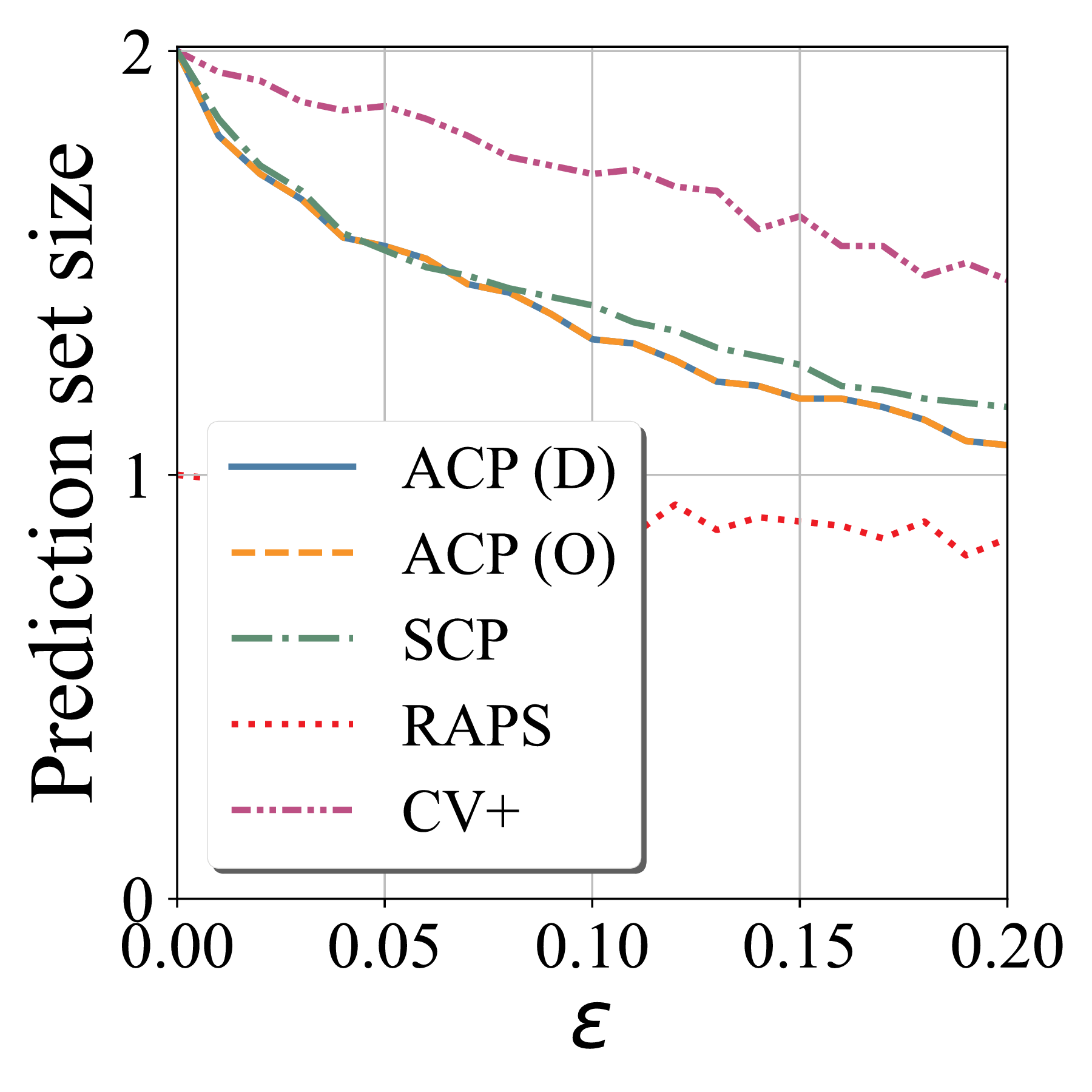}
        \caption*{US Census (LR)}
\end{subfigure} 
\begin{subfigure}[hb]{0.16\textwidth}
     \includegraphics[width=\linewidth]{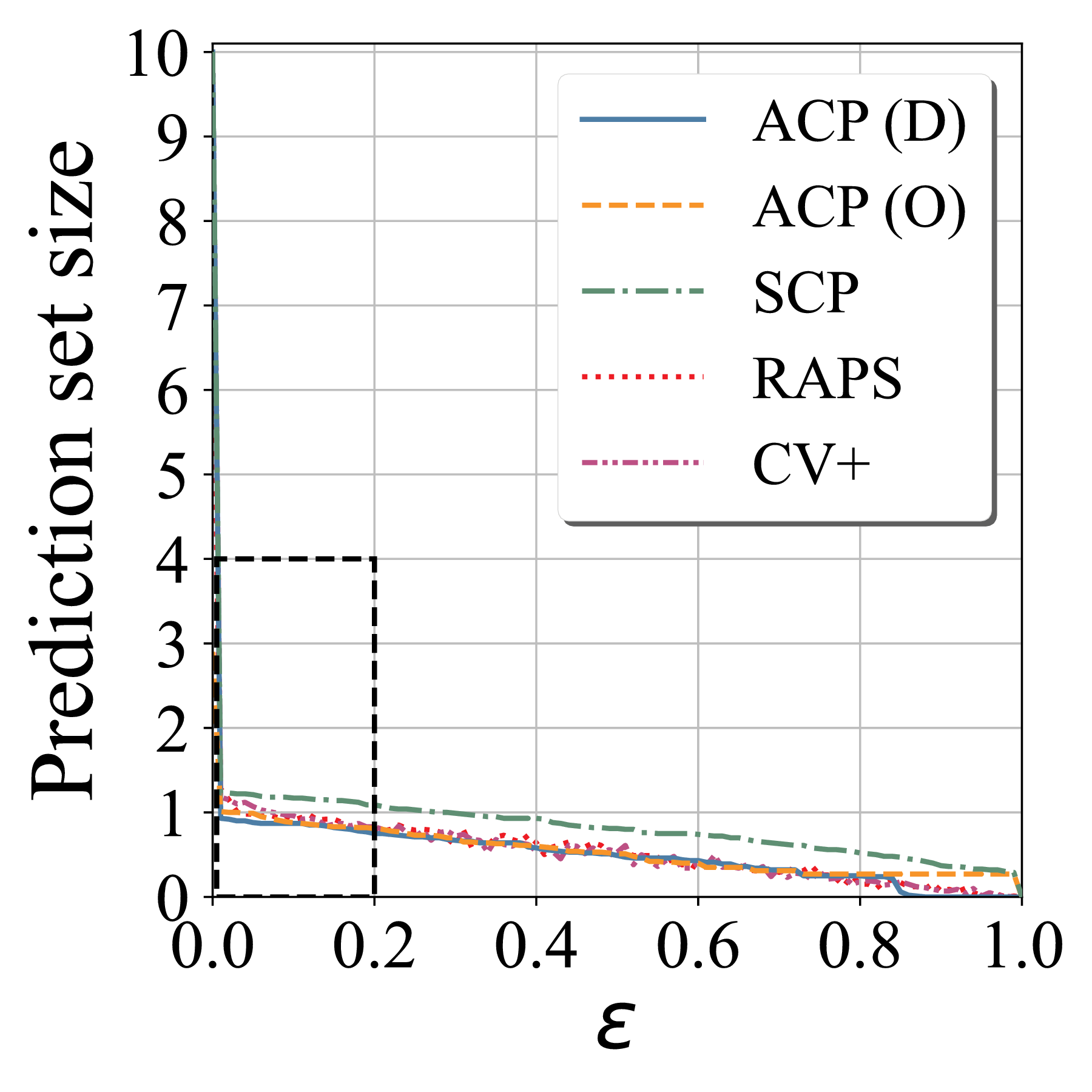}
     \caption*{MNIST (CNN)}
\end{subfigure} 
\begin{subfigure}[hb]{0.16\textwidth}
        \includegraphics[width=\linewidth]{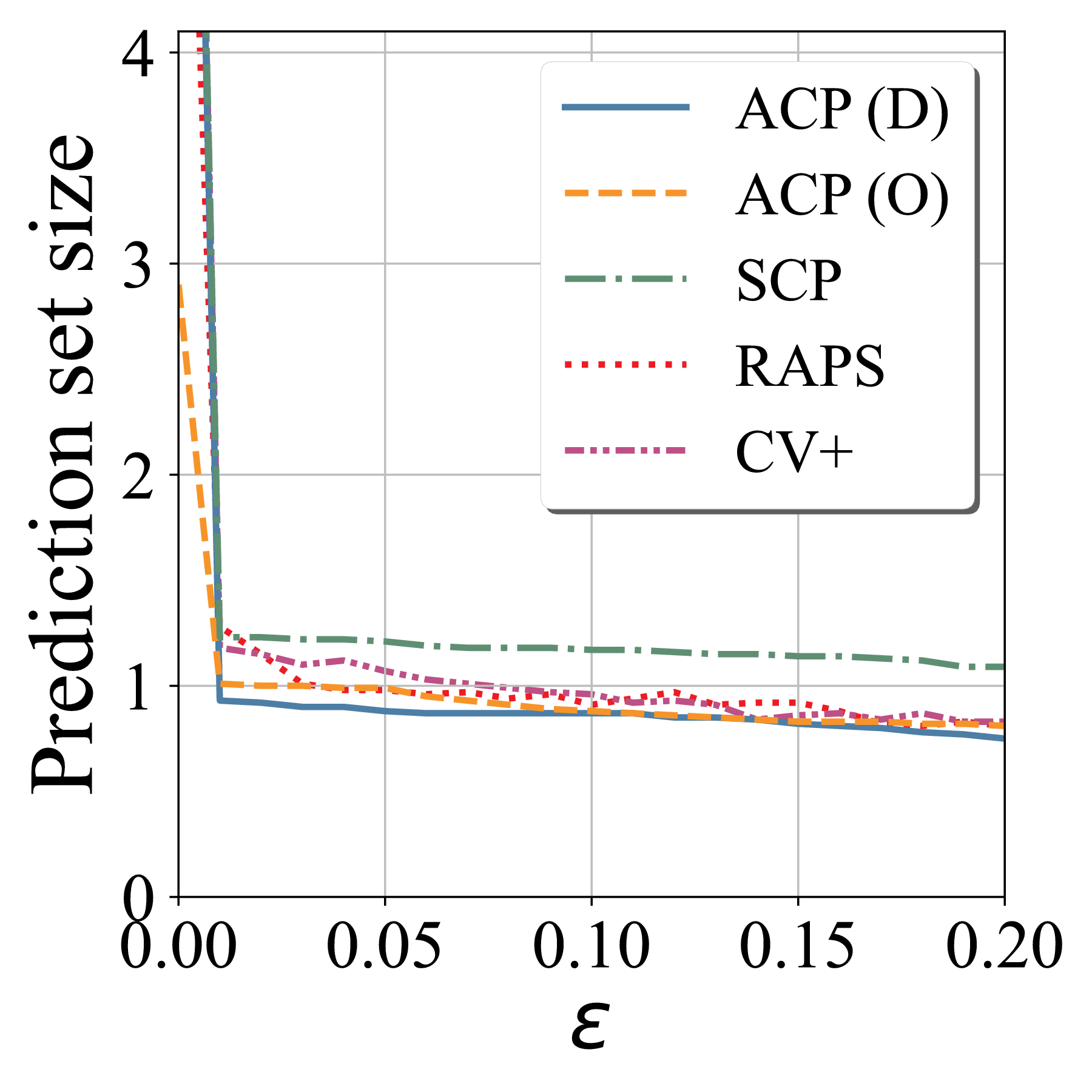}
        \caption*{MNIST (CNN)}
\end{subfigure} 
\begin{subfigure}[hb]{0.16\textwidth}
     \includegraphics[width=\linewidth]{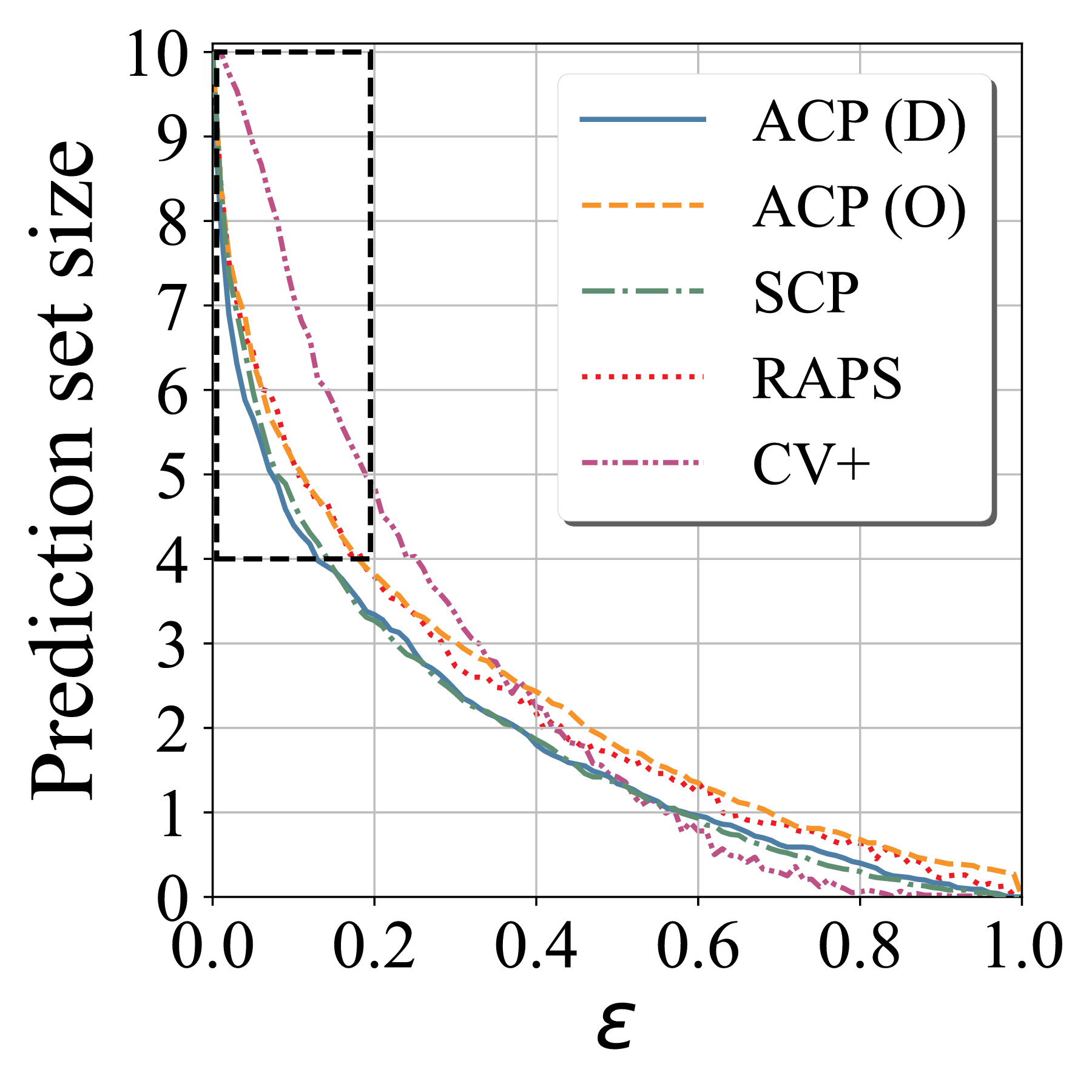}
     \caption*{CIFAR (CNN)}
\end{subfigure} 
\begin{subfigure}[hb]{0.16\textwidth}
        \includegraphics[width=\linewidth]{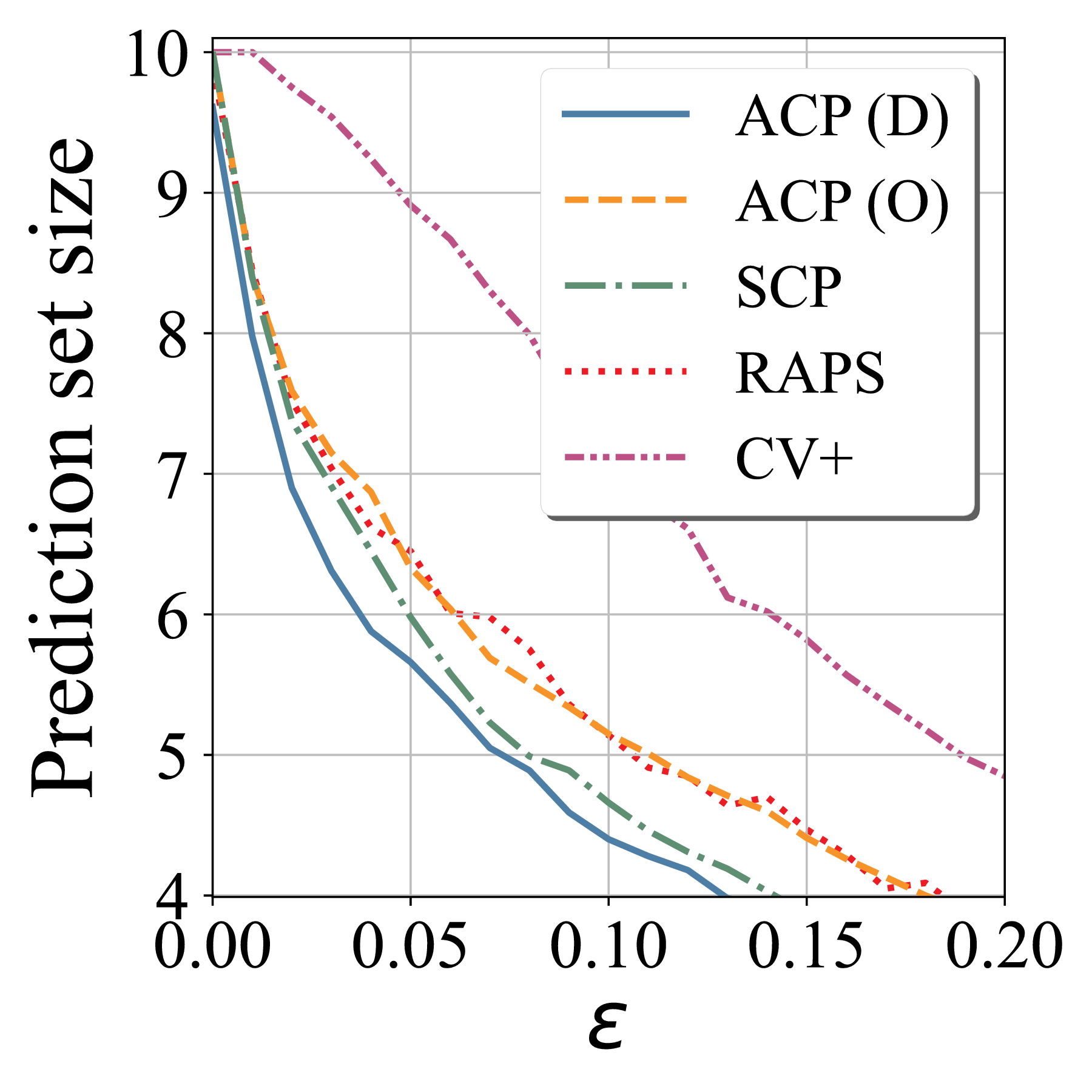}
        \caption*{CIFAR (CNN)}
\end{subfigure}
\caption{Average prediction set size w.r.t the significance level $\varepsilon$ for all settings and datasets. We show both the full curve and the corresponding to the interval of interest $\varepsilon \in [0,0.2]$} \label{fig:all_curves}
\end{figure*}